\documentclass[nohyperref]{article}

\usepackage{smile}
\usepackage{hyperref}

\usepackage{microtype}
\usepackage{graphicx}
\usepackage{booktabs} 

\usepackage{enumerate}
\usepackage{enumitem}

\newcommand{\alglinelabel}{%
  \addtocounter{ALC@line}{-1}
  \refstepcounter{ALC@line}
  \label
}

\usepackage[accepted]{icml2022}

\usepackage{amsmath}
\usepackage{amssymb}
\usepackage{mathtools}
\usepackage{amsthm}

\usepackage[capitalize,noabbrev]{cleveref}

\theoremstyle{plain}
\newtheorem{theorem}{Theorem}[section]

\newtheorem{lemma}[theorem]{Lemma}

\theoremstyle{definition}
\newtheorem{example}[theorem]{Example}

\newtheorem{assumption}[theorem]{Assumption}
\theoremstyle{remark}
\newtheorem{remark}[theorem]{Remark}

\newcommand{\Bstar}{B_{\star}}
\newcommand{\Tstar}{T_{\star}}
\newcommand{\pistar}{\pi^{\star}}
\newcommand{\sinit}{s_{\textnormal{init}}}
\newcommand{\sgn}{\textnormal{sgn}}

\usepackage[textsize=tiny]{todonotes}

\icmltitlerunning{Learning Stochastic Shortest Path with Linear Function Approximation}

\begin{document}

\twocolumn[
\icmltitle{Learning Stochastic Shortest Path with Linear Function Approximation}



\icmlsetsymbol{equal}{*}

\begin{icmlauthorlist}
\icmlauthor{Yifei Min}{sds}
\icmlauthor{Jiafan He}{uclacs}
\icmlauthor{Tianhao Wang}{sds}
\icmlauthor{Quanquan Gu}{uclacs}
\end{icmlauthorlist}

\icmlaffiliation{sds}{Department of Statistics and Data Science, Yale University, CT 06520, USA}
\icmlaffiliation{uclacs}{Department of Computer Science, University of California, Los Angeles, CA 90095, USA}

\icmlcorrespondingauthor{Quanquan Gu}{qgu@cs.ucla.edu}

\icmlkeywords{Machine Learning, ICML}

\vskip 0.3in
]



\printAffiliationsAndNotice{}  

\begin{abstract}
We study the stochastic shortest path (SSP) problem in reinforcement learning with linear function approximation, where the transition kernel is represented as a linear mixture of unknown models. We call this class of SSP problems as linear mixture SSPs. We propose a novel algorithm with Hoeffding-type confidence sets for learning the linear mixture SSP, which can attain an $\tilde{\mathcal{O}}(d B_{\star}^{1.5}\sqrt{K/c_{\min}})$ regret. Here $K$ is the number of episodes, $d$ is the dimension of the feature mapping in the mixture model, $B_{\star}$ bounds the expected cumulative cost of the optimal policy, and $c_{\min}>0$ is the lower bound of the cost function.
Our algorithm also applies to the case when $c_{\min} = 0$, and an $\tilde{\mathcal{O}}(K^{2/3})$ regret is guaranteed. To the best of our knowledge, this is the first algorithm with a sublinear regret guarantee for learning linear mixture SSP. 
Moreover, we design a refined Bernstein-type confidence set and propose an improved algorithm, which provably achieves an $\tilde{\mathcal{O}}(d B_{\star}\sqrt{K/c_{\min}})$ regret.
In complement to the regret upper bounds, we also prove a lower bound of $\Omega(dB_{\star} \sqrt{K})$. 
Hence, our improved algorithm matches the lower bound up to a $1/\sqrt{c_{\min}}$ factor and poly-logarithmic factors, achieving a near-optimal regret guarantee.
\end{abstract}

\section{Introduction}
The Stochastic Shortest Path (SSP) model refers to a type of reinforcement learning (RL) problems where an agent repeatedly interacts with a stochastic environment and aims to reach some specific goal state while minimizing the cumulative cost. 
Compared with other popular RL settings such as episodic and infinite-horizon Markov Decision Processes (MDPs), the horizon length in SSP is random, varies across different policies, and can potentially be infinite because the interaction only stops when arriving at the goal state.
Therefore, the SSP model includes both episodic and infinite-horizon MDPs as special cases, and is comparably more general and of broader applicability. 
In particular, many goal-oriented real-world problems fit better into the SSP model, such as navigation and GO game \citep{andrychowicz2017hindsight, nasiriany2019planning}.

In recent years, there emerges a line of works on developing efficient algorithms and the corresponding analyses for learning SSP. 
Most of them consider the episodic setting, where the interaction between the agent and the environment proceeds in $K$ episodes \citep{cohen2020near, tarbouriech2020no}.  
For tabular SSP models where the sizes of the action and state space are finite, \citet{cohen2021minimax} developed a finite-horizon reduction algorithm that achieves the minimax regret $\tilde\cO(\Bstar\sqrt{SAK})$, where $\Bstar$ is the largest expected cost of the optimal policy starting from any state, $S$ is the number of states and $A$ is the number of actions.
In a similar setting, \citet{tarbouriech2021stochastic} proposed the first algorithm that is minimax optimal, parameter-free and horizon-free at the same time. 
However, the algorithms  mentioned above only apply to tabular SSP problems where the state and action space are small. In order to deal with SSP problems with large state and action spaces, function approximation techniques  \citep{yang2019sample,jin2020provably,jia2020model,zhou2021provably,wang2020optimism,wang2020reinforcement} are needed. 


Following the recent line of work on model-based reinforcement learning with linear function approximation \citep{modi2020sample, jia2020model,ayoub2020model,zhou2021provably}, we consider a linear mixture SSP model, which extends the tabular SSP.
More specifically, we assume that the transition probability is parametrized by $\PP(s'|s,a) = \langle\bphi(s'|s,a),\btheta^*\rangle$ for all triplet $(s,a,s')\in\cS\times\cA\times\cS$, where $\cS$ is the state space and $\cA$ is the action space. 
Here we assume that $\bphi\in \RR^d$ is a known ternary feature mapping, and $\btheta^*\in\RR^d$ is an \emph{unknown} model parameter vector that needs to be learned. 
Such a setting has been previously studied for episodic MDPs \citep{modi2020sample, jia2020model,ayoub2020model,cai2020provably} and infinite-horizon discounted MDPs \citep{zhou2021provably}. 
Nevertheless, algorithms developed in these works do not apply to SSP since the horizon length is random as mentioned above. 

To tackle the challenge of varying horizon length, we propose a model-based optimistic algorithm with linear function approximation, dubbed \texttt{LEVIS}, for learning the linear mixture SSP. At the core of our algorithm are a confidence set of the model parameters and a Damped Extended Value Iteration (\texttt{DEVI}) subroutine for computing the optimistic estimate of the value function, which together guarantee that the algorithm will reach the goal state in every episode. 
Compared with the EVI subroutine developed for infinite-horizon discounted MDPs \citep{zhou2021provably}, we introduce a shrinking factor $q\approx 1/t$ in our \texttt{DEVI} with $t$ being the cumulative number of time steps, which guarantees the convergence of \texttt{DEVI}. 
To compensate for the bias introduced by this shrinking factor,
our algorithm performs lazy policy update, which is triggered by the doubling of the time interval between two policy updates or the doubling of the determinant of the covariance matrix. 
With all these algorithmic designs, our algorithm with Hoeffding-type bonus is guaranteed to achieve a $\tilde O(d B_\star^{1.5}\sqrt{K/c_{\min}})$ regret when $c_{\min}>0$. To the best of our knowledge, this is the first algorithm that enjoys a sublinear regret for linear mixture SSP. 
Moreover, we provide an improved algorithm $\texttt{LEVIS}^{\texttt{+}}$ based on Bernstein-type bonus that achieves an $\tilde O(d B_\star\sqrt{K/c_{\min}})$ regret, which nearly matches our regret lower bound $\Omega(d\Bstar \sqrt{K})$ in terms of $d, \Bstar$ and $K$.
This lower bound is proved via the construction of a hard-to-learn linear mixture SSP instance.

It is worth noting that a recent work by \citet{vial2021regret} studied a different linear SSP model that is similar to the linear MDP \citep{yang2019sample, jin2020provably}, where both the underlying transition probability and the cost function are linear in a known $d$-dimensional feature mapping $\bpsi \in \RR^d$, i.e., $\PP(s'|s,a) = \langle\bpsi(s,a),\bmu(s)\rangle$ and $c(s,a)=\langle\bpsi(s,a),\btheta\rangle$, and $\bmu(\cdot)$ and $\btheta$ are unknown.
For this model, their proposed algorithms can achieve an $\tilde \cO(\sqrt{K})$ regret. Their algorithms are either computationally inefficient or under stronger assumptions such as orthonormal feature mappings.
Their results are recently improved by \citet{chen2021improved} via a reduction to the finite-horizon MDP, yielding an efficient algorithm with an $\tilde\cO(\sqrt{K})$ regret, and a computationally inefficient but ``horizon free'' algorithm.
The linear SSP model is different from our linear mixture SSP model, and we refer the readers to Appendix \ref{subsec:linear mixture ssp} for a detailed discussion. \citet{chen2021improved} also proposed algorithms for linear mixture SSPs via a reduction to learning finite-horizon linear mixture MDPs.

Our contributions are summarized as follows:
\begin{itemize}
    \item We propose to study a linear mixture SSP model, and devise a novel and simple algorithm, dubbed \textbf{L}ower confidence \textbf{E}xtended \textbf{V}alue \textbf{I}teration for \textbf{S}SP (\texttt{LEVIS}), for learning SSP with linear function approximation. 
    
    \item We prove that \texttt{LEVIS} achieves a regret of order $\tilde\cO(d \Bstar^{1.5} \sqrt{K/c_{\min}})$ when $c_{\min}>0$ and the agent has an order-accurate estimate $B \geq \Bstar$\footnote{We say $B$ is an order-accurate estimate of $B^*$, if there exists some unknown constant $\kappa \geq 1$ such that $\Bstar \leq B \leq \kappa \Bstar$. }. 
    For the general case where $c_{\min} = 0$, our algorithm can achieve an $\tilde\cO(K^{2/3})$ regret guarantee by using a cost perturbation trick \citep{tarbouriech2021stochastic}. 
    
    \item We further propose an improved version of \texttt{LEVIS} called $\texttt{LEVIS}^{\texttt{+}}$ using Bernstein-type bonus, and prove that $\texttt{LEVIS}^{\texttt{+}}$ achieves an $\tilde\cO(d\Bstar \sqrt{K/c_{\min}})$ regret. 
    
    \item We prove that for linear mixture SSP, the regret of any learning algorithms is at least $\Omega(d \Bstar \sqrt{K})$.
    Hence, our $\texttt{LEVIS}^{\texttt{+}}$ algorithm nearly achieves the lower bound.
\end{itemize}

\textbf{Notation} 
We use lower case letters to denote scalars, and use lower and upper case bold face letters to denote vectors and matrices respectively. For any positive integer $n$, we denote by $[n]$ the set $\{1,\dots,n\}$. 
For a vector $\xb\in \RR^d$ , we denote by $\|\xb\|_1$ the Manhattan norm and denote by $\|\xb\|_2$ the Euclidean norm. For a vector $\xb\in \RR^d$ and matrix $\bSigma\in \RR^{d\times d}$, we define $\|\xb\|_{\bSigma}=\sqrt{\xb^\top\bSigma\xb}$. For two sequences $\{a_n\}$ and $\{b_n\}$, we write $a_n=O(b_n)$ if there exists an absolute constant $C$ such that $a_n\leq Cb_n$. We use $\tilde O(\cdot)$ to hide the logarithmic factors. 

\section{Related Work}
\noindent\textbf{Online learning in SSP} 
SSP problems can be dated back to \citep{bertsekas1991analysis,bertsekas2013stochastic,bertsekas2012dynamic}, but it is until recently that the regret minimization in online learning of SSP has been studied.
In the tabular case, \citet{tarbouriech2020no} proposed the first algorithm achieving an $\tilde\cO(D^{3/2}S\sqrt{AK/c_{\min}})$ regret where $D$ is the diameter of SSP\footnote{The diameter of an SSP is defined as the longest possible shortest path from any initial state to the goal state.}.
The regret was further improved to $\tilde O(B_\star S\sqrt{AK})$ by \citet{rosenberg2020near,cohen2020near}, with an extra $\sqrt{S}$ factor compared with the $\Omega(B_\star \sqrt{SAK})$ lower bound \citep{rosenberg2020near}. 
More recently, the $\tilde O(B_\star \sqrt{SAK})$ minimax optimal regret were obtained by \citet{cohen2021minimax} and \citet{tarbouriech2020improved} independently using different approaches.
Specifically,  \citet{cohen2021minimax} reduced SSP to a finite-horizon MDP with a large terminal cost assuming $B_\star$ is known;
while \citet{tarbouriech2021stochastic} avoid such requirement by adaptively estimating $B_\star$ with a doubling trick, together with a value iteration sub-routine ensuring the optimistic estimate of the value function. 
Our proposed method shares a similar spirit with the latter approach, but for learning SSP with linear function approximation.

The above algorithms are all model-based. 
Very recently, \citet{chen2021implicit} developed the first model-free algorithm for SSP which achieves the minimax optimal regret when the minimum cost among all state-action pairs $c_{\min}$ is strictly positive. 
Their method is motivated by the \texttt{UCB-ADVANTAGE} algorithm \citep{zhang2020almost}. \citet{chen2022policy} proposed the first policy optimization algorithm for tabular SSP.
For other settings of SSP,  \citep{rosenberg2020stochastic, chen2021finding, chen2021minimax} studied  the case of adversarial costs. 
Also, the pioneering work by \citep{bertsekas1991analysis} studied the pure planning problem in SSP where the agent has full knowledge of all the model parameters, and is followed by a series of works \citep{bonet2007speed,kolobov2011heuristic,bertsekas2013stochastic,guillot2020stochastic}. 
On the other hand, \citet{tarbouriech2021sample} studied the sample complexity of SSP assuming the access to a generative model. 
\citet{jafarnia2021online} proposed the first posterior sampling algorithm for SSP. 
Multi-goal SSPs have also been studied by \citet{lim2012autonomous, tarbouriech2020improved}. 

\noindent\textbf{Linear function approximation} 
Linear MDP is one of the most widely studied models for RL with linear function approximation, which assumes both the transition probability and reward functions are linear functions of a known feature mapping \citep{yang2019sample, jin2020provably}. 
Representative work in this direction include \citet{du2019good, zanette2020frequentist,wang2020reinforcement,fei2021risk,he2021logarithmic}, to mention a few. 

Another popular model for RL with linear function approximation is the so-called linear mixture MDP/linear kernel MDP \citep{yang2020reinforcement, modi2020sample, jia2020model, ayoub2020model, cai2020provably,min2021variance, zhou2021provably, zhou2020nearly}. 
For the finite-horizon setting, \citet{jia2020model} proposed a \texttt{UCRL-VTR} algorithm that achieves a $\tilde\cO(d\sqrt{H^3T})$ regret bound. 
\citet{zhou2020nearly} further improve the result by proposing a \texttt{UCRL-VTR+} algorithm that attains the nearly minimax optimal regret $\tilde\cO(dH\sqrt{T})$  based on a novel Bernstein-type concentration inequality. 
For the discounted infinite horizon setting, \citet{zhou2021provably} proposed a \texttt{UCLK} algorithm with an $\tilde\cO(d\sqrt{T}/(1-\gamma)^2)$ regret, and also give a $\tilde\cO(d\sqrt{T}/(1-\gamma)^{1.5})$ lower bound. The lower bound is later matched up to logarithmic factors by the \texttt{UCLK+} algorithm \citep{zhou2020nearly}. The SSP model studied in this paper can be seen as an extension of linear mixture MDPs.


\section{Preliminaries}\label{sec: prelim}

\paragraph{Stochastic Shortest Path} An SSP instance is an MDP $M\coloneqq \{ \cS,\cA, \PP, c, \sinit,g \}$, where $\cS$ and $\cA$ are the finite state space and action space respectively. 
Here $\sinit$ denotes the initial state and $g \in \cS$ is the goal state. 
We denote the cost function by $c:\cS \times \cA \to [0,1]$, where $c(s,a)$ is the immediate cost of taking action $a$ at state $s$. 
The goal state $g$ incurs zero cost, i.e., $c(g,a) = 0$ for all $a\in\cA$. 
For any $(s',s,a)\in\cS\times\cA\times\cS$, $\PP(s'|s,a)$ is the probability to transition to $s'$ given the current state $s$ and action $a$ being taken.
The goal state $g$ is an absorbing state, i.e., $\PP(g|g,a)=1$ for all action $a\in\cA$. 

\textbf{Linear mixture SSP} In this work, we assume the transition probability function $\PP$ to be a linear mixture of some basis kernels \citep{modi2020sample,ayoub2020model,zhou2020nearly}. 
\begin{assumption}\label{assump: linear kernel mdp}
Assume the feature mapping $\bphi:\cS\times\cA\times\cS\to\RR^d$ is known. 
There exists an \emph{unknown} vector $\btheta^* \in \RR^d$ with $\|\btheta^*\|_2\leq \sqrt{d}$ such that $\PP(s'|s,a) = \langle \bphi(s'|s,a) , \btheta^* \rangle$ for any state-action-state triplet $(s,a,s') \in \cS \times \cA \times \cS$. 
Moreover, for any bounded function $V:\cS\to[0, B]$, it holds that $\|\bphi_V(s,a)\|_2 \leq B\sqrt{d}$ for all $(s,a)\in\cS\times\cA$, where $\bphi_V(s,a) \coloneqq \sum_{s'\in\cS}\bphi(s'|s,a)V(s')$.
\end{assumption}
For simplicity, for any function $V: \cS \rightarrow \RR$, we denote $\PP V(s,a)=\sum_{s'} \PP(s'|s,a) V(s')$ for all $(s,a)\in\cS\times\cA$. 
Therefore, under Assumption \ref{assump: linear kernel mdp}, we have
\begin{align*}
    \PP V(s,a) & =  \sum_{s'\in\cS} \PP(s'|s,a) V(s') 
    \\ & = \sum_{s' \in \cS} \langle \bphi(s'|s,a) , \btheta^* \rangle V(s') 
    \\ & = \langle \bphi_V(s,a) , \btheta^* \rangle. 
\end{align*}

\textbf{Proper policies} A stationary and deterministic policy is a mapping $\pi: \cS \to \cA$ such that the action $\pi(s)$ is taken given the current state $s$. 
We denote by $T^\pi(s)$ the expected time that it takes by following $\pi$ to reach the goal state $g$ starting from $s$. 
We say a policy $\pi$ is proper if $T^\pi(s)<\infty$ for any $s \in \cS$ (otherwise it is improper). 
We denote by $\Pi_{\textnormal{proper}}$ the set of all stationary, deterministic and proper policies. 
We assume that $\Pi_{\textnormal{proper}}$ is non-empty, which is the common assumption in previous works on online learning of SSP \citep{rosenberg2020near,rosenberg2020stochastic,cohen2021minimax,tarbouriech2021stochastic,jafarnia2021online,chen2021implicit}.

\begin{assumption}\label{assump: existence proper policy}
The set of all stationary, deterministic and proper policies is non-empty, i.e., $\Pi_{\textnormal{proper}} \neq \varnothing$. 
\end{assumption}

\begin{remark}
The above assumption is weaker than Assumption 1 in \citet{vial2021regret} which requires that all stationary policies are proper. 
\end{remark} 
For any policy $\pi$, we define the cost-to-go function (a.k.a., value function) as 
\begin{align*}
    V^\pi(s) \coloneqq \lim_{T \to +\infty} \EE \left[ \sum_{t=1}^T c(s_t,\pi(s_t)) \middle| s_1 = s \right], 
\end{align*}where $s_{t+1}\sim \PP\left(\cdot|s_{t},\pi(s_t)\right)$.
$V^\pi(s)$ can possibly be infinite if $\pi$ is improper. 
The corresponding action-value function of policy $\pi$ is defined as 
\begin{align*}
    & Q^{\pi}(s,a) 
    \\ &\coloneqq \lim_{T \to \infty} \EE \left[ c(s_1,a_1) + \sum_{t=2}^T c(s_t,\pi(s_t)) \middle| s_1=s,\ a_1=a \right],
\end{align*}

where $s_{2}\sim \PP(\cdot|s_{1},a_1)$ and $s_{t+1}\sim \PP(\cdot|s_{t},\pi(s_t))$ for all $t\geq 2$. 
Since $c(\cdot,\cdot) \in [0,1]$, for any proper policy $\pi \in \Pi_{\textnormal{proper}}$, $V^\pi$ and $Q^\pi$ are both bounded functions. 

\textbf{Bellman optimality} For any function $V:\cS \to \RR$, we define the optimal Bellman operator $\cL$ as
\begin{align}\label{eq: bellman operator}
    \cL V(s) \coloneqq \min_{a\in \cA}\{ c(s,a) + \PP V(s,a) \}.
\end{align} 
\vspace{-0.6cm}

Intuitively speaking, we want to learn the optimal policy $\pistar$ such that $V^\star(\cdot) \coloneqq V^{\pistar}(\cdot)$ is the unique solution to the Bellman optimality equation $V = \cL V$ and $\pistar$ minimizes the value function $V^{\pi}(s)$ component-wise over all policies. 
It is known that, in order for such $\pistar$ to exist, one sufficient condition is Assumption \ref{assump: existence proper policy} together with an extra condition that any improper policy $\pi$ has at least one infinite-value state, i.e., for any $\pi\notin \Pi_{\textnormal{proper}}$, there exists some $s\in \cS$ s.t. $V^{\pi}(s) = +\infty$ \citep{bertsekas1991analysis,bertsekas2013stochastic,tarbouriech2021stochastic}. 
Note that this additional condition is satisfied in the case of strictly positive cost , where for any state $s\neq g$ and $a\in\cA$, it holds that $c(s,a) \geq c_{\min}$. 
To deal with the case of general cost function, one can adopt the cost perturbation trick \citep{tarbouriech2021stochastic} and consider a modified problem with cost function $c_{\rho}(s,a) \coloneqq \max\{c(s,a), \rho\}$ for some $\rho>0$.
This will introduce an additional cost of order $ \cO(\rho T)$ to the regret of the original problem, where $T$ is the total number of steps.
Therefore, the second condition can be avoided, and we can assume the existence of $\pi^\star$. 

Throughout the paper, we denote by $\Bstar$ the upper bound of the optimal value function $V^\star$, i.e., $\Bstar \coloneqq \max_{s\in\cS}V^\star(s)$. 
Also, we define $\Tstar \coloneqq \max_{s\in\cS} T^{\pistar} (s)$, which is finite under Assumption \ref{assump: existence proper policy}. 
Since the cost is bounded by 1, we have $\Bstar \leq \Tstar < + \infty$. 
Without loss of generality, we assume that $\Bstar\geq 1$. 
Furthermore, we denote the corresponding optimal action-value function by $Q^\star \coloneqq Q^{\pistar}$ which satisfies the following Bellman equation for all $(s,a)\in\cS \times \cA$:
\begin{align}\label{eq: bellman optimal condition}
    Q^\star(s,a) & = c(s,a) + \PP V^\star(s,a), \notag
    \\  V^\star(s) & = \min_{a\in\cA} Q^\star(s,a).
\end{align}

\textbf{Learning objective} 
Under Assumption \ref{assump: linear kernel mdp}, we assume $c$ to be known for the ease of presentation.
We study the episodic setting where each episode starts from a fixed initial state $\sinit$ and ends only if the agent reaches the goal state $g$. 
Given the total number of episodes, $K$, the objective of the agent is to minimize the regret over $K$ episodes defined as
\begin{align}\label{eq: regret def}
    R_K \coloneqq \sum_{k=1}^K \sum_{i=1}^{I_k} c_{k,i} - K \cdot V^\star(\sinit),
\end{align}
where $I_k$ is the length of the $k$-th episode and $c_{k,i} = c(s_{k,i}, a_{k,i})$ is the cost triggered at the $i$-th step in the $k$-th episode. 
Note that $R_K$ might be infinite if some episode never ends.

\section{An Algorithm with Hoeffding-type Bonus}\label{sec: algorithm}

In this section, we propose a model-based algorithm for learning linear mixture SSPs, which is displayed in Algorithm~\ref{alg: linear kernel ssp}. 
\texttt{LEVIS} is inspired by the \texttt{UCLK}-type of algorithms originally designed for discounted linear mixture MDPs \citep{zhou2020nearly, zhou2021provably}. 
Our algorithm takes a multi-epoch form, where each episode is divided into epochs of different lengths \citep{jaksch2010near, lattimore2012pac}. 
Within each epoch, the agent executes the greedy policy induced by an optimistic estimator of the optimal Q-function. 
The switch between any two epochs is triggered by a doubling criterion, and then the estimated Q-function is updated through a Discounted Extend Value Iteration (\texttt{DEVI}) sub-routine (Algorithm \ref{alg: EVI}).
We now give a detailed description of Algorithm \ref{alg: linear kernel ssp}.  


\begin{algorithm}[t]
	\caption{\texttt{LEVIS}}
	\label{alg: linear kernel ssp}
	\begin{algorithmic}[1]
	\STATE {\bfseries Input:} regularization parameter $\lambda$, confidence radius $\{\beta_t\}$, cost perturbation $\rho \in [0,1]$, an estimate $B \geq \Bstar$
	\STATE {\bfseries Initialize:} set $t\leftarrow 1$, $j\leftarrow 0$, $t_0=0$, $\bSigma_0 \leftarrow \lambda \Ib$, $ \bbb_0 \leftarrow \zero$, $Q_0(s,\cdot), V_0(s) \leftarrow 1$ $\forall s \neq g$ and $0$ otherwise  \alglinelabel{algline: initialization}
	\FOR{$k=1,\dots,K$}
	\STATE Set $s_t = \sinit$ \alglinelabel{algline: start from s init}
	\WHILE{$s_t \neq g$} \alglinelabel{algline: while not g}
	\STATE Take action $a_t = \argmin_{a\in\cA} Q_j(s_t,a)$, receive cost $c_t = c(s_t,a_t)$ and next state $s_{t+1}\sim \PP(\cdot|s_t,a_t)$ \alglinelabel{algline: take one step}
	\STATE Set $\bSigma_{t} \leftarrow \bSigma_{t-1}+\bphi_{V_j}(s_t,a_t)\bphi_{V_j}(s_t,a_t)^\top$ \alglinelabel{algline: bSigma update}
	\STATE Set $\bbb_t \leftarrow \bbb_{t-1} +  \bphi_{V_j}(s_t,a_t) V_j(s_{t+1})$ \alglinelabel{algline: b update}
	\IF{$\det (\bSigma_t)\geq 2 \det(\bSigma_{t_j})$ \textbf{or} \textcolor{blue}{$t \geq 2 t_j$} } \alglinelabel{algline: trigger condition}
	\STATE Set $j \leftarrow j+1$ \alglinelabel{algline: j update}
	\STATE Set $t_j \leftarrow t$,  \textcolor{blue}{$\epsilon_j \leftarrow \frac{1}{t_j}$} \alglinelabel{algline: def t_j}
	\STATE Set $\hat{\btheta}_j \leftarrow \bSigma_t^{-1} \bbb_t$ \alglinelabel{algline: hat btheta}
	\STATE Set $
	    \cC_j \leftarrow \left\{ \btheta: \ \| \bSigma_{t_j}^{1/2} (\btheta - \hat{\btheta}_j)\|_2 \leq \beta_{t_j} \right\}
	$ \alglinelabel{algline: confidence ellipsoid}
	\STATE Set $Q_j(\cdot,\cdot) \leftarrow \texttt{DEVI}(\cC_j, \textcolor{blue}{\epsilon_j}, \textcolor{blue}{\frac{1}{t_j}},\rho)$ \alglinelabel{algline: discount factor 1 over t_j}
	\STATE Set $V_j (\cdot) \leftarrow \min_{a\in\cA} Q_j (\cdot,a)$ \alglinelabel{algline: V_j equal min Q_j}
	\ENDIF
	\STATE Set $t\leftarrow t+1$ 
	\ENDWHILE
	\ENDFOR
	\end{algorithmic}
\end{algorithm}

\begin{algorithm}[t]
	\caption{\texttt{DEVI}}
	\label{alg: EVI}
	\begin{algorithmic}[1]
	\STATE {\bfseries Input:} confidence set $\cC$, error parameter $\epsilon$, transition bonus $q$, cost perturbation $\rho \in [0,1]$
	\STATE {\bfseries Initialize:} $i \leftarrow 0$, and $Q^{(0)}(\cdot,\cdot),\ V^{(0)}(\cdot) = 0$, and $V^{(-1)}(\cdot) = + \infty$
	\STATE Set $ Q(\cdot,\cdot) \leftarrow Q^{(0)} (\cdot,\cdot) $ 
	\IF{$\cC \cap \cB \neq \phi$}
	\WHILE{$\| V^{(i)} - V^{(i-1)} \|_\infty \geq \epsilon $} \label{algline: EVI updating criterion}
    \STATE
	\small{\begin{align}
        & Q^{(i+1)}(\cdot,\cdot)  \leftarrow c_\rho (\cdot,\cdot)  
        + \textcolor{blue}{(1-q)}\min_{\btheta \in \cC\cap \cB} \langle \btheta, \bphi_{V^{(i)}} (\cdot,\cdot) \rangle  \label{eq: EVI value iteration}\\ 
        & V^{(i+1)} (\cdot) \leftarrow \min_{a\in\cA} Q^{(i+1)} (\cdot, a) \label{eq: EVI V is min Q}
	\end{align}}
	\STATE Set $i \leftarrow i+1$
	\ENDWHILE
	\STATE $Q(\cdot,\cdot) \leftarrow Q^{(i+1)}(\cdot,\cdot)$ 
	\ENDIF
	\STATE {\bfseries Output:} $Q(\cdot,\cdot)$
	\end{algorithmic}
\end{algorithm}


In Algorithm \ref{alg: linear kernel ssp}, we maintain two global indices. 
Index $t$ represents the total number of steps, and index $j$ tracks the number of calls to the \texttt{DEVI} sub-routine, where the output of \texttt{DEVI} is an optimistic estimator of the optimal action-value function. 
Each episode starts from a fixed initial state $\sinit$ (Line \ref{algline: start from s init}), ends when the goal state $g$ is reached (Line \ref{algline: while not g}) and is decomposed into epochs indexed by the global index $j$.
Within epoch $j$, the agent repeatedly executes the policy induced by the current estimate $Q_j$ (Line \ref{algline: take one step}) and updates $\bSigma_t$ and $\bbb_t$ (Lines \ref{algline: bSigma update} and~\ref{algline: b update}).
Each epoch ends when the either criterion in Line \ref{algline: trigger condition} is triggered, and the \texttt{DEVI} subroutine performs an optimistic planning to update the action-value function estimator (Lines~\ref{algline: j update} to~\ref{algline: V_j equal min Q_j}). 


\subsection{Updating Criteria: Coupling Features with Time}\label{sec: updating criteria in algorithm}
As mentioned before, Algorithm \ref{alg: linear kernel ssp} runs in epochs indexed by $j$, and one epoch ends when either of the two updating criteria is triggered (Line \ref{algline: trigger condition}). 
The first updating criterion is satisfied once the determinant of $\bSigma_t$ is at least doubled compared to its determinant at the end of the previous epoch. 
This is called lazy policy update that has been used in the linear bandits and RL literature \citep{abbasi2011improved,zhou2021provably,wang2021provably},
which reflects the diminishing return of learning the underlying transition. 
The intuition behind the determinant doubling criterion is that the determinant can be viewed as a surrogate measure of the exploration in the feature space. Thus, one only updates the policy when there is enough exploration being made since last update.
Moreover, this update criterion reduces the computational cost as the total number of epochs would be bounded by $\cO(\log T)$. Here $T$ is the total number of steps through all $K$ episodes.
The doubling visitation criterion used in tabular SSP \citep{jafarnia2021online,tarbouriech2021stochastic} can be viewed as a special case of this doubling determinant-based criterion.

However, the above criterion alone cannot guarantee finite length for each epoch as we lack the boundedness of $\|\bphi_V(\cdot,\cdot)\|$,
which holds for tabular SSP naturally since at most $|\cS||\cA| \max_{s\in \cS, a\in \cA}n(s,a)$ 
steps suffice to double $n(s,a)$ for at least one pair $(s,a)$ by the pigeonhole principle. To address this issue, we introduce an extra triggering criterion: $t \geq 2 t_j$. 
It turns out that despite of being extremely simple this criterion endows the algorithm with several nice properties. 
First, together with the \texttt{DEVI} error parameter $\epsilon_j = 1/t_j$, we can bound the cumulative error from value iterations in epoch $j$ by a constant, i.e., $(2 t_j-t_j) \cdot \epsilon_j = 1$. 
Second, it will not increase the total number of epochs since the time step doubling can happen at most $\cO(\log T)$ times, which is consistent with the first criterion.
These two properties together enable us to bound the total error from value iteration by $\cO(\log T)$. 
Finally, this criterion is fairly easy to implement and has negligible time and space complexity. 

In summary, our two updating criteria reflect a fine-grained characterization of the extent of exploration by coupling the  feature space and time interval. 

\subsection{Optimistic Planning: Contraction via Perturbation}  
The optimism of Algorithm \ref{alg: linear kernel ssp} is realized by the construction of the confidence set $\cC_j$ (Line \ref{algline: def t_j}), which is fed into the \texttt{DEVI} subroutine. We now describe the estimation of the Q-function in the \texttt{DEVI} sub-routine (Algorithm \ref{alg: EVI}). 
\texttt{DEVI} requires the access to a confidence region $\cC_j$ that contains the true model parameter $\btheta^*$ with high probability (Line~\ref{algline: confidence ellipsoid}). 
We construct the confidence region $\cC_j$ as an ellipsoid centered at $\hat\btheta_j$ (Line~\ref{algline: hat btheta}) which can be viewed as an estimate of $\btheta^*$.
The radius of the confidence region $\cC_j$ is specified by a parameter $\beta_t$~(Line \ref{algline: confidence ellipsoid}).
Since not every $\btheta\in \cC_j$ defines a valid probability transition, we further take the intersection between $\cC_j$ and a constraint set $\cB$ defined as
\vspace{-0.2cm}
\begin{align*}
    \cB \coloneqq \Big\{\btheta: \ & \forall (s,a), \langle \bphi(\cdot|s,a) , \btheta \rangle \ \textnormal{is a probability} \\ &\textnormal{distribution}  \ \textnormal{and} \ \langle \bphi(s'|g,a), \btheta \rangle = \ind\{s'=g\}  \Big\} . \notag
\end{align*} 
Then $\cC_j \cap \cB$ is still a confidence region containing the true model parameter $\btheta^*$ with high probability since $\btheta^*\in\cB$. 
Algorithm \ref{alg: EVI} requires two additional inputs: the optimality gap $\epsilon_j$ and the discount factor $q$.
The use of $\epsilon_j$ is standard, while the use of the discount factor is new and the key to ensuring convergence of \texttt{DEVI}.

Specifically, \eqref{eq: EVI value iteration} in Algorithm \ref{alg: EVI} repeatedly performs one-step value iteration by applying the Bellman operator to the set $\cC_j \cap \cB$. 
This is motivated by the Bellman optimality equation in \eqref{eq: bellman optimal condition}, and uses $\min_{\btheta \in \cC\cap \cB} \langle \btheta, \bphi_{V^{(i)}} \rangle$ as an optimistic estimate for $\PP V^*$. 
However, using this estimate alone cannot guarantee the convergence of \texttt{DEVI} because $\langle \cdot, \bphi_{V^{(i)}} \rangle$ is not a contractive map, which holds for free in the discounted setting \citep{jaksch2010near,zhou2021provably}, but not in SSPs. 
More specifically, in the \texttt{EVI} algorithm for the discounted MDP (e.g., Algorithm 2 in \citep{zhou2021provably}), there is an intrinsic discount factor $0<\gamma<1$, which ensures that the Bellman operator is a contraction. Consequently, the value iteration converges within a finite number of iterations. In contrast, the Bellman equation of SSP does not have such a discount factor.
To address this issue, in \eqref{eq: EVI value iteration}, we introduce a $1-q$ discount factor to ensure the contraction property. 
Although this causes an additional bias to the estimated transition probability function, we can alleviate it by choosing $q$ properly.
In particular, for each epoch $j$ we set $q = 1/t_j$ (Line \ref{algline: discount factor 1 over t_j}), and we can show that this bias will only cause an additive term of order $\cO(\log T )$ in the final regret bound.

Besides the convergence guarantee, the $1-q$ discount factor also brings an additional benefit that it biases the estimated transition kernel towards the goal state $g$, further encouraging optimism. 
Similar design can also be found in the \texttt{VISGO} algorithm proposed by \citet{tarbouriech2021stochastic}. 
The intuition behind is to guarantee the existence of proper policies under the estimated transition probability function. 
As a result, the output of the value iteration, which solves $V = \tilde\cL V$ approximately for the Bellman operator $\tilde\cL$ induced by the estimated transition,
can induce a greedy policy that is proper under the estimated transition. 

The main computational overhead of \texttt{LEVIS} is from \texttt{DEVI}, which is quite efficient to implement. We discuss this in Appendix~\ref{sec: discuss DEVI computation}.

\section{Regret Bound of \texttt{LEVIS}}\label{sec: main results}
We present the main theoretical results for Algorithm \ref{alg: linear kernel ssp} by giving regret upper bounds for both positive and general cost functions.

\subsection{Upper Bound: Positive Cost Functions}\label{sec: positive cost}

We first consider a special case where the cost is strictly positive (except for the goal state $g$). 
\begin{assumption}\label{assump: c_min}
We assume there exists an \emph{unknown} constant $c_{\min}\in(0,1)$ such that $c(s,a)\geq c_{\min}$ for all $s\in \cS\setminus \{g\}$ and $a\in \cA$. 
\end{assumption}
Let $T$ be the total number of steps in Algorithm \ref{alg: linear kernel ssp}, then the above assumption allows us to lower bound the total cumulative cost after the $K$ episodes by $c_{\min}\cdot T$.
Note that this provides a relation between the deterministic $K$ and the random quantity $T$.
For simplicity, we assume the agent has access to $B$, an order-accurate estimate of $\Bstar$ satisfying $\Bstar \leq B \leq \kappa \Bstar$ for some unknown constant $\kappa \geq 1$. 
Similar assumptions have also been imposed in previous works \citep{tarbouriech2021stochastic,vial2021regret}.


\begin{theorem}\label{thm: regret upper bound with c_min}
    Under Assumptions \ref{assump: linear kernel mdp}, \ref{assump: existence proper policy} and \ref{assump: c_min}, for any $\delta>0$, let $\rho=0$ and $\beta_t = B \sqrt{d \log \left( 4(t^2+t^3 B^2 /\lambda)/\delta \right)} + \sqrt{\lambda d}$ for all $t\geq 1$, where $B \ge \Bstar$ and $\lambda \geq 1$.
    Then with probability at least $1-\delta$, the regret of Algorithm \ref{alg: linear kernel ssp} satisfies
    \begin{align}\label{eq:main_thm_bound}
        R_K = \cO\bigg( & B^{1.5} d\sqrt{ K/c_{\min}} \cdot \log^2\left( \frac{ K B d }{c_{\min} \delta} \right) \notag 
        \\ & + \frac{B^2 d^2}{c_{\min}} \log^2\left( \frac{ K B d }{c_{\min} \delta} \right) \bigg). 
    \end{align}
    \vspace{-0.5cm}
\end{theorem} 
If $B=O(\Bstar)$, Algorithm \ref{alg: linear kernel ssp} attains an $\tilde O({\Bstar}^{1.5}d\sqrt{K/c_{\min}})$ regret.
The dominating term in~\eqref{eq:main_thm_bound} has a dependency on $1/c_{\min}$. 
For the tabular SSP, \citet{ cohen2021minimax,jafarnia2021online, tarbouriech2021stochastic} avoid such a dependency by using a more delicate analysis. 
However, it remains an open question whether a similar result can be achieved for the linear mixture SSP.
\begin{remark}
Set the parameter $\delta$ in Theorem \ref{thm: regret upper bound with c_min} as $\delta=1/K$ and define the high probability event $\Omega$ as Theorem~\ref{thm: regret upper bound with c_min} holds. 
Then, we can obtain the expected regret bound:
\begin{align*}
    \EE[R_K]&\leq \EE\big[R_K|\Omega\big]\Pr[\Omega]+K \Pr[\bar{\Omega}]\notag\\
    &=\cO\bigg( B^{1.5} d\sqrt{ K/c_{\min}} \cdot \log^2\left( \frac{ K B d }{c_{\min} } \right) 
    \\ & \qquad \quad + \frac{B^2 d^2}{c_{\min}} \log^2\left( \frac{ K B d }{c_{\min} } \right) \bigg),
\end{align*}
which implies an $\tilde O({\Bstar}^{1.5}d\sqrt{K/c_{\min}})$ expected regret.
\end{remark}
The proof of Theorem \ref{thm: regret upper bound with c_min} is in Appendix \ref{sec:proof_specific_cost}.




\subsection{Upper Bound: General Cost Functions}\label{sec: general cost}

Without Assumption \ref{assump: c_min}, an $\tilde O(K^{2/3})$ regret can be achieved by running Algorithm \ref{alg: linear kernel ssp} with $\rho = K^{-1/3}$.


\begin{theorem}\label{thm: corollary general cost}
Under Assumptions \ref{assump: linear kernel mdp} and \ref{assump: existence proper policy}, for any $\delta>0$, let $\rho=K^{-1/3}$ and $\beta_t = B \sqrt{d \log \left( 4(t^2+t^3 B^2 /\lambda)/\delta \right)} + \sqrt{\lambda d}$ for all $t\geq 1$, where $B \ge \Bstar $ and $\lambda \geq 1$. 
Then with probability at least $1-\delta$, the regret of Algorithm \ref{alg: linear kernel ssp} satisfies 
\begin{align*}
    R_K = \cO \left( \tilde{B}^{1.5} d K^{2/3} \cdot \chi + \Tstar K^{2/3} + \tilde{B}^2 d^2 K^{1/3} \cdot \chi \right),
\end{align*}
where $\tilde{B} = B + \Tstar/K^{1/3}$ and $\chi = \log^2\big( (B+\Tstar )Kd/\delta \big)$.
\end{theorem}
In Theorem \ref{thm: corollary general cost}, the regret depends on $\tilde{B}$ instead of $B$. 
Note that $\tilde{B}$ is approximately equal to $\Bstar$ when $K = \Omega(\Tstar^3)$ and $B=O(\Bstar)$. 
Here $\Tstar$ is defined in Section \ref{sec: prelim} as the maximum expected time it takes for the optimal policy to reach the goal state starting from any state.

The cost perturbation $\rho$ is a common trick to deal with the case of general cost functions in the SSP literature \citep{tarbouriech2020no, cohen2020near, tarbouriech2021stochastic}. Similar to \citet{tarbouriech2020no}, the term $c_{\min}^{-1}$ is multiplicative with $K$ in our regret bound given by Theorem \ref{thm: regret upper bound with c_min}, leading to an $\tilde{\cO}(K^{2/3})$ regret in the case of general cost functions. 
Similarly, the regret bound for linear SSP in \citet{vial2021regret} also has a multiplicative $c_{\min}^{-1}$. Some later work on tabular SSP \citep{cohen2020near, tarbouriech2021stochastic} has shown that it is possible to make the term $c_{\min}^{-1}$ additive and improve the regret to $\tilde{\cO}(K^{1/2})$ for general cost functions. How to get an additive $c_{\min}^{-1}$ term for the linear mixture SSP is an interesting future direction. 

For the choice of the other parameters in Algorithm \ref{alg: linear kernel ssp}, by Theorems \ref{thm: regret upper bound with c_min} and \ref{thm: corollary general cost}, we can set $\lambda = 1$ in both the positive and general cost cases. 
It is also not uncommon to assume a known upper bound $B \geq \Bstar$ in existing SSP literature \citep{cohen2021minimax, vial2021regret}. 
While is possible to deal with unknown $B$ with a doubling trick for tabular SSP \citep{rosenberg2020near, tarbouriech2021stochastic}, it remains an open question for the linear SSP setting.

\section{An Improved Algorithm with Bernstein-type Bonus}\label{sec: bernstein type algorithm}

Despite its simple form, the major drawback of Algorithm \ref{alg: linear kernel ssp} is that it only uses Hoeffding-type confidence sets, which possibly costs too much exploration as it ignores the variance information. 
Consequently, in the regret bound in Theorem~\ref{thm: regret upper bound with c_min}, the dependence on $\Bstar$ is not optimal.
This is comparable to the situation for episodic linear mixture MDPs \citep{jia2020model,ayoub2020model}, where the size of Hoeffding-type confidence sets loosely scales with the horizon length $H$. It has been show that Bernstein-type bonus can sharpen the dependence on $H$ for both tabular MDPs \citep{lattimore2012pac, azar2017minimax,jin2018q, zhang2019regret,zhang2020almost,he2021nearly} and linear mixture MDPs \citep{zhou2020nearly,wu2021nearly,he2021nearly,zhang2021improved}. 
Therefore, to improve over the previous algorithm, we further incorporate the variance information and develop a improved algorithm called $\texttt{LEVIS}^{\texttt{+}}$~(i.e., Algorithm~\ref{alg: linear kernel ssp bernstein}).

\begin{algorithm}[ht!]
	\caption{$\texttt{LEVIS}^{\texttt{+}}$}
	\label{alg: linear kernel ssp bernstein}
	\begin{algorithmic}[1]
	\STATE {\bfseries Input:} regularization parameter $\lambda$, confidence radius $\{\hat\beta_t, \ \check\beta_t, \ \tilde\beta_t \}$, cost perturbation $\rho \in [0,1]$, an estimate $B \geq \Bstar$
	\STATE {\bfseries Initialize:} set $t\leftarrow 1$, $j\leftarrow 0$, $t_0=0$, $\bSigma_0 \leftarrow \lambda \Ib$, $\tilde\bSigma_0 \leftarrow \lambda \Ib$, $ \bbb_0 \leftarrow \zero$, $ \tilde\bbb_0 \leftarrow \zero$, $\hat\btheta_0 \leftarrow 0$, $\tilde\btheta_0 \leftarrow 0$
	, $Q_0(s,\cdot), V_0(s) \leftarrow 1$ for all $s \neq g$ and $0$ otherwise  
	\FOR{$k=1,\dots,K$}
	\STATE Set $s_t = \sinit$ 
	\WHILE{$s_t \neq g$} 
	\STATE Take action $a_t = \argmin_{a\in\cA} Q_j(s_t,a)$, receive cost $c_t = c(s_t,a_t)$ and next state $s_{t+1}\sim \PP(\cdot|s_t,a_t)$ \alglinelabel{algline: take action and observe bernstein}
	\STATE Set $[\hat{\VV}_t V_j](s_t, a_t)$ and $E_t$ by \eqref{eq: VV and E} \alglinelabel{algline:variance}
	\STATE Set $\hat{\sigma}_t^2 \leftarrow \max\{B^2/d, [\hat{\VV}_t V_j](s_t, a_t) + E_t \}$\alglinelabel{algline: hat sigma}
	\STATE Set $\bSigma_{t} \leftarrow \bSigma_{t-1}+\hat{\sigma}_t^{-2} \bphi_{V_j}(s_t,a_t)\bphi_{V_j}(s_t,a_t)^\top$ \alglinelabel{algline: weighted hat bSigma}
	\STATE Set $\bbb_t \leftarrow \bbb_{t-1} +  \hat{\sigma}_t^{-2} \bphi_{V_j}(s_t,a_t) V_j(s_{t+1})$ \alglinelabel{algline: weighted hat b}
	\STATE Set $\hat{\btheta}_t \leftarrow \bSigma_t^{-1} \bbb_t $ \alglinelabel{algline: weighted hat btheta}
	\STATE Set $\tilde\bSigma_{t} \leftarrow \tilde\bSigma_{t-1}+ \bphi_{V_j^2}(s_t,a_t)\bphi_{V_j^2}(s_t,a_t)^\top$ \alglinelabel{algline: weighted tilde bSigma}
	\STATE Set $\tilde\bbb_t \leftarrow \tilde\bbb_{t-1} + \bphi_{V_j^2}(s_t,a_t) V_j^2(s_{t+1})$ 
	\STATE Set $\tilde{\btheta}_t \leftarrow \tilde\bSigma_t^{-1} \tilde\bbb_t $ \alglinelabel{algline: tilde theta}
	\IF{$\det (\bSigma_t)\geq 2 \det(\bSigma_{t_j})$ \textbf{or} $t \geq 2 t_j$ } \alglinelabel{algline:bernstein_criterion}
	\STATE Set $j \leftarrow j+1$ 
	\STATE Set $t_j \leftarrow t$,  $\epsilon_j \leftarrow \frac{1}{t_j}$, $q_j \leftarrow \frac{1}{t_j}$ \alglinelabel{algline: def t_j bernstein}
    \STATE Set confidence ellipsoid $\hat\cC_j \leftarrow \left\{ \btheta: \ \| \bSigma_{t_j}^{1/2} (\btheta - \hat{\btheta}_{t_j})\|_2 \leq \hat\beta_{t_j} \right\}$ \alglinelabel{algline:confidence_set_bernstein}
	\STATE Set $Q_j(\cdot,\cdot) \leftarrow \texttt{DEVI}(\hat\cC_{t_j}, \epsilon_j, q_j,\rho)$ 
	\STATE Set $V_j (\cdot) \leftarrow \min_{a\in\cA} Q_j (\cdot,a)$ \alglinelabel{algline: bernstein V is min Q}
	\ENDIF
	\STATE Set $t\leftarrow t+1$ 
	\ENDWHILE
	\ENDFOR
	\end{algorithmic}
\end{algorithm}

\subsection{The $\texttt{LEVIS}^{\texttt{+}}$ Algorithm}

Our $\texttt{LEVIS}^{\texttt{+}}$ algorithm is presented in Algorithm~\ref{alg: linear kernel ssp bernstein}.
Compared with the previous \texttt{LEVIS} algorithm,  $\texttt{LEVIS}^{\texttt{+}}$ shares the same structure but employs a more complicated estimation procedure which involves estimation of the variance of the value function (Lines~\ref{algline:variance} to \ref{algline: hat sigma}).
Based on the estimated variance, we then apply weighted ridge regression to obtain estimate $\hat\btheta_j$ of the model parameter (Lines~\ref{algline: weighted hat bSigma} to \ref{algline: weighted hat btheta}),
where each data point $(s_t, a_t, s_{t+1})$ is weighted by  $\hat\sigma_t^2$
which upper bounds the conditional variance of $V_t(s_{t+1})$ given $(s_t, a_t)$.
In contrast to the (unweighted) ridge regression used in Algorithm~\ref{alg: linear kernel ssp}, here we have a tighter concentration of $\hat\btheta_j$ around the true parameter $\btheta^*$ as long as $\hat\sigma_t^2$ improves upon the crude upper bound $B^2$, and hence obtain a tighter confidence set. 
Intuitively, the variance of the state-action pairs can be viewed as a surrogate measure of the data quality. Hence by weighting the data points using their (estimated) variance in the regression, the algorithm is able to learn the model more efficiently. 

Due to the space limit, we give a detailed introduction of the algorithm design in Appendix~\ref{sec: detail of bernstein design}.

\subsection{Regret Bound of $\texttt{LEVIS}^{\texttt{+}}$}

We now introduce the theoretical result for Algorithm \ref{alg: linear kernel ssp bernstein}. 
The following theorem establishes the regret upper bound of Algorithm~\ref{alg: linear kernel ssp bernstein}.
Compared to our lower bound in Theorem~\ref{thm: lower bound}, it indicates that Algorithm~\ref{alg: linear kernel ssp bernstein} achieves a near-optimal regret with  an appropriate choice of the parameters $\{\check\beta_t, \tilde\beta_t, \hat\beta_t\}_{t \geq 1}$. 

Specifically, one should choose $\check\beta_t = \tilde\cO\left( d \right)$, $\tilde\beta_t = \tilde\cO( \sqrt{dB^4})$, and $\hat{\beta}_t = \tilde\cO(\sqrt{d})$, where $\tilde\cO(\cdot)$ hides logarithmic terms.
The detailed choice of $\{\check\beta_t, \tilde\beta_t, \hat\beta_t\}_{t \geq 1}$ is given by \eqref{eq: choices of 3 betas} in Appendix~\ref{sec: detail of bernstein design}.

\begin{theorem}\label{thm: bernstein regret upper bound with cmin}
    Under Assumptions \ref{assump: linear kernel mdp}, \ref{assump: existence proper policy} and \ref{assump: c_min}, for any $\delta >0$, let $\rho=0$, $\lambda = 1/B^2$ and $\{\check\beta_t, \tilde\beta_t, \hat\beta_t\}_{t \geq 1}$ be given by \eqref{eq: choices of 3 betas}. Then with probability at least $1-7\delta$, for any sufficiently large $K$, the regret of Algorithm \ref{alg: linear kernel ssp bernstein} satisfies
    \begin{align*}
        R_K = \tilde\cO\left( d^2B + dB\sqrt{K} + \sqrt{d}B^{1.5} \sqrt{\frac{K}{c_{\min}}}\right) , 
    \end{align*} where $\tilde\cO(\cdot)$ hides a factor polynomial in $\log(KB/(\lambda \delta c_{\min}))$. 
\end{theorem}


\begin{remark}
Compared to the confidence sets $\{\hat\cC_j\}_{j\in[J]}$ used in Algorithm \ref{alg: linear kernel ssp bernstein}, the confidence sets $\{\cC_j\}_{j\in[J]}$ in Algorithm \ref{alg: linear kernel ssp} are too conservative, as they only use the crude upper bound $B^2$ on the variance of the estimated value functions that appear in the algorithm. In fact, the variance of these functions can be significantly smaller than the crude upper bound. As a result, by using the variance information of the data, $\{\hat\cC_j\}_{j\in[J]}$ are tighter confidence sets and still contain the true model parameter with high probability.
\end{remark}

\begin{remark}
From Theorem \ref{thm: bernstein regret upper bound with cmin}, we can see that if $d \geq B$ and $B$ is an order-accurate estimate of $\Bstar$, i.e., $B = \cO(\Bstar)$, then the regret can be simplified to $\tilde\cO(d\Bstar\sqrt{K/c_{\min}})$, matching the lower bound in Theorem \ref{thm: lower bound} up to $1/\sqrt{c_{\min}}$ and poly-logarithmic factors. As a comparison, the concurrent result for linear mixture SSPs in \citet{chen2021improved} achieves an $\tilde\cO(\Bstar\sqrt{d T_{\star} K} + \Bstar d \sqrt{K})$ regret, where $T_{\star}$ is the expected time it takes the optimal policy to reach the goal state maximized over any initial state.
    Using the fact that $T_{\star}\leq \Bstar/c_{\min}$ and Theorem \ref{thm: lower bound}, their bound is near-optimal when $d\geq T_{\star}$. 
\end{remark}

\section{Lower Bound}\label{sec: lower bound}
We provide a hardness result for learning linear mixture SSPs by proving a lower bound for the expected regret suffered by any deterministic learning algorithm. For this purpose, we construct a class of hard-to-learn SSP instances with $|\cS|=2$ and a large action set, and show that any deterministic algorithm will at least suffer from one instance in the class. The detail is deferred to Appendix \ref{sec: lower bound details}.

\begin{theorem}\label{thm: lower bound}
    Under Assumption \ref{assump: linear kernel mdp}, suppose $d \geq 2$, $\Bstar \geq 2$ and $K > (d-1)^2/2^{12}$. 
    Then for any possibly non-stationary history-dependent policy $\pi$, there exists a linear mixture SSP instance with parameter $\btheta^*$ such that $\EE_{\pi, \btheta^*} \left[ R_K \right] \geq d \Bstar \sqrt{K}/1024$.
\end{theorem}

\begin{remark}
The expectation in the lower bound is over the trajectory induced by policy $\pi$ in the SSP instance parameterized by $\btheta^*$. 
Note that we allow the policy $\pi$ to be non-stationary and history-dependent. 
This is equivalent to assuming a deterministic learning algorithm, which is sufficient for establishing a lower bound \citep{cohen2020near}. 
\end{remark}
\begin{remark}
Our instance for the lower bound can be also adapted to a linear SSP instance \citep{vial2021regret,chen2021improved}, which yields a same $\Omega(d\Bstar\sqrt{K})$ lower bound. (See Remark \ref{rmk:linear_lower_bound} for a detailed discussion.)

\end{remark}

\section{Overview of the Proof Technique}\label{sec: proof sketch}

We explain the proof of the main results in Section \ref{sec: main results}. 
We first present a unifying master theorem that underlies Theorems \ref{thm: regret upper bound with c_min} and \ref{thm: corollary general cost}, and then discuss the technical difficulty and corresponding proof techniques to deal with it.

\subsection{A Master Theorem}
The following master theorem serves as a backbone for proving the regret bound for specific classes of cost functions.

\begin{theorem}\label{thm: regret upper bound with T}
    Under Assumptions \ref{assump: linear kernel mdp} and \ref{assump: existence proper policy}, for any $\delta>0$, let $\rho=0$ and $\beta_t = B \sqrt{d \log \left( 4(t^2+t^3 B^2 /\lambda)/\delta \right)} + \sqrt{\lambda d}$ for all $t\geq 1$, where $B \geq \Bstar$ and $\lambda \geq 1$. 
    Then with probability at least $1-\delta$, the regret of Algorithm \ref{alg: linear kernel ssp} satisfies
    \begin{align*}
        R_K & \leq  6 \beta_T \sqrt{dT \log\left( 1 + \frac{T\Bstar^2}{\lambda}\right)}  + 7d\Bstar \cdot \iota , 
    \end{align*}
   where $T$ is the total number of steps, $\iota = \log( T + \frac{T^2 \Bstar^2 d}{\lambda})$.
\end{theorem} 

Theorem \ref{thm: regret upper bound with T} gives an $\tilde \cO(\sqrt{T})$ regret upper bound with respect to the total number of steps $T$.
However, for SSP problems, the horizon of each episode is unknown and $T$ can be far greater than $K$, so we need to translate the dependence on $T$ to $K$. 
For positive cost functions, this can be achieved by using Assumption \ref{assump: c_min}, which imposes an implicit constraint that the total cost of an episode is bounded by $c_{\min}$ times the length of the episode.
On the other hand, for general cost functions, it suffices to pick $\rho = K^{-1/3}$ and we  can then apply the same reasoning as in the previous case.
Please see Appendix \ref{sec:proof_specific_cost} for detailed calculations.

We discuss the proof of this master theorem below, which is one of our main technical contributions.

\subsection{Proof Sketch of Theorem \ref{thm: regret upper bound with T}}

\textbf{Regret Decomposition}
In our analysis, instead of dealing with the regret in \eqref{eq: regret def} directly, we first implicitly group the time steps into intervals and index them by $m=1,\ldots,M$ as in Lemma \ref{lem: regret decomposition}. The basic idea here is to group all the time steps into disjoint intervals of which the end points are either the end of an episode or the time steps when the \texttt{DEVI} subroutine is called\footnote{The interval decomposition is implicit and 
only for the purpose of analysis. This is in contrast to the epoch decomposition, which is explicit and indexed by $j$ in Algorithm \ref{alg: linear kernel ssp}. The difference is that an epoch ends when \texttt{DEVI} is called, while an interval ends when either \texttt{DEVI} is called or the goal state $g$ is reached (i.e., an episode ends). }. The purpose of such a decomposition is to guarantee that within each interval the optimistic action-value function  remains the same and so the induced policy. 
This has been used in several existing works on SSP \citep{rosenberg2020near,rosenberg2020stochastic,tarbouriech2021stochastic} as well. 


    

\begin{lemma}
\label{lem: regret decomposition}
Assume the event in Lemma \ref{lem: confidence set optimism} holds, then the following holds for the regret defined in~\eqref{eq: regret def}\footnote{$R(M)$ is the same as $R_K$. We use a different notation to emphasize the interval decomposition.}: 
\begin{align}\label{eq: regret decomposition}
    & R(M) \notag
    \\ &\leq \underbrace{\textstyle{\sum_{m=1}^M \sum_{h=1}^{H_m}} \left[ c_{m,h} + \PP V_{j_m}(s_{m,h},a_{m,h}) - V_{j_m} (s_{m,h}) \right]}_{E_1}  \notag 
    \\ & \quad \ + \underbrace{\textstyle{\sum_{m=1}^M \sum_{h=1}^{H_m}} \left[ V_{j_m}(s_{m,h+1}) - \PP V_{j_m} (s_{m,h},a_{m,h}) \right]}_{E_2} \notag
    \\ & \quad \ + \textstyle{\big[ \sum_{m=1}^M\big( \sum_{h=1}^{H_m} V_{j_m}(s_{m,h}) - V_{j_m}(s_{m,h+1}) \big)} \notag
    \\ & \qquad \ \underbrace{ \qquad   -\textstyle{\sum_{m \in \cM(M)} V_{j_m}(\sinit) \big]} \qquad \qquad \quad }_{E_3} + 1,
\end{align} 
where $j_m=j$ is the index of the value function estimate $V_j$ used in the $m$-th interval\footnote{This is well-defined since the same $V_j$ is used at all time steps within one interval.} and $H_m$ is the length of the $m$-th interval.
\end{lemma}

\textbf{Bounding $E_1$.} 
Controlling term $E_1$ is the essential and most difficult part.
Roughly speaking, $E_1$ is the accumulated Bellman error of the \texttt{DEVI} outputs $Q_j(\cdot,\cdot)$ on the sample state-action trajectory. 
The ordinary method is to bound the sum of width of the confidence regions $C_j$. 
However, we face unique challenges. 
Recall from \eqref{eq: EVI value iteration} that our value iteration sub-routine \texttt{DEVI} includes a transition bonus $q$ for the sake of convergence. 
Such bonus leads to a biased Bellman operator, and the biases would accumulate over time and become an additive term in the regret bound of the order $\cO(\sum_{m} q_{j_m} H_m)$. 
To bound the bias, we need to bound $H_m$ and choose an appropriate $q_{j_m}$. 

For $H_m$, we have $H_m = \cO( t_{j_{m+1}} - t_{j_m})$, where $t_{j_{m+1}} - t_{j_m}$ is the length of epoch $j$, i.e., the number of time steps between two consecutive \texttt{DEVI} calls. 
As mentioned in Section~\ref{sec: updating criteria in algorithm}, while the classical determinant-based criterion alone cannot guarantee finite epoch length, the additional time step doubling criterion fixes this issue by enforcing $t_{j_{m+1}} \leq 2 t_{j_m}$. 
Furthermore, by picking each $q_j = \frac{1}{t_j}$, the total bias can be upper bounded as 
\begin{align*}
    \textstyle{\cO\left( \sum_{j=1}^J t_j^{-1} \cdot ( 2 t_j - t_j) \right) = \cO(J)}.
\end{align*}
In this way, it suffices to bound the total number of calls to \texttt{DEVI} to bound the accumulative bias. 

For the rest part of $E_1$, we can show that every time  when \texttt{DEVI} is called, the output is an optimistic action-value function estimator with high probability (by Lemma \ref{lem: confidence set optimism}). 
Finally, we need to bound the total difference between the estimated functions and the optimal action-value function. 
This follows from the elliptical potential lemma and the determinant-based doubling criterion. 
The details of bounding $E_1$ is deferred to Appendix \ref{sec: E1 bound}. 

\textbf{Bounding $E_2$ and $E_3$.} Since $E_2$ is the sum of a martingale difference sequence, it can be bounded by $\cO(\sqrt{T \log(T/\delta)})$ using standard martingale concentration inequality. 

The term $E_3$ can be transformed into a telescoping sum under the interval decomposition. After the transformation, it can be shown that only those terms with $j_m \neq j_{m+1}$ would contribute to the final regret. Since the number of intervals with $j_m \neq j_{m+1}$ are at most $\cO(J)$, the problem again reduces to bounding the number of \texttt{DEVI} calls.

\textbf{Analysis of \textnormal{\texttt{DEVI}}.} By the algorithmic design we elaborated in Section \ref{sec: algorithm}, \texttt{DEVI} guarantees optimism and finite-time convergence, as summarized in Lemma~\ref{lem: confidence set optimism} below. 

\begin{lemma}\label{lem: confidence set optimism}
For all $t\geq 1$, let $\rho=0$ and $\beta_t = B \sqrt{d \log \left( 4(t^2+t^3 B^2 /\lambda)/\delta \right)} + \sqrt{\lambda d}$, where $B \geq \Bstar$. Then with probability at least $1-\delta/2$, for all $j \geq 1$, \texttt{DEVI} converges in finite time, and it holds that $\btheta^* \in \cC_j \cap \cB, 0 \leq  Q_j(\cdot,\cdot) \leq Q^\star(\cdot,\cdot)$, and $0 \leq  V_j(\cdot) \leq V^\star(\cdot)$.
\end{lemma}

Note that in Lemma \ref{lem: confidence set optimism} the optimism only holds for the \texttt{DEVI} output, i.e., $V_j$ for any $j\geq 1$. The initialization $V_0$ in Line \ref{algline: initialization} of the main Algorithm \ref{alg: linear kernel ssp} is not necessarily satisfy optimistic since it is possible that $V^\star(s)< 1$ for some $s$. 
Still, such an initialization satisfies  $\|V_0\|_\infty = 1 \leq \Bstar$, which is sufficient to establish the optimism for $j\geq 1$. 

\section{Conclusions}
In this paper, we propose a novel algorithm for linear mixture SSP and prove its regret upper and lower bounds. 
For future work, there are several important directions. 
First, there is a $\Bstar^{0.5}$ gap between the current upper and lower bounds. 
Second, it remains open to prove an $\tilde{\cO}(\sqrt{K})$ regret bound for linear mixture SSP for general cost functions.

\section*{Acknowledgements}
We thank the anonymous reviewers for their helpful comments. 
JH and QG are partially supported by the National Science Foundation CAREER Award 1906169. The views and conclusions contained in this paper are those of the authors and should not be interpreted as representing any funding agencies.

\bibliography{reference}
\bibliographystyle{icml2022}

\newpage
\appendix
\onecolumn


\section{Additional Discussions}
\subsection{Discussion on the Linear Mixture MDPS}
\label{subsec:linear mixture ssp}
The linear mixture MDP \citep{modi2020sample,ayoub2020model,zhou2021provably} is a commonly considered model for linear function approximation, where one assumes the transition probability function $\PP$ to be a linear mixture of some basis kernels. The linear mixture MDP covers several important MDP models studied in the literature. We briefly discuss them here. 




\begin{example}[Tabular MDPs]\label{exp: tabular}
For a tabular MDP $M(\cS, \cA, \gamma, r, \PP)$ with $|\cS|, |\cA| \leq \infty$, the transition probability kernel can be represented by $|\cS|^2|\cA|$ \emph{unknown} parameters. 
The tabular MDP is a special case of linear mixture MDPs with the  feature mapping $\bphi(s'|s,a) = \eb_{(s,a,s')} \in \RR^d$ and parameter vector $\btheta = [\PP(s'|s,a)]\in\RR^d$, where $d = |\cS|^2|\cA|$ and $\eb_{(s,a,s')}$ denotes the corresponding natural basis in the $d$-dimensional Euclidean space.
\end{example}

\begin{example}[Linear combination of base models, \citealt{modi2020sample}]\label{example:basemodel}
For an MDP $M(\cS, \cA, \gamma, r, \PP)$, suppose there exist $m$ base transition probability kernels $\{p_i(s'|s,a)\}_{i=1}^{m}$, a feature mapping $\bpsi(s,a): \cS \times \cA \rightarrow \Delta^{d'}$ where $\Delta^{d'}$ is a $(d'-1)$-dimensional simplex, and an \emph{unknown} matrix $\Wb \in \RR^{m \times d'} \in [0,1]^{m\times d'}$ such that $\PP(s'|s,a) = \sum_{k=1}^{m} [\Wb\bpsi(s,a)]_k p_k(s'|s,a)$. Then it is a special case of linear mixture MDPs with feature mapping $\bphi(s'|s,a) = \text{vec}(\pb(s'|s,a)\bpsi(s,a)^\top)\in \RR^d$ and parameter vector $\btheta = \text{vec}(\Wb)\in \RR^d$ where $d = md'$, $\text{vec}(\cdot)$ is the vectorization operator, and $\pb(s'|s,a) = [p_k(s'|s,a)] \in \RR^{m}$. 
\end{example}

\begin{example}[linear-factored MDP, \citealt{yang2019sample}] 
For an MDP $M(\cS, \cA, \gamma, r, \PP)$, suppose that there exist feature mappings $\bpsi_1(s,a): \cS \times \cA \rightarrow \RR^{d_1}$ satisfying $\|\bpsi_1(s,a)\|_2 \leq \sqrt{d_1}$, $\bpsi_2(s'): \cS \rightarrow \RR$ satisfying for any $V: \cS \rightarrow [0, R]$, $\|\sum_s V(s) \bpsi_2(s)\|_2 \leq R$ and an \emph{unknown} matrix $\Mb \in \RR^{d_1 \times d_2}$ satisfying $\|\Mb\|_F \leq \sqrt{d_1}$ such that $\PP(s'|s,a) = \bpsi_1(s,a)^\top \Mb\bpsi_2(s')$. Then it is a special case of linear mixture MDPs with  feature mapping $\bphi(s'|s,a) = \text{vec}\big(\bpsi_2(s')\bpsi_1(s,a)^\top\big) \in \RR^d$ and parameter vector $\btheta = \text{vec}(\Mb) \in \RR^d$, where $d = d_1 d_2$.
\end{example}

For more discussions, please refer to, for example, Section 2 in \citet{ayoub2020model}, or Section 3 in \citet{zhou2021provably}.

\subsection{Computational Complexity of \texttt{DEVI}}\label{sec: discuss DEVI computation}

In \texttt{DEVI}, we need to solve a sequence of optimization problems as given by \eqref{eq: EVI value iteration} and \eqref{eq: EVI V is min Q}. The computational complexity of solving \eqref{eq: EVI value iteration} dominates that of \eqref{eq: EVI V is min Q}. Fortunately, the objective function in \eqref{eq: EVI value iteration} is convex and the constraint set $\cC\cap\cB$ is a convex set, so we can use projected gradient descent~\citep{boyd2004convex} to solve it.

\texttt{DEVI} involves a number of iterations (i.e., the while-loop). The total number of iterations is $\cO(t\log t)$, where $t$ is the time step at which the \texttt{DEVI} is called. To see this, note that according to the stopping criterion $\|V^{(i)} -V^{(i-1)}\|_\infty \geq \epsilon$ and the update rule \eqref{eq: EVI value iteration}, we solve for $n$ such that $(1-q)^n \leq \epsilon$. Then since $\epsilon = 1/t$ and $q = 1/t$, we have $n = \cO(t\log t)$. 

The cost of each \texttt{DEVI} iteration is $\cO(dRB|\cA|)$, where $d$ is the feature dimension, $R$ is the computational cost for calculating $\bphi_V$, and $B$ is the cost for solving an optimization problem of the form $\min_{\btheta\in\cC\cap\cB}\langle \bphi_V, \xb\rangle$.
The optimization problem can be solved by using project gradient descent, as mentioned above. 
To calculate $\bphi_V$, \citet{zhou2021provably} proposed using Monte Carlo sampling to estimate $\bphi_V$ with high accuracy, which turns out to be quite efficient. We refer interested readers to Appendix~B of \citet{zhou2021provably} for more detail.

\subsection{Disscussion on Future Directions}

There are many promising future directions for the SSP problem. We discuss a few. 

First, our current algorithm design is based on the principle of Optimism-in-Face-of-Uncertainty (OFU).  
It is possible to develop Thompson Sampling (TS) type algorithms by following the well-known \texttt{PSRL} algorithm~\citep{osband2013more}.
\citet{jafarnia2021online} studied this topic under the tabular case and gave an $\tilde{\cO}(\Bstar S\sqrt{AK})$ regret bound.
The empirical advantage of TS-type algorithms in this setting is still under-studied. In the bandit setting, \citet{chapelle2011empirical} showed that TS is more robust to delayed feedback and can outperform OFU-type algorithms. In the reinforcement learning setting, such a comprehensive empirical evaluation does not exist. We believe it is an important future research direction. 

Furthermore, most existing works on SSP focus on the online episodic setting. The recent work by \citet{yin2022offline} gives the first theoretical result for SSP in the offline setting. They studied both the policy evaluation problem and the policy learning problem.
While they considered the tabular case, it is natural to extend offline SSP to the linear and general function approximation cases. 


Also, while existing SSP works all focus on the regret minimization setting, it would be very meaningful if we can incorporate other useful factors into this framework. 
One example is the risk-sensitive RL, where the goal is to consider certain risk of the decisions made and minimize the regret simultaneously \citep{fei2020risk, fei2021exponential,jaimungal2022robust, fei2022cascaded,greenberg2022efficient}. This problem setting is very helpful for applications where the risk is a crucial factor, such as finance and AI medicine. 
Another example is the corruption-robust RL, where the goal is to find efficient algorithms in an environment with corrupted reward signals and transitions (e.g., rewards and transitions are picked by an adversary) \citep{lykouris2021corruption,he2022nearly,zhang2022corruption}. 
This setting is very useful because data
corruption is a big threat against the security of
many ML systems.

\section{Numerical Simulations}\label{sec: experiment}

\begin{figure}[ht!]
	\centering
	\begin{subfigure}{0.41\textwidth}
		\centering
		\includegraphics[width=\linewidth]{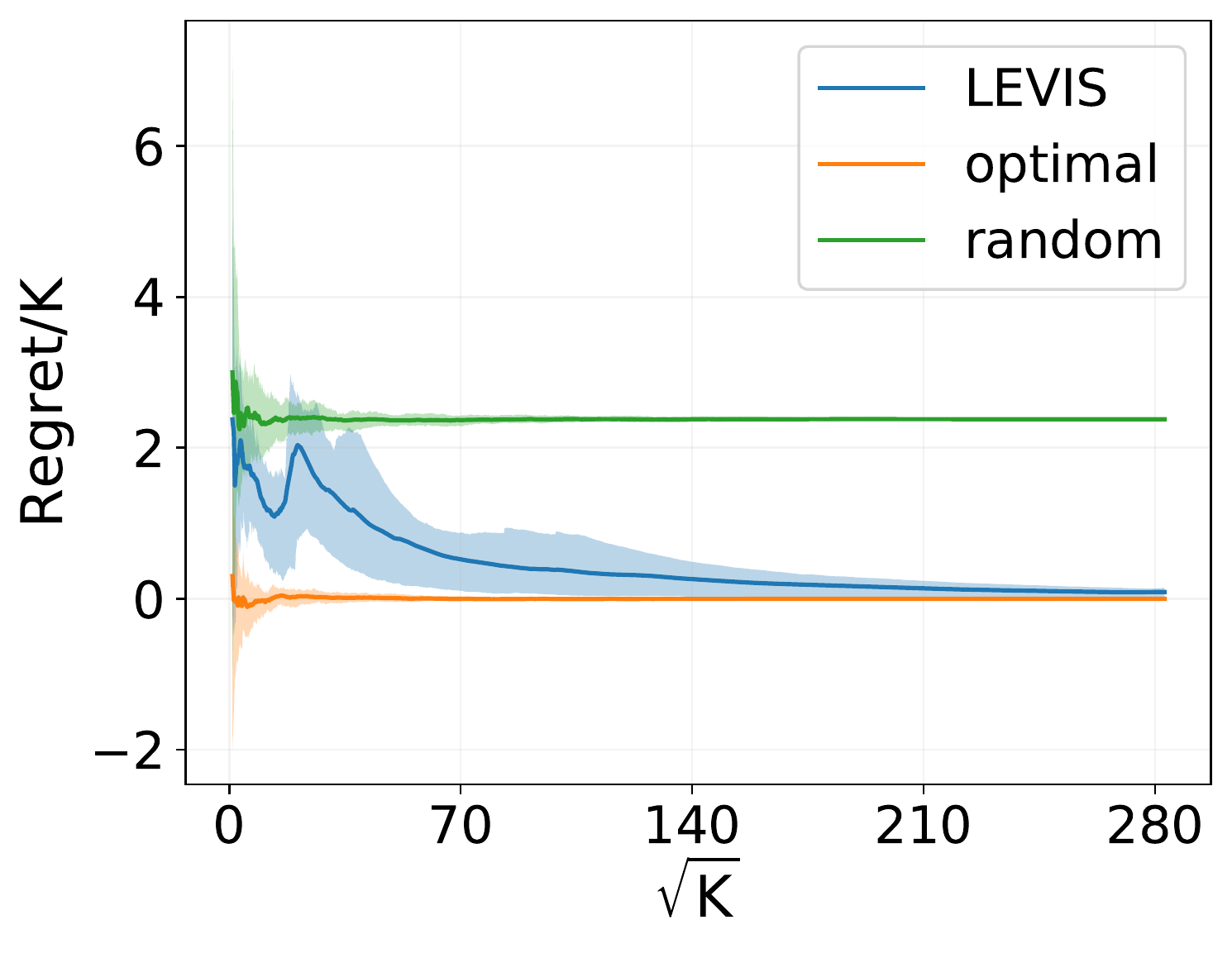}
		\caption{Average regret versus $\sqrt{K}$.}
		\label{fig: sqrt regret 5 3}
	\end{subfigure}
	\begin{subfigure}{0.43\textwidth}
		\centering
		\includegraphics[width=\linewidth]{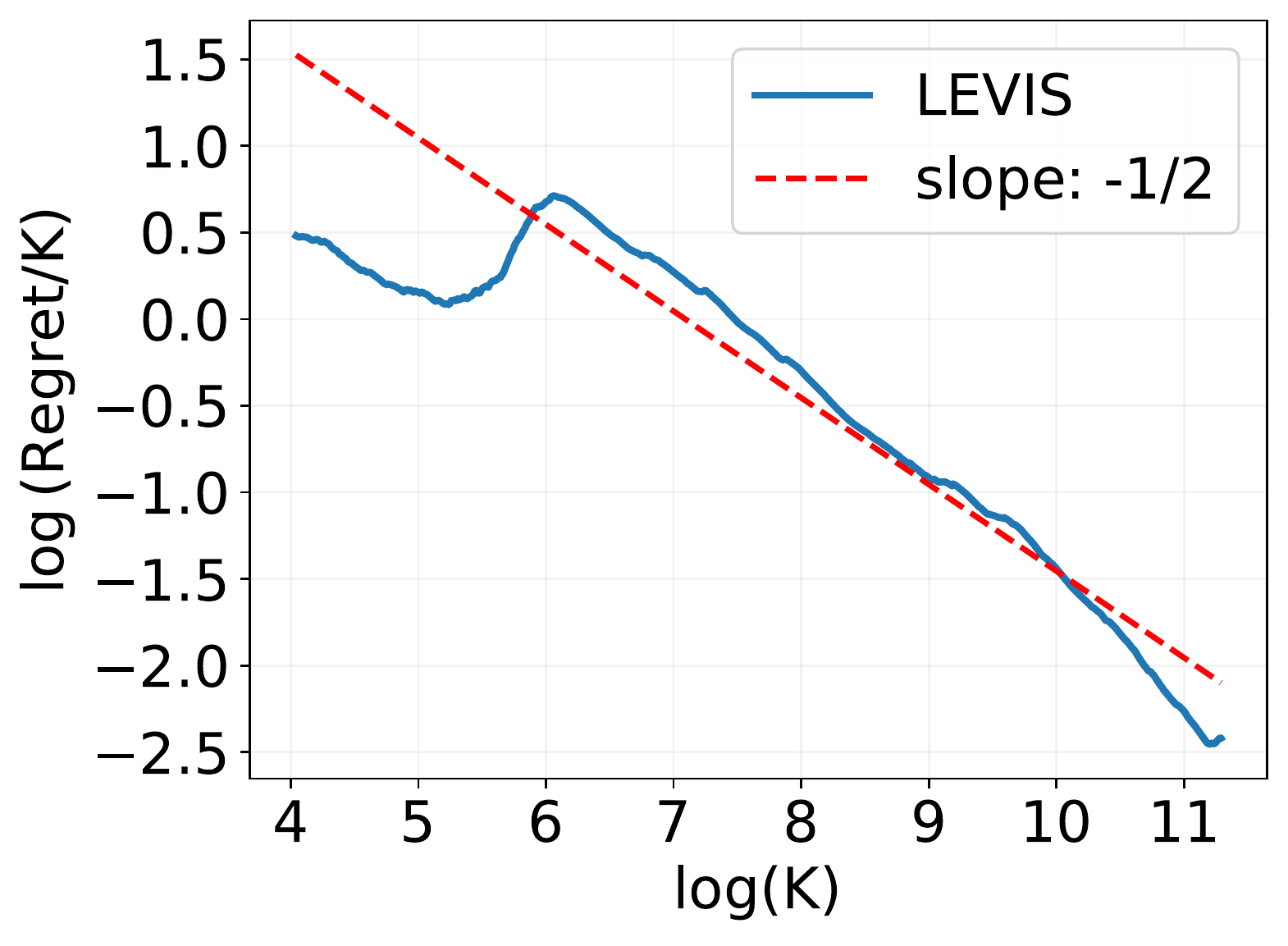}
		\caption{Log-log plot of average regret versus $K$. }
		\label{fig: log log}
	\end{subfigure}
	\caption{The left plot shows the average regret (i.e. $R_K/K$) of implementing Algorithm \ref{alg: linear kernel ssp} on the SSP instance described in Appendix \ref{sec: experiment} with $\lambda=1$, $\rho = 0$ and failing probability $0.01$. The curve is the average of 40 trials. Colored areas indicate empirical [10\%,90\%] confidence intervals. The right plot is the log-log plot of $R_K/K$ and $K$. The red dotted line has a slope equal to $-1/2$. It is clear that the curve has a slope very close to $-1/2$.}\label{fig: regret}
\end{figure}


In this section, we present some results from numerical simulations, which corroborate our theory. 
We construct an SSP instance based on the example used in the proof of the lower bound. Specifically, we have the action space $\cA = \{-1,1\}^{d-1}$ with $|\cA| = 2^{d-1}$. The state space is $\cS = \{\sinit, g\}$. We choose $\delta, \Delta$ and $\Bstar$ such that $\delta + \Delta = 1/\Bstar$ and $\delta > \Delta$. The true model parameter $\btheta^*$ is given by 
\begin{align*}
    \btheta^* = \left[ \frac{\Delta}{d-1},\cdots, \frac{\Delta}{d-1},1 \right]^\top \in \RR^d. 
\end{align*} The feature mapping is defined as 
\begin{align*}
    \bphi(\sinit|\sinit,\ab) & = [-\ab,1-\delta]^\top,
    \\ \bphi(\sinit|g,\ab) & = \mathbf{0},
    \\ \bphi(g|\sinit,\ab) & = [\ab,\delta]^\top,
    \\ \bphi(g|g,\ab) & = [\mathbf{0}_{d-1},1]^\top.
\end{align*}
Here we use $\ab$ instead of $a$ to emphasize that the action is vector-valued. One can verify that this is indeed a linear mixture SSP with the following transition function:
\begin{align*}
    \PP(\sinit|\sinit,\ab) & = 1-\delta - \langle \ab,\btheta \rangle, 
    \\ \PP(g|\sinit,\ab) & = \delta + \langle \ab,\btheta \rangle , 
    \\ \PP(g|g,\ab) & = 1,
    \\ \PP(\sinit|g,\ab) & = 0,
\end{align*}for all $\ab \in \cA$. For more details about this SSP instance, please refer to Appendix \ref{sec: lower bound details}. Note that this is a very hard SSP instance since it is difficult to distinguish between different actions, as we  will later show in the proof of the lower bound. 

The experimental results are shown in Fig. \ref{fig: regret}. We compare the performance of \texttt{LEVIS} with that of the optimal policy and the random policy. Here the optimal policy always chooses $\ab = \mathbf{1}_{d-1}$ to maximize the probability of reaching $g$ from $\sinit$ by the construction of the SSP, and the random policy picks $\ab \in \cA$ uniformly at random. 
We set $d=5$ and $\Bstar = 3$ in the simulation.
In Fig. \ref{fig: sqrt regret 5 3}, we plot the average regret $R_K/K$ versus $\sqrt{K}$. It is evident that \texttt{LEVIS} has a sublinear regret, as opposed to the linear regret of the random policy. 
To further verify that the cumulative regret $R_K$ indeed grows at an $\tilde\cO(\sqrt{K})$ rate, in Fig. \ref{fig: log log}, we make the log-log plot of $R_K/K$ and $K$. 
The red dotted line has a slope equal to $-1/2$. 
We see from Fig.~\ref{fig: log log} that the slope of the curve is very close to $-1/2$. 
This verifies the $\tilde\cO(\sqrt{K})$ regret of \texttt{LEVIS}. These results corroborate our theoretical findings.

\section{Proof of Regret Decomposition}\label{sec: proof o fregret decomposition}

In this section, we prove the regret decomposition given by Lemma \ref{lem: regret decomposition}. 

\begin{proof}[Proof of Lemma \ref{lem: regret decomposition}]

We first explain the details of the interval decomposition. The first interval begin at $t=1$, and an interval ends once either one of the two conditions is met: (1) the \texttt{DEVI} sub-routine is triggered (i.e., either the determinant of the covariance matrix or the time index is doubled); (2) the goal state $g$ is reached, i.e., the current episode ends. We remark that this interval decomposition is only implicit since it is not implemented by the algorithm explicitly. Note that by the two conditions described above, each interval has bounded length almost surely. Indeed, even if the goal state is never reached or the determinant is never doubled due to $\bphi_V$ having small norm, the time step only requires 
the number of iterations to be doubled. 

We index the intervals by $m = 1,2,\cdots$, and denote by $M$ as the total number of intervals, which is possibly infinite. The length of the $m$-th interval is denoted by $H_m$. With a slight abuse of notation, we denote the trajectory for the $m$-th interval as $(s_{m,1},a_{m,1}, \cdots, s_{m,H_m},a_{m,H_m}, s_{m,H_m+1})$, where we have $s_{m,H_m+1} = g$ if interval $m$ ends with condition (2) being met, and $s_{m,H_m+1} = s_{m+1,1}$ otherwise. We denote by $\cM(M)\subseteq[M]$ the set of intervals which are the first interval of their corresponding episodes. We define the mapping $j_m$, such that for each $m\in[M]$, $j_m = j$ is the index of the value function estimate $V_j$ used in the $m$-th interval.

Now let's see how the regret can be expressed under the interval decomposition introduced above. 
The regret can be written as
\begin{align}\label{eq: proof of regret decomposition}
    R(M) &\leq \sum_{m=1}^M \sum_{h=1}^{H_m} c_{m,h} - \sum_{m\in\cM(M)} V_{j_m}(\sinit) + 1 \notag\\ 
    & = \sum_{m=1}^M \sum_{h=1}^{H_m} c_{m,h} + \sum_{m=1}^M\left( \sum_{h=1}^{H_m} V_{j_m} (s_{m,h+1}) - V_{j_m}(s_{m,h}) \right) \notag 
    \\ & \qquad + \sum_{m=1}^M\left( \sum_{h=1}^{H_m} V_{j_m}(s_{m,h}) - V_{j_m}(s_{m,h+1}) \right) - \sum_{m \in \cM(M)} V_{j_m}(\sinit)  + 1 \notag\\ 
    &= \underbrace{\sum_{m=1}^M \sum_{h=1}^{H_m} \left[ c_{m,h} + \PP V_{j_m}(s_{m,h},a_{m,h}) - V_{j_m} (s_{m,h}) \right]}_{E_1}\notag\\
    &\qquad + \underbrace{\sum_{m=1}^M \sum_{h=1}^{H_m} \left[ V_{j_m}(s_{m,h+1}) - \PP V_{j_m} (s_{m,h},a_{m,h}) \right]}_{E_2} \notag 
    \\ & \qquad + \underbrace{ \sum_{m=1}^M\left( \sum_{h=1}^{H_m} V_{j_m}(s_{m,h}) - V_{j_m}(s_{m,h+1}) \right) - \sum_{m \in \cM(M)} V_{j_m}(\sinit) }_{E_3} + 1.
\end{align}

The inequality in the above holds because of the optimism of $V_j$ for $j\geq 1$. Here please note that, since $V_0$ is not the output of \texttt{DEVI}, optimism does not necessarily hold for $V_0$. Therefore, we simply add $1$ at the RHS of the first inequality by the fact that $|V_0|\leq 1$ and the first interval has length equal to $1$ according to the time step doubling updating criterion. 

\end{proof}

\section{Proof for Upper Bounds}\label{sec: proof of intermediate thm in appendix}

In this section we finish the proof of the key result Theorem \ref{thm: regret upper bound with T} by bounding the terms in the regret decomposition in Lemma \ref{lem: regret decomposition}.

\subsection{Bounding $E_1$}\label{sec: E1 bound}

\begin{lemma}\label{lem: E1 bound}
    Assume the event of Lemma \ref{lem: confidence set optimism} holds. Then we have
    \begin{align*}
        & \sum_{m=1}^M \sum_{h=1}^{H_m} \left[ c_{m,h} + \PP V_{j_m}(s_{m,h},a_{m,h}) - V_{j_m} (s_{m,h}) \right] 
        \\ & \leq 4\beta_T \sqrt{2Td \cdot \log\left( 1 + \Bstar^2 T/\lambda \right) } + 5d\Bstar \left[ \log \left( 1 + \frac{T\Bstar^2 d}{\lambda} \right) + \log(T) \right] + 4. 
    \end{align*}
\end{lemma}

\begin{proof}[Proof of Lemma \ref{lem: E1 bound}] 
By Line \ref{algline: take one step} and \ref{algline: V_j equal min Q_j} in the algorithm, for any $m$ and $h$, we have 
\begin{align*}
    V_{j_m}(s_{m,h}) = \min_{a\in\cA} Q_{j_m}(s_{m,h},a) = Q_{j_m}(s_{m,h},a_{m,h}).
\end{align*}Therefore $E_1$ can be rewritten as 
\begin{align}\label{eq: E1 equivalent form}
    E_1 & = \sum_{m=1}^M \sum_{h=1}^{H_m} \left[ c_{m,h} + \PP V_{j_m}(s_{m,h},a_{m,h}) - Q_{j_m} (s_{m,h},a_{m,h}) \right].
\end{align} 
Denote by $\cM_0(M)$ the set of $m$ such that $j_m\geq 1$, i.e., $\cM_0(M) = \{m\leq M: \ j_m \geq 1\}$. Then we see that $\cM_0(M)$ is the collection of intervals such that $Q_{j_m}$ is the output of \texttt{DEVI} instead of the initialization $Q_0$. Fix arbitrary $m \in \cM_0(M)$ and $h$.  
Since $Q_{j_m}$ is the output of \texttt{DEVI}, we have $Q_{j_m} = Q^{(l)}$ for some $l$, i.e., the $l$-th iteration in \texttt{DEVI}, and thus $V_{j_m} (\cdot) = \min_{a\in \cA} Q^{(l)} (\cdot,a) = V^{(l)}(\cdot) $. By the design of \texttt{DEVI}, we have 
\begin{align*}
    &Q^{(l)} (s_{m,h},a_{m,h}) 
    \\ & = c_{m,h} + (1-q)\cdot \min_{\btheta \in \cC_{j_m}\cap \cB} \langle \btheta, \bphi_{V^{(l-1)}} (s_{m,h},a_{m,h}) \rangle 
    \\ & = c_{m,h} + (1-q) \cdot \langle \btheta_{m,h}, \bphi_{V^{(l-1)}} (s_{m,h},a_{m,h}) \rangle 
    \\ & = c_{m,h} + (1-q) \cdot \langle \btheta_{m,h}, \bphi_{V^{(l)}} (s_{m,h},a_{m,h}) \rangle + (1-q) \cdot \langle \btheta_{m,h}, \left[\bphi_{V^{(l-1)}}-\bphi_{V^{(l)}}\right] (s_{m,h},a_{m,h}) \rangle,
\end{align*}where $\btheta_{m,h} = \argmin_{ \btheta \in \cC_j \cap \cB} \langle \btheta,\bphi_{V^{(l-1)}}(s_{m,h},a_{m,h}) \rangle$ and its existence is guaranteed under the event of Lemma \ref{lem: confidence set optimism}. 
Define $\PP_{m,h}$ as the transition kernel parametrized by $\btheta_{m,h}$, i.e.,  
\begin{align*}
    \PP_{m,h}(\cdot|\cdot,\cdot) = \langle \bphi(\cdot|\cdot,\cdot) , \btheta_{m,h} \rangle.
\end{align*}
Then from above we have
\begin{align*}
    &Q^{(l)} (s_{m,h},a_{m,h}) 
    \\ & = c_{m,h} + (1-q) \cdot \langle \btheta_{m,h}, \bphi_{V^{(l)}} (s_{m,h},a_{m,h}) \rangle + (1-q) \cdot \PP_{m,h} \left[ V^{(l-1)} -V^{(l)} \right] (s_{m,h},a_{m,h}) 
    \\ & \geq c_{m,h} + (1-q) \cdot \PP_{m,h} V^{(l)} (s_{m,h},a_{m,h})  - (1-q) \cdot \frac{1}{t_{j_m}}, 
\end{align*} where the inequality is by the \texttt{DEVI} terminal condition that $\|V^{(l)} - V^{(l-1)} \|_\infty \leq \epsilon_j = 1/t_{j_m}$. 
Therefore we have 
\begin{align*}
    Q_{j_m} (s_{m,h},a_{m,h}) \geq c_{m,h} + (1-q) \cdot \PP_{m,h} V_{j_m} (s_{m,h},a_{m,h}) - (1-q)\cdot \frac{1}{t_{j_m}},
\end{align*}  and it follows that 
\begin{align*}
    & c_{m,h} + \PP V_{j_m} (s_{m,h},a_{m,h}) - Q_{j_m} (s_{m,h},a_{m,h}) 
    \\ & \leq \PP V_{j_m}(s_{m,h},a_{m,h}) - (1-q) \cdot \PP_{m,h} V_{j_m} (s_{m,h},a_{m,h}) + (1-q)\cdot \frac{1}{t_{j_m}}
    \\ & = [\PP - \PP_{m,h}] V_{j_m} (s_{m,h},a_{m,h}) + q \PP_{m,h} V_{j_m} (s_{m,h},a_{m,h}) + (1-q) \cdot \frac{1}{t_{j_m}}
    \\ & \leq [\PP - \PP_{m,h}] V_{j_m} (s_{m,h},a_{m,h}) + \frac{\Bstar}{t_{j_m}} + (1-q) \cdot \frac{1}{t_{j_m}}
    \\ & = \langle \btheta^* - \btheta_{m,h} , \bphi_{V_{j_m}} (s_{m,h},a_{m,h}) \rangle + \frac{\Bstar + 1-q}{t_{j_m}}, 
\end{align*}where the second inequality is by the optimism $V_{j_m} \leq V^\star \leq \Bstar$ under the event of Lemma \ref{lem: confidence set optimism}, and $q = 1/t_{j_m}$ according to Line \ref{algline: discount factor 1 over t_j} in Algorithm \ref{alg: linear kernel ssp}. We then conclude that  
\begin{align}\label{eq: E1 upper bound 1}
    & \sum_{m\in\cM_0(M)} \sum_{h=1}^{H_m} \left[ c_{m,h} + \PP V_{j_m}(s_{m,h},a_{m,h}) - Q_{j_m} (s_{m,h},a_{m,h}) \right] \notag 
    \\ & \leq \underbrace{\sum_{m\in \cM_0(M)} \sum_{h=1}^{H_m} \langle \btheta^* - \btheta_{m,h},\bphi_{V_{j_m}} ( s_{m,h},a_{m,h} ) \rangle}_{A_1} + \underbrace{(\Bstar+1)\cdot \sum_{m\in \cM_0(M)} \sum_{h=1}^{H_m} \frac{1}{t_{j_m}}}_{A_2} . 
\end{align} 
\textbf{To bound $A_1$}: Recall that $\hat\btheta_{j_m}$ given by Line \ref{algline: hat btheta} is the center of the confidence ellipsoid $\cC_{j_m}$. First for each term $ \langle \btheta^* - \btheta_{m,h} , \bphi_{V_{j_m}}(s_{m,h},a_{m,h}) \rangle $ in $A_1$, we write 
\begin{align}\label{eq: E1 bound A1 1}
  &  \langle \btheta^* - \hat\btheta_{j_m} + \hat\btheta_{j_m} - \btheta_{m,h},\bphi_{V_{j_m}} ( s_{m,h},a_{m,h} ) \rangle \notag
    \\ & \leq  \left( \|\btheta^* - \hat\btheta_{j_m}\|_{\bSigma_{t(m,h)}} + \| \hat\btheta_{j_m} - \btheta_{m,h} \|_{\bSigma_{t(m,h)}}  \right) \cdot \| \bphi_{V_{j_m}}(s_{m,h},a_{m,h}) \|_{\bSigma_{t(m,h)}^{-1}} \notag
    \\ & \leq 2  \left( \|\btheta^* - \hat\btheta_{j_m}\|_{\bSigma_{t_{j_m}}} + \| \hat\btheta_{j_m} - \btheta_{m,h} \|_{\bSigma_{t_{j_m}}}  \right) \cdot \| \bphi_{V_{j_m}}(s_{m,h},a_{m,h}) \|_{\bSigma_{t(m,h)}^{-1}} \notag 
    \\ & \leq 4 \beta_T  \| \bphi_{V_{j_m}}(s_{m,h},a_{m,h}) \|_{\bSigma_{t(m,h)}^{-1}} . 
\end{align} Here the first inequality comes from the triangle inequality and Cauchy-Schwarz inequality. 
For the second inequality, recall that $t_{j_m}$ given by Line \ref{algline: def t_j} in Algorithm \ref{alg: linear kernel ssp} is the time step when the $j_m$-th \texttt{DEVI} sub-routine is called, while $t(m,h)$ is the time step corresponds to the $h$-th step in the $m$-th interval and $t(m,h)\geq t_{j_m}$. Therefore, by the determinant-doubling triggering condition, we must have $\det(\bSigma_{t(m,h)}) \leq 2 \det(\bSigma_{t_{j_m}})$, otherwise $t(m,h)$ and $t_{j_m}$ would not belong to the same interval $m$. The second inequality then follows from $\lambda_i(\bSigma_{t(m,h)}) \leq 2 \lambda_i(\bSigma_{t_{j_m}})$ $\forall i \in [d]$, where $\lambda_i(\cdot)$ is the $i$-th eigenvalue. 
The last inequality holds because under Lemma \ref{lem: confidence set optimism}, $\btheta^*$ and $\btheta_{m,h}$ belongs to the confidence ellipsoid $\cC_{j_m}$ defined by Line \ref{algline: confidence ellipsoid}. 

Also note that for each term $ \langle \btheta^* - \btheta_{m,h} , \bphi_{V_{j_m}}(s_{m,h},a_{m,h}) \rangle $ in $A_1$, we have
\begin{align}\label{eq: E1 bound A1 2}
    \langle \btheta^* - \btheta_{m,h} , \bphi_{V_{j_m}}(s_{m,h},a_{m,h}) \rangle&\leq  \langle \btheta^* , \bphi_{V_{j_m}}(s_{m,h},a_{m,h}) \rangle\notag\\
    &=\PP V_{j_m}(s_{m,h},a_{m,h})\notag\\
    &\leq \Bstar,
\end{align}
where both inequalities hold due to $0\leq V_{j_m} (\cdot) \leq \Bstar$.
Combine \eqref{eq: E1 bound A1 1} and \eqref{eq: E1 bound A1 2} and we have
\begin{align}\label{eq: E1 bound A1 3}
    A_1 & \leq 4 \beta_T \sum_{m\in\cM_0} \sum_{h=1}^{H_m} \min \left\{1, \ \| \bphi_{V_{j_m}}(s_{m,h},a_{m,h}) \|_{\bSigma_{t(m,h)}^{-1}} \right\} \notag 
    \\ & \leq 4\beta_T \sqrt{\left(\sum_{m\in\cM_0} \sum_{h=1}^{H_m} 1\right) \cdot \left( \sum_{m\in\cM_0} \sum_{h=1}^{H_m} \min \left\{1, \ \| \bphi_{V_{j_m}}(s_{m,h},a_{m,h}) \|^2_{\bSigma_{t(m,h)}^{-1}} \right\} \right)},
\end{align}where the first inequality holds due to $\Bstar < \beta_T$, and the second inequality is by Cauchy-Schwarz inequality. Note that
\begin{align*}
    & \sum_{m\in\cM_0} \sum_{h=1}^{H_m} \min \left\{1, \ \| \bphi_{V_{j_m}}(s_{m,h},a_{m,h}) \|^2_{\bSigma_{t(m,h)}^{-1}} \right\} 
    \\ & \leq  2 \left[ d \log \left( \frac{\textnormal{trace}(\lambda \Ib)+T\cdot\max_{m \in \cM_0}\|\bphi_{V_{j_m}}(\cdot,\cdot)\|_2^2}{d}\right) - \log\left( \det(\lambda \Ib)\right)\right]
    \\ & \leq 2d \log\left( \frac{\lambda d + T \Bstar^2 d}{\lambda d}\right) 
    \\ & = 2d\log\left( 1 + T\Bstar^2/\lambda \right), 
\end{align*}where the first inequality holds by Lemma \ref{lem: lemma 11 in abbasi}, and the second inequality holds because $V_{j_m}(\cdot) \leq \Bstar$ under Lemma \ref{lem: confidence set optimism} and thus $\max_{m\in\cM_0} \|\bphi_{V_{j_m}}(\cdot,\cdot)\|_2\leq \Bstar \sqrt{d} $ by Assumption \ref{assump: linear kernel mdp}. Combine the above inequality with \eqref{eq: E1 bound A1 3} and we conclude that
\begin{align}\label{eq: E1 bound A1}
    A_1 \leq 4 \beta_T \sqrt{2Td \cdot \log\left( 1 + \Bstar^2 T/\lambda \right) }.
\end{align}
\textbf{To bound $A_2$}: by the definition of $\cM_0$ we can rewrite $A_2$ as
\begin{align*}
    A_2 & = (\Bstar+1)\cdot \sum_{m\in \cM_0(M)} \sum_{h=1}^{H_m} \frac{1}{t_{j_m}} = (\Bstar+1)\cdot \sum_{j = 1}^J \sum_{t = t_j+1}^{t_{j+1}} \frac{1}{t_j}.
\end{align*} Note that the time step doubling condition $t \geq 2t_j$ in Line \ref{algline: trigger condition} implies that $t_{j+1} \leq 2 t_j$ for all $j$. Therefore we have 
\begin{align*}
    A_2 & \leq (\Bstar+1) \cdot \sum_{j=1}^J \frac{2 t_j}{t_j} = 2(\Bstar+1) J  \leq 4.5 d\Bstar \left[ \log \left( 1 + \frac{T\Bstar^2 d}{\lambda} \right) + \log(T) \right] ,
\end{align*}where the last step is by Lemma \ref{lem: bound number of calls to EVI}. Together with \eqref{eq: E1 upper bound 1} and \eqref{eq: E1 bound A1} we conclude that 
\begin{align}\label{eq: E1 bound major term 2}
    & \sum_{m\in\cM_0(M)} \sum_{h=1}^{H_m} \left[ c_{m,h} + \PP V_{j_m}(s_{m,h},a_{m,h}) - Q_{j_m} (s_{m,h},a_{m,h}) \right] \notag 
    \\ & \leq 4\beta_T \sqrt{2Td \cdot \log\left( 1 + \Bstar^2 T/\lambda \right) } + 5d\Bstar \left[ \log \left( 1 + \frac{T\Bstar^2 d}{\lambda} \right) + \log(T) \right]. 
\end{align} 
To bound $E_1$, it remains to bound the following
\begin{align*}
    & \sum_{m\in\cM_0^c} \sum_{h=1}^{H_m} \left[ c_{m,h} + \PP V_{j_m}(s_{m,h},a_{m,h}) - Q_{j_m} (s_{m,h},a_{m,h}) \right]. 
\end{align*} Note that by definition, $\cM_0^c$ are all the intervals $m$ such that $j_m = 0$, i.e., the intervals before the first call of the \texttt{DEVI} sub-routine. However, since $t_0 = 1$, by the triggering condition $t \geq 2 t_0$, we know that the first \texttt{DEVI} is called at $t=2$. Therefore we have 
\begin{align*}
    \sum_{m\in\cM_0^c} \sum_{h=1}^{H_m} \left[ c_{m,h} + \PP V_{j_m}(s_{m,h},a_{m,h}) - Q_{j_m} (s_{m,h},a_{m,h}) \right] &= \sum_{h=1}^2 \left[ c_{1,h} + \PP V_{0}(s_{1,h},a_{1,h}) - Q_{0} (s_{1,h},a_{1,h}) \right]
    \\ & \leq 4,
\end{align*}where the inequality holds because $c_{1,h}, V_0(\cdot) \leq 1$ and $0 \leq Q_0(\cdot,\cdot)$. Together with \eqref{eq: E1 bound major term 2} we conclude that 
\begin{align}\label{eq: E1 bound}
    E_1 & \leq  4\beta_T \sqrt{2Td \cdot \log\left( 1 + \Bstar^2 T/\lambda \right) } + 5d\Bstar \left[ \log \left( 1 + \frac{T\Bstar^2 d}{\lambda} \right) + \log(T) \right] + 4. 
\end{align}
\end{proof}

\subsection{Bounding $E_2$}\label{sec: E2 bound}

The term $E_2$ is the sum of a martingale difference sequence. However, the function $V_{j_m}$ is random and not necessarily bounded, which disqualifies us from applying tools like Azuma-Hoeffding inequality directly. To deal with this issue, we use an auxiliary sequence of functions. The result is summarized by the following lemma. 

\begin{lemma}\label{lem: E2 bound}
With probability at least $1-\delta$, both the event of Lemma \ref{lem: confidence set optimism} and the following hold
\begin{align*}
    \sum_{m=1}^M \sum_{h=1}^{H_m} \left[ V_{j_m}(s_{m,h+1}) - \PP V_{j_m} (s_{m,h},a_{m,h}) \right] \leq 2 \Bstar \sqrt{2T \log\left(\frac{2T}{\delta}\right)} . 
\end{align*}
\end{lemma}

\begin{proof}[Proof of Lemma \ref{lem: E2 bound}]
We define the filtration $\{\cF_{m,h}\}_{m,h}$ such that $\cF_{m,h}$ is the $\sigma$-field of all the history up until $(s_{m,h}, a_{m,h})$ which contains $(s_{m,h}, a_{m,h})$ but does not contain $s_{m,h+1}$. Then $(s_{m,h},a_{m,h})$ is $\cF_{m,h}$-measurable. Also note that the time step $t_{j_m}$ is no later than the time step $t(m,h)$, and thus the function $V_{j_m}$ is also $\cF_{m,h}$-measurable. By the definition of the operator $\PP$, we have 
\begin{align*}
    \EE\left[V_{j_m}(s_{m,h+1}) \middle| \cF_{m,h}\right] = \PP V_{j_m} (s_{m,h},a_{m,h}),  
\end{align*}which shows that the term $E_2$ is the sum of a martingale difference sequence. To deal with the problem that $V_{j_m}$ might not be uniformly bounded, we define an auxiliary sequence of functions 
\begin{align*}
    \tilde{V}_{j_m}(\cdot) \coloneqq \min\{\Bstar, V_{j_m}(\cdot) \},
\end{align*} and it immediately holds that $\tilde{V}_{j_m}$ is $\cF_{m,h}$-measurable. 
We now write $E_2$ as 
\begin{align*}
    E_2 & = \sum_{m=1}^M \sum_{h=1}^{H_m} \left[ \tilde{V}_{j_m}(s_{m,h+1}) - \PP \tilde{V}_{j_m} (s_{m,h},a_{m,h}) \right]
    \\ & \qquad + \sum_{m=1}^M \sum_{h=1}^{H_m} \left[ [V_{j_m} - \tilde{V}_{j_m} ](s_{m,h+1}) - \PP [V_{j_m}-\tilde{V}_{j_m}] (s_{m,h},a_{m,h}) \right] .
\end{align*} Since $\tilde{V}_{j_m}$ is bounded, we can apply Lemma \ref{lem: azuma hoeffding anytime} and get that, with probability at least $1-\delta/2$, 
\begin{align*}
    E_2 & \leq 2 \Bstar \sqrt{2T \log\left( \frac{T}{\delta/2} \right)} + \sum_{m=1}^M \sum_{h=1}^{H_m} \left[ [V_{j_m} - \tilde{V}_{j_m} ](s_{m,h+1}) - \PP [V_{j_m}-\tilde{V}_{j_m}] (s_{m,h},a_{m,h}) \right] .
\end{align*} Now note that under the event of Lemma \ref{lem: confidence set optimism}, we have $\tilde{V}_{j_m}  = V_{j_m}$ for all $j_m\geq 1$ by optimism and also $\tilde{V}_0 = V_0$ by the initialization, which implies that the second term in the RHS is zero. Therefore, take the intersection of the two events and we conclude that, with probability at least $1-\delta$, $E_2 \leq 2\Bstar \sqrt{2T\log(2T/\delta)}$.   
\end{proof}

\subsection{Bounding $E_3$}\label{sec: E3 bound}

To bound $E_3$, we first need the following lemma which shows that the total calls to \texttt{DEVI} in Algorithm \ref{alg: linear kernel ssp} can be bounded. The proof shows that our design of the update condition (i.e. Line \ref{algline: trigger condition} in Algorithm \ref{alg: linear kernel ssp}) is crucial to our regret analysis. Importantly, the determinant doubling criterion alone is not enough, and the novel time step doubling trick is necessary.

\begin{lemma}\label{lem: bound number of calls to EVI}
Conditioned on the event in Lemma \ref{lem: confidence set optimism}, the total number of calls to \texttt{DEVI} is bounded by $J \leq  2 d \log \left( 1 + \frac{T\Bstar^2 d}{\lambda} \right) + 2 \log(T) $.
\end{lemma}

\begin{proof}[Proof of Lemma \ref{lem: bound number of calls to EVI}]
By Line \ref{algline: trigger condition} we have $J = J_1 + J_2$ where $J_1$ is the total number of times that the determinant is doubled and $J_2$ is the total number of times that the time step is doubled. 
First we bound $J_1$. Note that $V_0$ is from the initialization instead of the output of \texttt{DEVI} and it holds that $V_0 \leq \Bstar$. By Line \ref{algline: bSigma update} of Algorithm \ref{alg: linear kernel ssp} and the initialization $\bSigma_0 = \lambda \Ib$, we have 
\begin{align*}
    \|\bSigma_T\|_2 & = \left\| \lambda\Ib + \sum_{j=0}^J \sum_{t=t_{j}+1}^{t_{j+1}} \bphi_{V_j}(s_t,a_t) \bphi_{V_j}(s_t,a_t)^\top \right\|_2 
    \\ & \leq \lambda + \sum_{j=0}^J \sum_{t=t_{j}+1}^{t_{j+1}} \| \bphi_{V_{j}}(s_t,a_t) \|_2^2 
    \\ & \leq \lambda + T \Bstar^2 d,
\end{align*}where the first inequality is by the triangle inequality and the second inequality holds by Assumption \ref{assump: linear kernel mdp} and $V_j \leq \Bstar$ for all $j \geq 0$ under the event of Lemma \ref{lem: confidence set optimism}. We then have that $\det(\bSigma_T)\leq (\lambda + T\Bstar^2d)^d$. It follows that 
\begin{align*}
    \left( \lambda + T\Bstar^2 d \right)^d \geq 2^{J_1} \cdot \det \left( \bSigma_0 \right) = 2^{J_1} \cdot \lambda^d,
\end{align*}by the determinant-doubling trigger condition. From the above inequality we conclude that 
\begin{align*}
    J_1 \leq 2 d \log \left( 1 + \frac{T\Bstar^2 d}{\lambda} \right) . 
\end{align*} To bound $J_2$, note that $t_0 = 1$ and thus $2^{J_2}\leq T$, which immediately gives $J_2 \leq \log_2(T) \leq 2 \log(T)$. Altogether we conclude that 
\begin{align*}
    J \leq 2 d \log \left( 1 + \frac{T\Bstar^2 d}{\lambda} \right) + 2 \log(T). 
\end{align*}
\end{proof}

We are now ready to bound $E_3$ in \eqref{eq: proof of regret decomposition}. 

\begin{lemma}\label{lem: regret decomp telescope 1}
Assume the event in Lemma \ref{lem: confidence set optimism} holds. Then we have
\begin{align*}
    & \sum_{m=1}^M\left( \sum_{h=1}^{H_m} V_{j_m}(s_{m,h}) - V_{j_m}(s_{m,h+1}) \right) - \sum_{m \in \cM(M)} V_{j_m}(\sinit) \leq  1 + 2 d \Bstar \log \left( 1 + \frac{T\Bstar^2 d}{\lambda} \right) + 2 \Bstar \log(T) .
\end{align*}
\end{lemma}

\begin{proof}[Proof of Lemma \ref{lem: regret decomp telescope 1}]
The proof resembles that of Lemma 31 in \citet{tarbouriech2021stochastic}. We first consider the first term in the LHS. Rearrange the summation and we have
\begin{align*}
    & \sum_{m=1}^M\left( \sum_{h=1}^{H_m} V_{j_m}(s_{m,h}) - V_{j_m}(s_{m,h+1}) \right) 
    \\ & = \sum_{m=1}^M V_{j_m} (s_{m,1}) - V_{j_m} (s_{m,H_m+1})
    \\ & = \sum_{m=1}^{M-1} \left( V_{j_{m+1}}(s_{m+1,1}) - V_{j_m}(s_{m,H_m+1}) \right) + \sum_{m=1}^{M-1} \left( V_{j_m} (s_{m,1}) - V_{j_{m+1}} (s_{m+1,1}) \right) 
    \\ & \quad \  + V_{j_M}(s_{M,1}) - V_{j_M} (s_{M,H_M+1}).
\end{align*} Note that second sum in the above equation is a telescoping sum. Thus we have
\begin{align}\label{eq: lem regret decomp telescope 1}
    & \sum_{m=1}^M\left( \sum_{h=1}^{H_m} V_{j_m}(s_{m,h}) - V_{j_m}(s_{m,h+1}) \right) \notag
    \\ & = \sum_{m=1}^{M-1} \left( V_{j_{m+1}}(s_{m+1,1}) - V_{j_m}(s_{m,H_m+1}) \right) + V_{j_1} (s_{1,1}) - V_{j_M}(s_{M,1}) \notag
    \\ & \qquad  + V_{j_M}(s_{M,1}) - V_{j_M} (s_{M,H_M+1}) \notag
    \\ & =  \sum_{m=1}^{M-1} \left( V_{j_{m+1}}(s_{m+1,1}) - V_{j_m}(s_{m,H_m+1}) \right) + V_{j_1} (s_{1,1}) - V_{j_M}(s_{M,H_M+1}) \notag
    \\ & \leq \sum_{m=1}^{M-1} \left( V_{j_{m+1}}(s_{m+1,1}) - V_{j_m}(s_{m,H_m+1}) \right) + V_{j_1} (s_{1,1}),
\end{align} 
where the inequality holds because $V_j(\cdot)$ is non-negative for all $j$. 

We now consider the term $V_{j_{m+1}}(s_{m+1,1}) - V_{j_m}(s_{m,H_m+1})$. Note that by the interval decomposition, interval $m$ ends if and only if either of the two conditions are met. If interval $m$ ends because goal is reached, then we have 
\begin{align*}
    V_{j_{m+1}}(s_{m+1,1}) - V_{j_m}(s_{m,H_m+1}) = V_{j_{m+1}} (\sinit) - V_{j_m}(g) = V_{j_{m+1}} (\sinit).
\end{align*}
If it ends because the \texttt{DEVI} sub-routine is triggered, then the value function estimator is updated by \texttt{DEVI} and $j_m\neq j_{m+1}$. In such case we simply apply the trivial upper bound $V_{j_{m+1}}(s_{m+1,1}) - V_{j_m}(s_{m,H_m+1}) \leq \max_j \|V_j\|_\infty$. By Lemma \ref{lem: bound number of calls to EVI}, this happens at most $J \leq 2 d \log \left( 1 + \frac{T\Bstar^2 d}{\lambda} \right) + 2\log(T)$ times. Therefore, we can further bound the RHS of \eqref{eq: lem regret decomp telescope 1} as 
\begin{align*}
    & \sum_{m=1}^M\left( \sum_{h=1}^{H_m} V_{j_m}(s_{m,h}) - V_{j_m}(s_{m,h+1}) \right) \notag
    \\ & \leq \sum_{m=1}^{M-1} V_{j_{m+1}} (\sinit) \cdot \ind\{m+1\in \cM(M)\} + V_{j_1} (s_{1,1}) + \left[ 2 d \log \left( 1 + \frac{T\Bstar^2 d}{\lambda} \right) + 2 \log(T)\right] \cdot \max_{j}\|V_j\|_\infty
    \\ & \leq \sum_{m \in \cM(M)} V_{j_m}(\sinit) + V_0(\sinit) + 2 d \Bstar \log \left( 1 + \frac{T\Bstar^2 d}{\lambda} \right) + 2 \Bstar \log(T)
    \\ & \leq  \sum_{m \in \cM(M)} V_{j_m}(\sinit) + 1 + 2 d \Bstar \log \left( 1 + \frac{T\Bstar^2 d}{\lambda} \right) + 2 \Bstar \log(T) ,
\end{align*}where the second inequality is by $\|V_j\|_\infty \leq \Bstar$ and the last step is by the initialization $\|V_0\|_\infty \leq 1$. 
\end{proof}

\subsection{Proof of Theorem \ref{thm: regret upper bound with c_min} and Theorem \ref{thm: corollary general cost}}\label{sec:proof_specific_cost}

\begin{proof}[Proof of Theorem \ref{thm: regret upper bound with c_min}]
    The total cost in $K$ episodes is upper bound by $R_K + K \Bstar$ and is lower bounded by $T \cdot c_{\min}$. Together with Theorem \ref{thm: regret upper bound with T}, with probability at least $1-\delta$, we have
    \begin{align*}
        T \cdot  c_{\min} \leq & 6 \beta_T \sqrt{dT \log\left( 1 + \frac{T\Bstar^2}{\lambda}\right)}  + 7d\Bstar \log\left( T + \frac{T^2 \Bstar^2 d}{\lambda}\right) + K \Bstar. 
    \end{align*}
    Solving the above inequality for the total number of steps $T$, we obtain that
    \begin{align*}
        T = \cO\left( \log^2\left(\frac{1}{\delta}\right)\cdot\left( \frac{K\Bstar}{c_{\min}} + \frac{B^2d^2}{c_{\min}^2}\right) \right).
    \end{align*}
    Plugging this into Theorem \ref{thm: regret upper bound with T} yields the desired result.
\end{proof}

\begin{proof}[Proof of Theorem \ref{thm: corollary general cost}]
    By picking $\rho=K^{-1/3}$, we have $c_{\min} \geq K^{-1/3}$. 
    Replacing $c_{\min}$ with $K^{-1/3}$ in Theorem \ref{thm: regret upper bound with c_min}, the regret for the perturbed SSP is upper bounded by
    \begin{align*}
        \cO\bigg(  \tilde{B}^{1.5} d K^{2/3} \cdot \log^2\left( \frac{ K \tilde{B} d }{ \delta} \right) + \tilde{B}^2 d^2 K^{1/3} \log^2\left( \frac{ K \tilde{B} d }{ \delta} \right) \bigg),
    \end{align*} where $\tilde{B}=B + \Tstar/\rho$. Since the difference between the optimal cost of the perturbed SSP and the original SSP is at most $\Tstar \rho K$, the regret for the original SSP is upper bounded by
    \begin{align*}
        \cO\bigg(  \tilde{B}^{1.5} d K^{2/3} \cdot \log^2\left( \frac{ K \tilde{B} d }{ \delta} \right) + \tilde{B}^2 d^2 K^{1/3} \log^2\left( \frac{ K \tilde{B} d }{ \delta} \right) \bigg) + \Tstar K^{2/3},
    \end{align*}which completes the proof. 
\end{proof}

\subsection{Proof of Theorem \ref{thm: regret upper bound with T}}\label{sec: proof of intermediate thm}

\begin{proof}
Note that the regret decomposition \eqref{eq: regret decomposition} is proved under the condition that the event of Lemma~\ref{lem: confidence set optimism} holds. Then together with Lemmas \ref{lem: confidence set optimism}, \ref{lem: E1 bound} and \ref{lem: E2 bound}, we conclude that with probability at least $1-\delta$, 
\begin{align*}
    R(M) & \leq 4\beta_T \sqrt{2Td \cdot \log\left( 1 + \Bstar^2 T/\lambda \right) } + 5d\Bstar \left[ \log \left( 1 + \frac{T\Bstar^2 d}{\lambda} \right) + \log(T) \right]
    \\ & \qquad + 2 \Bstar \sqrt{2T \log\left(\frac{2T}{\delta}\right)} 
    \\ & \qquad + 4 + 2d \Bstar \log \left( 1 + \frac{T\Bstar^2d}{\lambda} \right) + 2\Bstar \log (T) + 2 . 
\end{align*} Combining the lower order terms finishes the proof. 
\end{proof}

\section{Lower Bound}\label{sec: lower bound details}

\subsection{Proof of the Lower Bound}\label{sec: proof of lower bound}
\begin{proof}[Proof of Theorem \ref{thm: lower bound}]

We now construct a class of challenging SSP instances. We denote these SSPs by $M=\{ \cS, \cA, \PP_{\btheta} , c,\sinit,g  \}$. The state space $\cS$ contains two states, i.e., $\cS = \{\sinit, g\}$. The action space $\cA$ contains $2^{d-1}$ actions where each action $\ab \in \cA$ is a $(d-1)$-dimensional vector $\ab\in \{-1,1\}^{d-1}$. Here we use the boldface notation $\ab$ instead of $a$ to emphasize the action is represented by a vector.
The cost function is given as $c(\sinit,\ab) = 1$ and $c(g,\ab)=0$ for any $\ab\in\cA$. The transition kernel $\PP_{\btheta}$ of this SSP class is parameterized by a $(d-1)$-dimensional vector $\btheta \in \{-\frac{\Delta}{d-1}, \frac{\Delta}{d-1}\}^{d-1} = \Theta$. Specifically, for any $\ab\in\cA$, we have 
\begin{align*}
    \PP_{\btheta} (\sinit | \sinit, \ab) & = 1 - \delta - \langle \ab , \btheta \rangle, \quad \PP_{\btheta} (g | \sinit, \ab) = \delta + \langle \ab , \btheta \rangle, \quad \PP_{\btheta} (g|g,\ab) = 1,
\end{align*} 
where $\delta$ and $\Delta$ are parameters to be determined later. 
It is easy to verify that this is indeed an instance of linear mixture SSP with the parameter $\btheta^* = (\btheta^\top, 1)^\top \in \RR^d$ and the feature mapping $ \bphi(\sinit|\sinit,\ab)=(-\ab^\top,1-\delta)^\top$, $\bphi(g|\sinit,\ab) = (\ab^\top,\delta)^\top$ , $\bphi(\sinit|g,\ab) = \mathbf{0_d}$, and $\bphi(g|g,\ab) = (\mathbf{0}_{d-1}^\top,1)^\top$. 

\begin{remark}\label{rmk:linear_lower_bound}
In addition, this hard-to-learn instance can be adapted into a linear SSP studied in \citet{vial2021regret}.
More specifically, it suffices to set $\btheta^* = (1, \mathbf{0}_d^\top)^\top, \bmu(\sinit)=(1-\delta,-\sqrt{d}\btheta^\top,0),\bphi(\sinit,\ab)=(1,\ab^\top/\sqrt{d},0)^\top$ and $\bphi(g,\ab)=(0,\mathbf{0}_{d-1}^\top,1)^\top$.
Then the linear SSP defined by the cost function $c(s,\ab)=\bphi(s,\ab)^\top\btheta^*$ and the transition probability function $\PP_{\btheta}(s'|s,\ab)=\bphi(s,\ab)^{\top} \bmu(s')$ indeed recovers our construction above.
This suggests that our analysis also yields a $\Omega(dB_\star\sqrt{K})$ for linear SSP, further complementing the results in \citet{vial2021regret}.
\end{remark}

Note that for this SSP instance, the optimal policy is to always choose $\ab_{\btheta}$ in state $\sinit$, where $\ab_{\btheta}$ denote the vector whose entries has the same sign as the corresponding entries of $\btheta$, i.e., $\sgn(\ab_{\btheta,j}) = \sgn(\btheta_j)$ for $j=1,\cdots, d-1$. Here $\ab_{\btheta,j}$ and $\btheta_j$ denote the $j$-th entry of the respective vectors. Then the expected cost under the optimal policy is
\begin{align*}
    V_1^{\pistar_{\btheta}} (\sinit) = \sum_{t=1}^\infty (1-\delta - \Delta)^{t-1} (\delta + \Delta) t = \frac{1}{\delta + \Delta}. 
\end{align*}Therefore we will choose $\delta$ and $\Delta$ such that 
\begin{align}\label{eq: sum of delta and Delta}
    \delta + \Delta = \frac{1}{\Bstar}. 
\end{align} It remains to show that for any history-dependent and possibly non-stationary policy $\pi=\{\pi_t\}_{t = 1}^\infty$, there exists some valid choice of $\delta$ and $\Delta$ such that the corresponding SSP class is hard to learn. 

Let's consider the regret in an arbitrary episode $k$. Let $s_1 = \sinit$. The expected regret can be written as 
\begin{align}\label{eq: regret sum of gap 1}
    & R_{{\btheta},k} \notag
    \\ & = V_1^{\pi} (s_1) - V_1^{\pistar_{\btheta}} (s_1) \notag
    \\ & = V_1^{\pi} (s_1) - \EE_{\ab_1\sim \pi}[Q_1^{\pistar_{\btheta}} (s_1,\ab_1)] + \EE_{a_1\sim \pi}[Q_1^{\pistar_{\btheta}} (s_1,\ab_1)] - V_1^{\pistar} (s_1)  \notag
    \\ & = \EE_{\ab_1} [c(s_1,\ab_1)] + \EE_{\ab_1}\left\{ \EE_{s_2\sim \PP(\cdot|s_1,\ab_1)} [V_2^\pi(s_2)]  \right\} - \EE_{\ab_1} [c(s_1,\ab_1)] - \EE_{a_1}\left\{ \EE_{s_2\sim \PP(\cdot|s_1,\ab_1)} [V_2^{\pistar_{\btheta}}(s_2)]  \right\}  \notag
    \\ &\qquad + \EE_{a_1}[Q_1^{\pistar_{\btheta}} (s_1,\ab_1)] - V_1^{\pistar_{\btheta}}(s_1)  \notag
    \\ & = \EE_{\ab_1,s_2} [V_2^\pi(s_2) - V_2^{\pistar_{\btheta}}(s_2)] + \EE_{\ab_1}[Q_1^{\pistar_{\btheta}} (s_1,\ab_1)] - V_1^{\pistar_{\btheta}} (s_1), \notag
    \\ & = \EE_{\ab_1,s_2} [V_2^\pi(s_2) - V_2^{\pistar_{\btheta}}(s_2)] + \EE_{\ab_1} \left[ \frac{2\Delta}{d-1} \ind\{s_1 = \sinit\}\sum_{j=1}^{d-1} \ind\{\sgn(a_{1,j})\neq \sgn({{\theta}}_j)\} \right] \cdot \Bstar , 
\end{align}where the third equality is by the Bellman equation, and the last equality holds because choosing $\ab_1$ at state $s_1 = \sinit$ instead of $\ab_{\btheta}$ results in an extra probability of $\frac{2\Delta}{d-1} \sum_{j=1}^d \ind\{\sgn(a_{1,j})\neq \sgn({\theta}_j)\}$ to remain in $\sinit$ for step 2, which incurs an extra cost of 1 by our construction of the cost function. Now by recursion, we can write the regret in episode $k$ as
\begin{align*}
    R_{{\btheta},k} & = \frac{2\Delta \Bstar}{d-1} \cdot \sum_{i=1}^\infty \EE_k \left[  \ind\{s_i = \sinit \}\cdot \sum_{j=1}^{d-1} \ind\{ \sgn(a_{i,j}) \neq \sgn({{\theta}}_j) \} \right],
\end{align*} where the expectation $\EE_k$ is taken with respect to the trajectory induced by the transition kernel $\PP_{\btheta}$ and history-dependent policy $\pi$ given the history till the end of episode $k-1$. 

We can now write the total expected regret of $\pi$ in $K$ episodes given ${\btheta}$ as
\begin{align*}
    R_{\btheta}(K) &= \frac{2\Delta \Bstar}{d-1} \cdot \sum_{t=1}^{\infty} \EE_{{\btheta}}\left[ \ind\{s_t = \sinit \} \cdot \sum_{j=1}^{d-1} \ind\{ \sgn(a_{t,j}) \neq \sgn({{\theta}}_j) \} \right],
\end{align*}where the expectation is taken with respect to $\PP_{\btheta}$ and $\pi$. Here we omit the subscript $\pi$ since it is clear from the context.

We denote the total number of steps in $\sinit$ by $N \coloneqq \sum_{t=1}^{\infty} \ind\{s_t = \sinit\}$, and for $j = 1,\cdots, d-1$,  
\begin{align*}
    N_j({\btheta}) & \coloneqq \sum_{t=1}^\infty \ind\{ s_t = \sinit \} \cdot \ind\{ \sgn(a_{t,j}) \neq \sgn({{\theta}}_j) \}.
\end{align*} This allows us to write $R_{\btheta}(K) = \frac{2\Delta \Bstar}{d-1} \EE_{\btheta} [ \sum_{j=1}^{d-1} N_j({\btheta}) ]$. Now to bound the regret, we can rely on a standard technique using Pinsker's inequality \citep{jaksch2010near}. However, this would require each $N_j({\btheta})$ to be almost surely bounded, which does not hold in the case of SSP. To circumvent this issue, we apply the ``capping" trick from \citet{cohen2020near} that cap the learning process to contain only the first $T$ steps for some pre-determined $T$. To be specific, if the $K$ episodes are finished before the time $T$, then the agent remains in state $g$. In this case, the actual regret for this capped process is exactly equal to the uncapped process. On the other hand, if at time $T$ the agent has not finished all the $K$ episodes, it is stopped immediately. In this case the actual regret is smaller than that of the uncapped process. Therefore, we only need to lower bound the expected regret for this capped process. 

Let $N^- \coloneqq \sum_{t=1}^T \ind\{s_t = \sinit\}$, and 
\begin{align*}
    N^{-}_j({\btheta}) \coloneqq \sum_{t=1}^T \ind\{ s_t = \sinit \} \cdot \ind\{ \sgn(a_{t,j}) \neq \sgn({{\theta}}_j) \} . 
\end{align*} Then we can lower bound the expected regret by $R_{\btheta}(K) \geq \frac{2\Delta \Bstar}{d-1} \EE_{\btheta} [\sum_{j=1}^{d-1} N_j^{-}({\btheta}) ]$. For each ${\btheta} \in \{-\frac{\Delta}{d-1},\frac{\Delta}{d-1}\}^{d-1}$, let ${\btheta}^j$ denote the vector which differs from ${\btheta}$ only at the $j$-th entry. Then we sum over  ${\btheta}$ and get that
\begin{align}\label{eq: lower bound 2 copy}
    2 \sum_{{\btheta} \in {\Theta}} R_{\btheta}(K) & \geq \frac{2\Delta \Bstar}{d-1} \sum_{{\btheta}} \sum_{j=1}^{d-1} \left( \EE_{\btheta}[N_j^-({\btheta})] + \EE_{{\btheta}^j} [N_j^-({\btheta}^j)] \right) \notag 
    \\ & = \frac{2\Delta \Bstar}{d-1} \sum_{{\btheta}} \sum_{j=1}^{d-1} \left( \EE_{{\btheta}^j} [N^-] + \EE_{\btheta}[N_j^-({\btheta})] - \EE_{{\btheta}^j} [N_j^-({\btheta})] \right) \notag
    \\ & = \frac{2\Delta \Bstar}{d-1} \sum_{{\btheta}} \sum_{j=1}^{d-1} \left( \EE_{{\btheta}} [N^-] + \EE_{\btheta}[N_j^-({\btheta})] - \EE_{{\btheta}^j} [N_j^-({\btheta})] \right).
\end{align}

The next shows that for large enough $T$, $\EE_{\btheta} [N^{-}]$ is lower bounded for all ${\btheta}$. 

\begin{lemma}[Lemma C.2 in \citealt{cohen2020near}]\label{lem: lower bound steps in sinit}
If $T \geq 2K\Bstar$, then it holds that $\EE_{\btheta}[N^{-} ] \geq K\Bstar/4$ for all ${\btheta} \in \{-\frac{\Delta}{d-1},\frac{\Delta}{d-1}\}^{d-1}$. 
\end{lemma} 
We will also use the following lemma which is a version of Pinsker's inequality \citep{jaksch2010near,zhou2021provably}. 
\begin{lemma}[Pinsker's inequality]\label{lem: pinsker} Fix $T$ and denote the trajectory $\sbb = \{s_1,\cdots, s_T\} \in \cS^T$. For any two probability distributions $\cP_1$ and $\cP_2$ on $\cS^T$ and any bounded function $f: \cS^T \to [0,D]$, we have 
\begin{align*}
    \EE_{\cP_1} f(\sbb) - \EE_{\cP_2} f(\sbb) \leq D \cdot \sqrt{\frac{\log 2}{2}} \cdot \sqrt{\textnormal{KL}(\cP_2 || \cP_1)} .  
\end{align*}
\end{lemma}

Then we pick $T = 2K\Bstar$ and get 
\begin{align*}
     2 \sum_{{\btheta}} R_{\btheta} (K)  & \geq \frac{2 \Delta \Bstar}{d-1} \sum_{{\btheta}} \sum_{j=1}^{d-1} \left( \frac{K\Bstar}{4} + \EE_{\btheta}[N_j^-({\btheta})] - \EE_{{\btheta}^j} [N_j^-({\btheta})] \right) \notag 
    \\ & \geq \frac{2 \Delta \Bstar}{d-1} \sum_{{\btheta}} \sum_{j=1}^{d-1} \left( \frac{K\Bstar}{4} - T \sqrt{\frac{1}{2}} \sqrt{\textnormal{KL} (\cP_{{\btheta}}||\cP_{{\btheta}^j}) } \right),
\end{align*}where the first inequality is by Lemma \ref{lem: lower bound steps in sinit}, and the second inequality is by Lemma \ref{lem: pinsker}. The next lemma shows that the KL-divergence can be related to the quantity $N^-$.

\begin{lemma}\label{lem: KL divergence}
Suppose $4 \Delta < \delta \leq 1/3$. Then we have 
\begin{align*}
    \textnormal{KL}(\cP_{\btheta} || \cP_{{\btheta}^j} ) \leq \frac{16 \Delta^2}{(d-1)^2\delta} \EE_{{\btheta}} [ N^- ]. 
\end{align*}
\end{lemma}

It follows from Lemma \ref{lem: KL divergence} that 
\begin{align}
     2 \sum_{{\btheta}} R_{\btheta} (K)  & \geq \frac{2 \Delta \Bstar}{d-1} \sum_{{\btheta}} \sum_{j=1}^{d-1} \left( \frac{K\Bstar}{4} - T \sqrt{\frac{1}{2}} \cdot \frac{4 \Delta}{d-1}\cdot \frac{1}{\sqrt{\delta}} \sqrt{\EE_{{\btheta}} [N^-]} \right) \notag 
    \\ & \geq \frac{2 \Delta \Bstar}{d-1} \sum_{{\btheta}} \sum_{j=1}^{d-1} \left( \frac{K\Bstar}{4} - T^{3/2} \sqrt{\frac{1}{2}} \cdot \frac{4 \Delta}{d-1}\cdot \frac{1}{\sqrt{\delta}}  \right) \notag
    \\ & = \frac{2 \Delta \Bstar}{d-1} \sum_{{\btheta}} \sum_{j=1}^{d-1} \left( \frac{K\Bstar}{4} - (2K\Bstar)^{3/2} \sqrt{\frac{1}{2}} \cdot \frac{4 \Delta}{d-1}\cdot \frac{1}{\sqrt{\delta}}  \right) , 
\end{align} where the last inequality is by $N^- \leq T = 2K\Bstar$. Simplify the expression and we get that 
\begin{align}\label{eq: lower bound with Delta}
     \frac{1}{|{\Theta}|} \sum_{{\btheta}} R_{\btheta}(K)  & \geq \Bstar \frac{1}{|{\Theta}|} \cdot \frac{1}{d-1} \sum_{{\btheta}} \sum_{j=1}^{d-1} \left(  \frac{\Delta K \Bstar}{4} - \frac{8\Delta^2}{(d-1)\sqrt{\delta}} (K\Bstar)^{3/2} \right) \notag
    \\ & = \Bstar\left[ \frac{\Delta K \Bstar}{4} - \frac{8\Delta^2}{(d-1)\sqrt{\delta}} (K\Bstar)^{3/2} \right]. 
\end{align} We now pick 
\begin{align}\label{eq: Delta final expression}
    \Delta = \frac{(d-1)\sqrt{\delta}}{64 \sqrt{K \Bstar}} , 
\end{align} and $\delta$ such that $\delta + \Delta = 1/\Bstar$, plug into \eqref{eq: lower bound with Delta} and get that 
\begin{align*}
    \frac{1}{|{\Theta}|} \sum_{{\btheta}} R_{{\btheta}} (K) & \geq \frac{d \Bstar \sqrt{\delta} \sqrt{K \Bstar} }{512} \geq \frac{d\Bstar \sqrt{K}}{1024} ,
\end{align*}where the last step is by $\delta + \Delta = \frac{1}{\Bstar}$ and $\Delta < \delta$. Therefore, there must exist some ${\btheta} \in {\Theta}$ such that the expected regret $R_{\btheta}(K)$ satisfies
\begin{align*}
    R_{\btheta}(K) \geq \frac{d \Bstar \sqrt{K}}{1024}.
\end{align*} Taking ${\btheta}^*=({\btheta},1)^\top \in \RR^d$ finishes the proof of the lower bound. 
It remains to check the conditions. Note that by \eqref{eq: sum of delta and Delta} and \eqref{eq: Delta final expression}, we have 
\begin{align*}
    \delta + \frac{(d-1)\sqrt{\delta}}{64 \sqrt{K\Bstar}} = \frac{1}{\Bstar}.
\end{align*}Since we also have $\Delta < \delta$, we then require
\begin{align*}
    \frac{d-1}{64\sqrt{K \Bstar}} \leq \sqrt{\delta} < \frac{1}{\sqrt{\Bstar}},
\end{align*}which implies that $K > (d-1)^2/2^{12}$. This finishes the proof of Theorem \ref{thm: lower bound}. 

\end{proof}

\subsection{Proof of Lemmas in Appendix \ref{sec: proof of lower bound}}

Lemma \ref{lem: lower bound steps in sinit} is straightforward and we refer the reader to Lemma C.2 in \citealt{cohen2020near}. Lemma \ref{lem: pinsker} is a standard result. We thus omit their proof. Lemma \ref{lem: KL divergence} can be easily adapted from Lemma 6.8 in \citealt{zhou2021provably}. However, since the MDP instance we construct under the SSP setting differs from theirs under the discounted setting, we present the proof here for completeness. 

\begin{proof}[Proof of Lemma \ref{lem: KL divergence}] 
Denote the trajectory by $\sbb_t = \{s_1,s_2, \cdots, s_t\}$. The chain rule of the KL-divergence gives
\begin{align}\label{eq: KL 0}
    \textnormal{KL} (\cP_{\btheta} || \cP_{{\btheta}^j}) = \sum_{t=1}^{T-1} \textnormal{KL}\left[ \cP_{\btheta}(s_{t+1}|\sbb_t) \big|\big| \cP_{{\btheta}^j}(s_{t+1}|\sbb_t) \right], 
\end{align}where 
\begin{align*}
    \textnormal{KL}\left[ \cP_{\btheta}(s_{t+1}|\sbb_t) \big|\big| \cP_{{\btheta}^j}(s_{t+1}|\sbb_t) \right] \coloneqq \sum_{\sbb_{t+1} \in \cS} \cP_{{\btheta}}(\sbb_{t+1}) \log \frac{\cP_{\btheta}(s_{t+1}|\sbb_t)}{\cP_{{\btheta}^j}(s_{t+1}|\sbb_t)} . 
\end{align*} Then we write
\begin{align*}
    & \sum_{s_{t+1} \in \cS} \cP_{{\btheta}}(\sbb_{t+1}) \log \frac{\cP_{\btheta}(s_{t+1}|\sbb_t)}{\cP_{{\btheta}^j}(s_{t+1}|\sbb_t)} 
    \\ & = \sum_{\sbb_t \in \cS^{\times t}}\cP_{{\btheta}}(\sbb_t)  \sum_{s \in \cS} \cP_{{\btheta}}(s_{t+1}=s|\sbb_t) \log \frac{\cP_{\btheta}(s_{t+1} = s|\sbb_t)}{\cP_{{\btheta}^j}(s_{t+1} = s|\sbb_t)} 
    \\ & = \sum_{\sbb_{t-1}\in \cS^{\times (t-1)}} \cP_{{\btheta}} (\sbb_{t-1}) \sum_{s'\in\cS, \ab\in \cA} \cP_{{\btheta}}(s_t = s', \ab_t = \ab | \sbb_{t-1}) 
    \\ & \qquad \cdot \sum_{s \in \cS} \cP_{{\btheta}} (s_{t+1}=s | s_t = s', \ab_t = \ab , \sbb_{t-1}) \log \frac{\cP_{{\btheta}}(s_{t+1}= s| s_t = s', \ab_t = \ab, \sbb_{t-1} )}{\cP_{{\btheta}^j}(s_{t+1}= s| s_t = s', \ab_t = \ab, \sbb_{t-1} )}. 
\end{align*} Note that when $s' = g$, the transition is irrelevant of ${\btheta}$ and $\cP_{{\btheta}}(s_{t+1}= s| s_t = s', \ab_t = \ab, \sbb_{t-1} )=\cP_{{\btheta}^j}(s_{t+1}= s| s_t = s', \ab_t = \ab , \sbb_{t-1} )$ for all ${\btheta}$. Therefore the log-term in the above equation vanishes when $s' = g$. So we only need to consider the case where $s' = \sinit$ in the summation, and it follows that
\begin{align}\label{eq: KL 1}
     & \sum_{s_{t+1} \in \cS} \cP_{{\btheta}}(s_{t+1}) \log \frac{\cP_{\btheta}(s_{t+1}|\sbb_t)}{\cP_{{\btheta}^j}(s_{t+1}|\sbb_t)} \notag
    \\ & = \sum_{\sbb_{t-1}\in \cS^{\times (t-1)}} \cP_{{\btheta}} (\sbb_{t-1}) \sum_{\ab\in \cA} \cP_{{\btheta}}(s_t = \sinit, \ab_t = \ab | \sbb_{t-1}) \notag
    \\ & \qquad \cdot \sum_{s \in \cS} \cP_{{\btheta}} (s_{t+1}=s | s_t = \sinit, \ab_t = \ab , \sbb_{t-1}) \log \frac{\cP_{{\btheta}}(s_{t+1}= s| s_t = \sinit, \ab_t = \ab, \sbb_{t-1} )}{\cP_{{\btheta}^j}(s_{t+1}= s| s_t = \sinit, \ab_t = \ab, \sbb_{t-1} )} \notag
    \\ & = \sum_{\ab\in \cA} \cP_{{\btheta}} (s_t = \sinit, \ab_t = \ab) \cdot \sum_{s \in \cS} \cP_{{\btheta}} ( s_{t+1} = s | s_t = \sinit, \ab_t = \ab ) \log \frac{\cP_{{\btheta}}(s_{t+1}=s|s_t = \sinit, \ab_t = \ab)}{\cP_{{\btheta}^j}(s_{t+1}=s|s_t = \sinit, \ab_t = \ab)} . 
\end{align} Note that when $s_t = \sinit$, $s_{t+1}$ is either $\sinit$ or $g$ with probability $1-\delta - \langle \ab, {\btheta} \rangle$ and $\delta + \langle \ab , {\btheta} \rangle$. Then we can further write \eqref{eq: KL 1} as  
\begin{align}\label{eq: KL 2}
    & \sum_{s_{t+1} \in \cS} \cP_{\btheta}(s_{t+1}) \log \frac{\cP_{\btheta}(s_{t+1}|\sbb_t)}{\cP_{\btheta^j}(s_{t+1}|\sbb_t)} \notag
    \\ & = \sum_{\ab\in \cA} \cP_{{\btheta}} (s_t = \sinit, \ab_t = \ab) \notag
    \\ & \qquad \cdot \left[ (1-\delta - \langle \ab, {\btheta} \rangle) \cdot \log \frac{1-\delta - \langle \ab, {\btheta} \rangle}{1-\delta - \langle \ab, {\btheta}^j \rangle} + (\delta + \langle \ab , {\btheta} \rangle) \cdot \log \frac{\delta + \langle \ab , {\btheta} \rangle}{\delta + \langle \ab , {\btheta}^j \rangle} \right] \notag
    \\ & \leq \sum_{\ab\in \cA} \cP_{{\btheta}} (s_t = \sinit, \ab_t = \ab) \cdot \frac{2\langle \ab, {\btheta}^j - {\btheta} \rangle^2}{\delta + \langle \ab, {\btheta} \rangle} , 
\end{align}where the last step holds due to the following inequality with $\delta' = \delta + \langle \ab , {\btheta} \rangle $, and $\epsilon' = \langle \ab, {\btheta}^j - {\btheta} \rangle$. 
\begin{lemma}[Lemma 20, \citealt{jaksch2010near}]\label{lem: lower bound inequality bernoulli} For any real number $\delta'$ and $\epsilon'$ such that $0 \leq \delta' \leq 1/2$ and $\epsilon' \leq 1 - 2\delta'$, we have
\begin{align*}
    \delta' \log \frac{\delta'}{\delta'+\Delta} + (1-\delta') \log \frac{1-\delta'}{1-\delta'-\epsilon'} \leq \frac{2 (\epsilon')^2}{\delta'}.
\end{align*}
\end{lemma}To verify the assumptions of Lemma \ref{lem: lower bound inequality bernoulli}, note that $\delta' \leq \delta + \Delta \leq 1/12 + 1/3< 1/2$ by $4\Delta \leq \delta \leq 1/3$ from the assumption of Lemma \ref{lem: KL divergence}. Also note that
\begin{align*}
    \epsilon' = \langle a, {\btheta}^j - {\btheta} \rangle \leq 2 \Delta \leq 1-2(\Delta + \delta) \leq 1-2\delta',
\end{align*}where the first step is by the definition of ${\btheta}$, the second step is by $\delta \leq 1/12$ and $\delta+\Delta \leq 5/12$, and the last step is by $\delta' \leq \delta + \Delta$. Therefore, \eqref{eq: KL 2} holds and we have
\begin{align*}
     \sum_{s_{t+1} \in \cS} \cP_{{\btheta}}(s_{t+1}) \log \frac{\cP_{\btheta}(s_{t+1}|\sbb_t)}{\cP_{{\btheta}^j}(s_{t+1}|\sbb_t)} \notag
     & \leq \sum_{\ab\in \cA} \cP_{{\btheta}} (s_t = \sinit, \ab_t = \ab) \cdot \frac{2\langle \ab, {\btheta}^j - {\btheta} \rangle^2}{\delta - \Delta} 
    \\ & \leq \frac{2 \langle \ab , {\btheta}^j - {\btheta} \rangle^2}{ \delta/2 } \cdot \sum_{\ab\in \cA} \cP_{{\btheta}} (s_t = \sinit, \ab_t = \ab) 
    \\ & = \frac{4 (2\Delta)^2 }{(d-1)^2\delta} \cdot \sum_{\ab\in \cA} \cP_{{\btheta}} (s_t = \sinit, \ab_t = \ab) 
    \\ & = \frac{16\Delta^2}{(d-1)^2 \delta} \cdot \cP_{{\btheta}}(s_t = \sinit).
\end{align*}Together with \eqref{eq: KL 0} we have
\begin{align*}
    \textnormal{KL} (\cP_{\btheta} || \cP_{{\btheta}^j}) & = \sum_{t=1}^{T-1} \sum_{s_{t+1} \in \cS} \cP_{{\btheta}}(s_{t+1}) \log \frac{\cP_{\btheta}(s_{t+1}|\sbb_t)}{\cP_{{\btheta}^j}(s_{t+1}|\sbb_t)} 
    \\ & \leq \frac{16 \Delta^2}{(d-1)^2\delta} \sum_{t=1}^T \cP_{{\btheta}} (s_t = \sinit )
    \\ & = \frac{16 \Delta^2}{(d-1)^2\delta} \ \EE_{\btheta} [N^-],
\end{align*}where the last step is by the definition of $N^-$.

\end{proof}

\section{Lemmas for the Upper Bounds}

\subsection{Proof of Lemma \ref{lem: confidence set optimism}}\label{apdx: proof of confidence set optimism}

We first introduce the following classical result for self-normalized vector-valued  martingales.

\begin{lemma}[Theorem 1, \citealt{abbasi2011improved}] \label{lem: self normalized vector valued martingale}
Let $\{\cF_t\}_{t=0}^\infty$ be a filtration. Suppose $\{\eta_t\}_{t=1}^\infty$ is a $\RR$-valued stochastic process such that $\eta_t$ is $\cF_{t}$-measurable and $\eta_t | \cF_{t-1}$ is $B$-sub-Gaussian. Let $\{\bphi_t\}_{t=1}^\infty$ be an $\RR^d$-valued stochastic process such that $\bphi_t$ is $\cF_{t-1}$-measurable. Assume that $\bSigma$ is an $d\times d$ positive definite matrix. For any $t\geq 1$, define 
\begin{align*}
    \bSigma_t = \bSigma + \sum_{i=1}^t \bphi_i \bphi_i^\top \ , \quad \ba_t = \sum_{i=1}^t \eta_i \bphi_i.
\end{align*}Then, for any $\delta>0$, with probability at least $\delta$, for all $t$, we have 
\begin{align*}
    \| \bSigma_t^{-1/2}\ba_t\|_2 \leq B \sqrt{2 \log \left( \frac{\det(\bSigma_t)^{1/2}}{\delta \cdot \det(\bSigma)^{1/2}} \right)} .
\end{align*}

\end{lemma}

In the following proof we will decompose $t$ into different rounds. 
For all $j \geq 1$, round $j$ corresponds to $t \in [t_j + 1, t_{j+1}]$, during which the action-value function estimator is the output $Q_j$ of \texttt{DEVI}. We then apply an induction argument on the rounds to show that the optimism holds for all $j \geq 1$.

\begin{proof}[Proof of Lemma \ref{lem: confidence set optimism}]

From the initialization of Algorithm \ref{alg: linear kernel ssp}, we have $V_0 \leq \Bstar$. 
\\ Let's consider \textbf{round 1}. We define $\eta_t = V_0(s_{t+1}) - \langle \bphi_{V_0}(s_t,a_t), \btheta^* \rangle$ for $t \in [1,t_1]$. Then $\{\eta_t\}_{t=1}^{t_1}$ are $\Bstar$-sub-Gaussian. We then apply Lemma \ref{lem: self normalized vector valued martingale} and conclude that the following holds with probability at least $1-\frac{\delta}{t_1(t_1+1)}$, for all $t\in[1,t_1]$:
\begin{align}\label{eq: lem confidence set optimism 1}
    \left\| \bSigma_t^{-1/2} \sum_{i=1}^t \bphi_{V_0}(s_i,a_i) \eta_i \right\|_2 & \leq \Bstar\sqrt{2 \log \left( \frac{\det(\bSigma_t)^{1/2}}{\delta \cdot \lambda^{d/2}/(t_1(t_1+1))}
    \right)}\notag
    \\ & \leq \Bstar \sqrt{d \log \left( \frac{1+t d/(d\lambda)}{\delta/(t_1(t_1+1))} \right)} \notag
    \\ & \leq \Bstar \sqrt{d \log \left( \frac{t_1(t_1+1) + t \cdot t_1(t_1+1) \Bstar^2 /\lambda}{\delta} \right)} , 
\end{align} where the second step is by Assumption \ref{assump: linear kernel mdp}, Lemma \ref{lem: determinant trace} and the initialization $|V_0|\leq 1$. Consider the LHS of \eqref{eq: lem confidence set optimism 1}. We have 
\begin{align*}
    & \left\| \bSigma_t^{-1/2} \sum_{i=1}^t \bphi_{V_0}(s_i,a_i) \eta_i \right\| 
    \\ & = \left\| \bSigma_t^{1/2} \bSigma_t^{-1} \sum_{i=1}^t \bphi_{V_0}(s_i,a_i) V_0(s_{i+1}) - \bSigma_t^{1/2} \bSigma_t^{-1} \left( \bSigma_t - \lambda \Ib\right) \btheta^*\right\|_2
    \\ & = \left\| \bSigma_t^{1/2} \hat{\btheta}_t - \bSigma_t^{1/2} \btheta^* + \lambda \bSigma_t^{-1/2} \btheta^* \right\|_2
    \\ & \geq \left\| \bSigma_t^{1/2} (\hat\btheta_t - \btheta^* ) \right\|_2 - \left\| \lambda \bSigma_t^{-1/2} \btheta^* \right\|_2 
    \\ & \geq \left\| \bSigma_t^{1/2} (\hat\btheta_t - \btheta^* ) \right\|_2 - \lambda^{1/2} \cdot \sqrt{d},
\end{align*}where the first inequality holds by Cauchy-Schwarz inequality and the second inequality holds because $\|\btheta^* \|_2\leq \sqrt{d}$. Together with \eqref{eq: lem confidence set optimism 1} and the choice of $\beta_t$, we conclude that 
\begin{align*}
    \left\| \bSigma_t^{1/2} (\hat\btheta_t - \btheta^*) \right\|_2 \leq \Bstar \sqrt{d \log \left( \frac{t_1(t_1+1) + t \cdot t_1(t_1+1) \Bstar^2 /\lambda}{\delta} \right)} + \sqrt{\lambda d} \leq \beta_{t_1}.
\end{align*} Since the above holds for all $t\in[1,t_1]$, it follows that with probability at least $1-\frac{\delta}{t_1(t_1+1)}$, the true parameter $\btheta^*$ is in the set $\cC_1\cap\cB$. 

To show that the output $Q_1$ and $V_1$ of \texttt{DEVI} are optimistic, we apply a second induction argument on the loop of \texttt{DEVI}. For the base step, note that by non-negativity of $Q^\star$ and $V^\star$, we have $Q^{(0)}\leq Q^{\star}$ and $V^{(0)} \leq V^\star$. We now assume $Q^{(i)}$ and $V^{(i)}$ are optimistic. For the $i+1$-th iteration, we have 
\begin{align*}
    Q^{(i+1)} (\cdot,\cdot) & = c(\cdot,\cdot) + (1-q) \cdot \min_{\btheta \in \cC_1 \cap \cB} \langle \btheta , \bphi_{V^{(i)}} (\cdot,\cdot) \rangle
    \\ & \leq c(\cdot,\cdot) + (1-q)\cdot \PP V^{(i)}(\cdot,\cdot)
    \\ & \leq c(\cdot,\cdot) +   \PP V^{(i)}(\cdot,\cdot) 
    \\ & \leq Q^\star(\cdot,\cdot),
\end{align*}where the first step is because we are considering the case where $\rho= 0$, the second step is because we are taking the minimum over a set that contains $\btheta^*$, the third step is by non-negativity of $\PP V^{(i)}(\cdot,\cdot)$, and the last step is by the Bellman optimal condition \eqref{eq: bellman optimal condition} and the induction hypothesis that $V^{(i)}$ is optimistic. By induction, we conclude that $Q^{(i)}$ is optimistic for all $i$, and thus the final output $Q_1(\cdot,\cdot)$ and thus $V_1(\cdot)$ are both optimistic. We finish the proof for round 1. 

Now for our outer induction, let's suppose that the event in Lemma \ref{lem: confidence set optimism} holds for round 1 to $j-1$ with high probability. That is, we define the event 
\begin{align*}
    \cE_{j-1} \coloneqq \left\{ \btheta^* \in \cC_i \cap \cB, \quad V_i(\cdot) \leq V^\star(\cdot)\leq \Bstar , \quad Q_i(\cdot,\cdot) \leq Q^\star(\cdot,\cdot) \quad \textnormal{for all} \ i\in[1,j-1] \right\},
\end{align*}and assume that $\Pr(\cE_{j-1}) \geq 1 - \delta'$ for some $\delta' > 0$. 
We now show that the event $\cE_j$ also holds with high probability. Similar to the proof of Lemma \ref{lem: E2 bound}, we construct an auxiliary sequence of functions
\begin{align*}
    \tilde{V}_i(\cdot) \coloneqq \min\left\{ \Bstar, V_i(\cdot)\right\}, \quad i\in[1,j-1].
\end{align*} 
We also denote, for any $i\in[1,j]$ and for any $t \in [t_{i-1}+1, t_i]$, 
\begin{align*}
    \tilde{\eta}_t & = V_{i-1}(s_{t+1}) - \langle \bphi_{\tilde{V}_{i-1}}(s_t,a_t), \btheta^* \rangle ,
    \\ \tilde{\bSigma}_t & = \lambda \Ib + \sum_{l=1}^t \bphi_{\tilde{V}_{i(l)-1}} (s_l,a_l) \bphi_{\tilde{V}_{i(l)-1}} (s_l,a_l)^\top , 
    \\ \tilde{\btheta}_t & = \tilde{\bSigma}_t^{-1} \sum_{l=1}^t \bphi_{\tilde{V}_{i(l)-1}}(s_l,a_l) \tilde{V}_{i(l)-1} (s_{l+1}) , 
    \\ \tilde{\cC}_i & = \left\{ \btheta \in \RR^d : \ \ \left\|\tilde{\bSigma}_{t_i}^{1/2}  (\tilde{\btheta}_{t_i} - \btheta^*) \right\|_2 \leq \beta_{t_i} \right\}, 
\end{align*}where $i(l)$ is the round that contains the time step $l$, i.e., $l \in [t_{i-1}+1, t_i]$. Observe that, by this construction $\{\tilde{\eta}_t\}_{t=1}^{t_j}$ are almost surely $\Bstar$-sub-Gaussian. This allows us to apply Lemma \ref{lem: self normalized vector valued martingale} and do the similar computation as above, and get that, with probability at least $1-\frac{\delta}{t_j(t_j+1)}$, we have the event $\tilde{\cE}_j$ holds where
\begin{align*}
    \tilde{\cE}_j \coloneqq \left\{ \btheta^* \in \tilde{\cC}_j \cap \cB, \quad V_j(\cdot) \leq V^\star(\cdot)\leq \Bstar , \quad Q_j(\cdot,\cdot) \leq Q^\star(\cdot,\cdot) \right\},
\end{align*}and $Q_j$ is the output of $\texttt{DEVI}(\tilde{\cC}_j, \epsilon_j, \frac{1}{t_j}, \rho)$. 

Now, observe that under the event $\cE_{j-1}$, the optimism implies that $\tilde{V}_i = V_i$ for all $i\in[1,j-1]$. It follows that under $\cE_{j-1}$, we have $\tilde{\eta}_t = \eta_t$, $\tilde{\bSigma}_t = \bSigma_t$, $\tilde{\btheta}_t = \hat{\btheta}_t$ for all $t \leq t_j$, and thus $\tilde{\cC}_j = \cC_j$. We then have
\begin{align*}
    \cE_j = \cE_{j-1}\cap \tilde{\cE}_j , 
\end{align*} and by the union bound we have that $\Pr(\cE_j) \geq 1 - \delta' - \frac{\delta}{t_j(t_1+1)}$.

Now, by induction and taking the union bound 
\begin{align*}
    \sum_{j=1}^J \frac{\delta}{t_j(t_j+1)} = \sum_{j=1}^J \delta\cdot\left( \frac{1}{t_j} - \frac{1}{t_j+1} \right) \leq \delta , 
\end{align*}
we conclude that with probability at least $1-\delta$, the good event holds for all $j\leq J$, where $J$ is the total number of times \texttt{DEVI} being called. 
Note that compared with the analysis of \texttt{EVI} in the discounted MDPs setting (for example in \citealt{zhou2021provably}), our analysis of \texttt{DEVI} in SSP uses the induction argument and a union bound, which results in extra $t$ factors in the logarithmic term in the confidence radius $\beta_t$. At last, replacing $t(t+1)$ with $2t^2$ and $\delta$ with $\delta/2$ gives the final expression for $\beta_t$. 

It remains to argue that \texttt{DEVI} always converges in finite time. To begin with, note that it suffices to show that $\|V^{(i)} - V^{(i-1)}\|_\infty$ shrinks exponentially. We now claim that $\|Q^{i} - Q^{(i-1)}\|_\infty$ shrinks exponentially, which together with \eqref{eq: EVI V is min Q} gives the desired result since $\|V^{(i)}-V^{(i-1)} \|_\infty \leq \|Q^{(i)} - Q^{(i-1)}\|_\infty$. To show this, first note that for any $(s,a)$ pair, 
\begin{align*}
    |Q^{(i)} (s,a) - Q^{(i-1)} (s,a) | & = (1-q) \cdot \left| \min_{\btheta \in \cC\cap \cB} \langle \btheta,\bphi_{V^{(i-1)}}(s,a) \rangle - \min_{\btheta \in \cC\cap \cB} \langle \btheta,\bphi_{V^{(i-2)}}(s,a) \rangle \right| 
    \\ & \leq (1-q)\cdot \max_{\btheta \in \cC \cap \cB } \left| \langle \btheta , \bphi_{V^{(i-1)}} (s,a) - \bphi_{V^{(i-2)}} (s,a) \rangle \right|
    \\ & = (1-q)\cdot \left| \langle \bar\btheta , \bphi_{V^{(i-1)}} (s,a) - \bphi_{V^{(i-2)}} (s,a) \rangle \right|
    \\ & = (1-q) \cdot \left|\bar\PP(V^{(i-1)} - V^{(i-2)})(s,a) \right|
    \\ & \leq (1-q) \cdot \max_{s' \in \cS} \left|V^{(i-1)} (s') - V^{(i-2)}(s') \right|
    \\ & = (1-q) \cdot \max_{s' \in \cS} \left|\min_{a'}Q^{(i-1)} (s',a') - \min_{a'}Q^{(i-2)}(s',a') \right|
    \\ & \leq (1-q)\cdot \| Q^{(i-1)} - Q^{(i-2)} \|_\infty,
\end{align*} where $\bar\btheta$ is the $\btheta$ in the non-empty set $\cC \cap \cB$ that achieves the maximum. Here the first inequality holds due to the maximum function, the second inequality holds because $\bar\PP(\cdot|s,a)$ is a probability distribution, and the last inequality holds due to the same reason as the first one. Now, since $s,a$ are arbitrary in the above, we conclude that $\|Q^{(i)} - Q^{(i-1)}\|_\infty \leq (1-q) \|Q^{(i-1)} - Q^{(i-2)}\|_\infty$. This finishes the proof.

\end{proof}

\section{The Bernstein-type Algorithm}\label{sec: proof of bernstein extension}



In this section, we give the full details of the Bernstein-type algorithm $\texttt{LEVIS}^{\texttt{+}}$.
We introduce the algorithm design in Appendix~\ref{sec: detail of bernstein design}.
We then present the analysis of $\texttt{LEVIS}^{\texttt{+}}$ and the corresponding proof of Theorem~\ref{thm: bernstein regret upper bound with cmin}.

It is worth mentioning that Bernstein-type confidence sets have been utilized in the existing literature on the tabular SSP \citep{rosenberg2020near, cohen2021minimax, tarbouriech2021stochastic, jafarnia2021online, chen2021implicit}. 
The Bernstein technique can help get rid of the dependence on $c_{\min}>0$ in the tabular setting (i.e., to achieve $\tilde\cO(\sqrt{K})$ for $c_{\min} = 0$), while it remains a challenge in the linear function approximation setting.
Also, such variance-aware confidence sets allow `horizon-free' algorithms for linear SSP \citep{chen2021improved}, though its regret suffers from worse dependence on the feature dimension $d$. 
Concurrent to our result, \cite{chen2021improved} also proposed an algorithm for linear mixture SSP which utilizes Bernstein-type confidence sets. Similar to ours, they use the technique from the $\texttt{UCRL-VTR}^\texttt{+}$ algorithm from \citet{zhou2020nearly}. The major difference is that their algorithm is based on the reduction to a finite-horizon linear mixture MDP while our algorithm is a direct algorithm without such a reduction.

\subsection{A Detailed Introduction of $\texttt{LEVIS}^{\texttt{+}}$}\label{sec: detail of bernstein design}

Here we go over the details of $\texttt{LEVIS}^{\texttt{+}}$. 

We define a variance operator $\VV$ associated with the unknown underlying transition such that for any function $f: \cS \to \RR$ and $(s,a)\in\cS\times\cA$, 
\begin{align}\label{eq: definition of VV}
    [\VV f] (s,a) \coloneqq [\PP f^2] (s,a) - \left( [\PP f] (s,a)  \right)^2,
\end{align}
 where
$[\PP f](s,a)  =  \sum_{s'\in\cS} \PP(s'|s,a) f(s')$.
One can see from the definition that $[\VV f] (s,a)$ is the conditional variance of $f(s')$ where $s' \sim \PP(\cdot| s,a)$. 
Further note that $[\VV f](s,a)$ bears a nice form which allows us to estimate it by regression. 
To see this, we can calculate
\begin{align*}
   [\PP f^2] (s,a)  & = \sum_{s'\in\cS} \PP(s'|s,a) f^2(s')  = \sum_{s'\in\cS} \langle \bphi(s'|s,a) , \btheta^{*} \rangle f^2(s') = \langle \bphi_{f^2}(s,a) , \btheta^{*} \rangle , 
   \\ [\PP f] (s,a)  & = \langle \bphi_{f}(s,a) , \btheta^{*} \rangle, 
\end{align*}
which are linear functions of $\bphi_{f^2}$ and $\bphi_f$ respectively. 
Therefore, we can estimate $[\VV f](s,a)$ by solving regression over $\PP f$ and $\PP f^2$.

Moreover, it turns out that, to apply the Bernstein-type concentration inequality, it suffices to use an upper bound of $\VV V_j$ where $V_j$ is the value function.
Thus, the strategy here is to first construct an estimator $\hat\VV_t V_j$ (Line~\ref{algline:variance}) and then add a deviation term to $\hat\VV_t V_j$ to get $\hat\sigma_t^2$ (Line~\ref{algline: hat sigma}), which can be shown to be a tight upper bound of $\VV V_j$ with high probability.
Note that here we only need to apply the Hoeffding-type concentration inequality since this deviation term will not appear in the final regret bound. 

Consequently, in Algorithm~\ref{alg: linear kernel ssp bernstein}, we maintain two regularized Gram matrix matrices $\bSigma_t$ (Line~\ref{algline: weighted hat bSigma})  and $\tilde\bSigma_t$ (Line~\ref{algline: weighted tilde bSigma}). 
Here $\tilde\bSigma_t$ is the unweighted Gram matrix used to estimate the linear parameter of the $\PP V_j^2$ term in the variance, and the estimator is given by $\tilde\btheta_t$ (Line~\ref{algline: tilde theta}).
On the other hand, $\bSigma_t$ is a variance-weighted Gram matrix used to estimate the linear parameter of the $\PP V_j$ term in the variance, and the estimator is denoted by $\hat\btheta_t$ (Line~\ref{algline: weighted hat btheta}). 
Furthermore, $\hat\btheta_t$ is also the center of the confidence ellipsoid $\hat\cC_j$ (Line~\ref{algline:confidence_set_bernstein}). 

With $\tilde\btheta_t$ and $\hat\btheta_t$ in hand, for any $t\geq 1$ and $j = j_t$, we can compute the variance estimator by 
\begin{align}\label{eq: VV and E}
\begin{aligned}
    [\hat{\VV}_t V_j] (s_t, a_t) &= [\langle \bphi_{V_j^2}(s_t, a_t), \tilde\btheta_{t-1} \rangle]_{[0, B^2]} - (\langle \bphi_{V_j}(s_t, a_t), \hat\btheta_{t-1} \rangle_{[0, B]})^2,\\ 
    E_t &= \min\{B^2, 2B \check\beta_t\| \bSigma_{t-1}^{-1/2} \bphi_{V_j}(s_t, a_t)\|_2 \} + \min\{B^2, \tilde\beta_t\| \tilde\bSigma_{t-1}^{-1/2} \bphi_{V_j^2}(s_t, a_t)\|_2 \} ,
\end{aligned}
\end{align} 
where $\hat\btheta_0$ and $\tilde\btheta_0$ are initialized to be $\zero$. 
We then set $\hat{\sigma}_t^2 = \max\{B^2/d, [\hat{\VV}_t V_j](s_t, a_t) + E_t \}$, and collecting all the components above yields the Bernstein-type confidence set.
Finally, the rest of the procedures (Lines~\ref{algline:bernstein_criterion}~to~\ref{algline: bernstein V is min Q}) are the same as those in Algorithm~\ref{alg: linear kernel ssp}.

\textbf{The Choice of Key Parameters. }
We set $\{\check\beta_t, \tilde\beta_t, \hat\beta_t\}_{t \geq 1}$ as follows:
\begin{align}\label{eq: choices of 3 betas}
\begin{aligned}
    \check\beta_t & =  8 d \sqrt{\log\left(1+\frac{t}{\lambda}\right) \log\left(\frac{32t^4}{\delta}\right)} + 4 \sqrt{d} \log\left(\frac{32t^4}{\delta}\right) + \sqrt{\lambda d}= \tilde\cO\left( d \right), 
    \\ \tilde\beta_t & = 8  \sqrt{dB^4 \log\left(1+ \frac{tB^4}{d\lambda} \right) \log\left(\frac{32t^4}{\delta}\right) } + 4 B^2 \log\left(\frac{32t^4}{\delta}\right) + \sqrt{\lambda d} = \tilde\cO( \sqrt{dB^4}),\\ 
    \hat{\beta}_t & = 8  \sqrt{d \log\left(1+ \frac{t}{\lambda}\right) \log\left(\frac{32t^4}{\delta}\right)} + 4 \sqrt{d} \log\left(\frac{32t^4}{\delta}\right) + \sqrt{\lambda d} =\tilde\cO(\sqrt{d}),
\end{aligned}
\end{align}
where we use $\tilde\cO(\cdot)$ to hide logarithmic terms.

\subsection{Analysis of $\texttt{LEVIS}^{\texttt{+}}$}

Here, we briefly introduce the main steps in establishing the regret upper bound for Algorithm \ref{alg: linear kernel ssp bernstein}. The detailed proof is in Appendix \ref{sec: proof of bernstein regret upper bound with cmin}.

First, we have the following result, which is the counterpart of Lemma \ref{lem: confidence set optimism}. We define a map $j(t)$ such that for any $t\geq 1$, $j=j(t)$ is the index of the value function estimate $V_j$ used for data collection (i.e. line \ref{algline: take action and observe bernstein} to \ref{algline: tilde theta}) in step $t$ of Algorithm \ref{alg: linear kernel ssp bernstein}. 

\begin{lemma}\label{lem: key bern lemma}
There exist choice of $\{\hat\beta_t, \tilde\beta_t, \check\beta_t\}_{t\geq 1}$, such that with probability at least $1-3\delta$, for all $t$ and $j = j(t)$, the \texttt{DEVI} subroutine in Algorithm~\ref{alg: linear kernel ssp bernstein} converges in finite time and the following holds
\begin{align}\label{eq: event of key bern lemma}
    \btheta^* \in \hat\cC_{t} \cap \cB, \quad 0\leq Q_{j}(\cdot,\cdot) \leq Q^{\star}(\cdot,\cdot),  \quad \textnormal{and} \quad |\hat{\VV}_t V_{j}(s_t,a_t) - \VV V_{j}(s_t, a_t)|\leq E_t .  
\end{align}
\end{lemma}

\begin{proof}[Proof of Lemma \ref{lem: key bern lemma}] 
See Section \ref{sec: proof of key bern lemma}.    
\end{proof}

Using the Bernstein concentration inequality, we can choose $\{\check\beta_t, \tilde\beta_t, \hat\beta_t\}$ as given by \eqref{eq: choices of 3 betas}.

\paragraph{Regret decomposition} Compared to the interval decomposition for the Hoeffding case, we add an extra condition for triggering a new interval in the Bernstein case: a new interval start when either of the three conditions is met: (1) the cumulative cost in an interval exceeds $\Bstar$; (2) \texttt{DEVI} is triggered; (3) the goal state $g$ is reached.
It is easy to check that, the regret decomposition has the same form as that of Lemma \ref{lem: regret decomposition} with the extra condition. 
To see this, note that in the proof of Lemma \ref{lem: regret decomposition} in Section \ref{sec: proof o fregret decomposition}, we only require that the true parameter $\btheta^*$ is in the confidence sets and all $V_j$'s are upper bounded by $V^{\star}$. 
These conditions hold under the event of Lemma \ref{lem: key bern lemma}. Hence we have the following regret decomposition. 
\begin{align}\label{eq: bernstein regret decomposition}
    R(M) &\leq \underbrace{\sum_{m=1}^M \sum_{h=1}^{H_m} \left[ c_{m,h} + \PP V_{j_m}(s_{m,h},a_{m,h}) - V_{j_m} (s_{m,h}) \right]}_{E_1}\notag\\
    &\qquad + \underbrace{\sum_{m=1}^M \sum_{h=1}^{H_m} \left[ V_{j_m}(s_{m,h+1}) - \PP V_{j_m} (s_{m,h},a_{m,h}) \right]}_{E_2} \notag 
    \\ & \qquad + \underbrace{ \sum_{m=1}^M\left( \sum_{h=1}^{H_m} V_{j_m}(s_{m,h}) - V_{j_m}(s_{m,h+1}) \right) - \sum_{m \in \cM(M)} V_{j_m}(\sinit) }_{E_3} + 1.
\end{align} 

Given the above interval decomposition of the regret, we can get the following theorem which serves as a master theorem similar to Theorem \ref{thm: regret upper bound with T}.

\begin{theorem}\label{thm: bernstein regret upper bound with T}
Under Assumption \ref{assump: linear kernel mdp}, \ref{assump: existence proper policy} and \ref{assump: c_min}, for any $\delta >0$, let $\rho=0$, $\lambda = 1/B^2$ and $\{\check\beta_t, \tilde\beta_t, \hat\beta_t\}_{t \geq 1}$ as given by \eqref{eq: choices of 3 betas}. Then with probability at least $1-7\delta$,  
    \begin{align*}
    R(M) & = \tilde\cO\left( \sqrt{B^2 d T + B^2 d^2 M + B^2 d^{3.5} T^{0.5}+\frac{B^3 d^4}{c_{\min}}M^{0.5} }\right) , 
\end{align*}where $\tilde\cO(\cdot)$ hides a term of  $C\cdot\log^2(TB/(\lambda \delta c_{\min}))$ for some problem-independent constant $C$.
\end{theorem}

\begin{proof}[Proof of Theorem \ref{thm: bernstein regret upper bound with T}]
    Please see Appendix \ref{sec: proof of bernstein upper bound with T}.
\end{proof}

Theorem \ref{thm: bernstein regret upper bound with cmin} can then be established using Theorem \ref{thm: bernstein regret upper bound with T}.




\subsection{Proof of Theorem \ref{thm: bernstein regret upper bound with cmin}}\label{sec: proof of bernstein regret upper bound with cmin}

We first introduce two useful results. 
The following lemma bound the total number of \texttt{DEVI} calls in Algorithm \ref{alg: linear kernel ssp bernstein}.

\begin{lemma}\label{lem: J bound bernstein}
    Let $J$ denote the total number of \texttt{DEVI} calls by Algorithm \ref{alg: linear kernel ssp bernstein}. Then on the event of Lemma \ref{lem: key bern lemma},
    \begin{align*}
        J \leq  2 d \log \left( 1 + \frac{T}{\lambda} \right) + 2 \log(T) . 
    \end{align*}
\end{lemma}

\begin{proof}
    The proof follows from the same analysis as Lemma \ref{lem: bound number of calls to EVI}, by noting that $\bSigma_{T} = \lambda \Ib+\sum_{t=1}^T \hat{\sigma}_t^{-2} \bphi_{V_j}(s_t,a_t)\bphi_{V_j}(s_t,a_t)^\top$ with $\|\bphi_{V_j}(s_t,a_t)/\hat\sigma_t\|_2 \leq \sqrt{d}$ since$\hat\sigma_t \geq B/\sqrt{d} $ and $|V_j|\leq \Bstar \leq B$ on the event of Lemma \ref{lem: key bern lemma}.
\end{proof}

The following lemma bounds the total number of intervals $M$, which is a direct result of the interval decomposition.
\begin{lemma}\label{lem: M relation with cost and J}
    Let $C_M$ denote the total cost over $M$ intervals. Then the total number of intervals $M$ satisfies
    \begin{align*}
        M \leq \frac{C_M}{B} + K + J.
    \end{align*}
\end{lemma}
\begin{proof}[Proof of Lemma \ref{lem: M relation with cost and J}]
This follows immediately from the three conditions in the interval decomposition \eqref{eq: bernstein regret decomposition}.
\end{proof}

We are now ready to prove Theorem \ref{thm: bernstein regret upper bound with cmin}.

\begin{proof}[Proof of Theorem \ref{thm: bernstein regret upper bound with cmin}]
Recall that $V^\star(\sinit)$ denotes the cost of the optimal policy. Then by \eqref{eq: regret def} and Theorem \ref{thm: bernstein regret upper bound with T}, we have 
\begin{align*}
    C_M &= R(M) + K\cdot V^\star(\sinit) \notag
    \\ & = \tilde\cO\left( \sqrt{B^2 d T + B^2 d^2 M + B^2 d^{3.5} T^{0.5}+\frac{B^3 d^4}{c_{\min}}M^{0.5} }\right) + K\cdot V^\star(\sinit),
\end{align*}where the first step holds since $R(M) = R_K$, and $\tilde\cO(\cdot)$ hides a term of  $C\cdot\log^2(TB/(\lambda \delta c_{\min}))$. 
Suppose $K\geq d^5 + B^2d^4/c_{\min}^2$. Then the above can be simplified to
\begin{align*}
    C_M = \tilde\cO\left( \sqrt{B^2 d T + B^2 d^2 M } \right) + K\cdot V^\star(\sinit). 
\end{align*}
We then have 
\begin{align}\label{eq: CM bound eq 1}
    C_M & \leq  C\log^2\left( \frac{TB}{\lambda \delta c_{\min}} \right) \cdot \left( \sqrt{B^2dT + \frac{B^2d^2C_M}{B}+B^2d^2K+B^2d^3\log(\frac{T}{\lambda})} \right) + K\cdot V^\star(\sinit) \notag
    \\ & \leq C\log^2\left( \frac{TB}{\lambda \delta c_{\min}} \right)\left( d\sqrt{B}\sqrt{C_M} + \sqrt{B^2dT + B^2d^2K + B^2d^3\log(\frac{T}{\lambda})} \right) + K\cdot V^\star(\sinit) \notag
    \\ & \leq C' \log^2\left( \frac{TB}{\lambda \delta c_{\min}} \right)\left( d\sqrt{B}\sqrt{C_M} + \sqrt{B^2dT + B^2d^2K} \right) + K\cdot V^\star(\sinit) ,
\end{align}where the first step is by Lemma \ref{lem: J bound bernstein} and \ref{lem: M relation with cost and J}, the second step is by $\sqrt{a+b} \leq \sqrt{a} + \sqrt{b}$ for $a,b >0$, and the last step is by $\lambda = 1/B^2$ and hence $B^2d^3 \log(T/\lambda) = \cO(B^2dT)$ for all $T > K  >B^2d^4$. 

We now apply the result that $c\leq a\sqrt{c}+b \implies c \leq (a+\sqrt{b})^2 $ for $a,b \geq 0$, and get from \eqref{eq: CM bound eq 1} that
\begin{align}\label{eq: CM bound eq 2}
    C_M = \cO\left( d^2B \log^4 \left( \frac{TB}{\lambda \delta c_{\min}} \right) + \log^2\left( \frac{TB}{\lambda \delta c_{\min}} \right) \sqrt{B^2dT + B^2d^2K}   \right)+ K\cdot V^\star(\sinit). 
\end{align}Furthermore, since $C_M \geq c_{\min}T$, we have 
\begin{align}\label{eq: bernstein T bound eq 1}
    c_{\min} \cdot T & = \cO\left( d^2B \log^4 \left( \frac{TB}{\lambda \delta c_{\min}} \right) + \log^2\left( \frac{TB}{\lambda \delta c_{\min}} \right) \sqrt{B^2dT + B^2d^2K} \right) + K\cdot V^\star(\sinit)  .
\end{align} 
Applying $\sqrt{a+b}\leq \sqrt{a}+\sqrt{b}$ and $c\leq a\sqrt{c}+b \implies c \leq (a+\sqrt{b})^2 $ to \eqref{eq: bernstein T bound eq 1}, we can further simplify and get that for all sufficiently large $K$ and $T$, 
\begin{align}\label{eq: bernstein T bound eq 2}
    T & = \cO\left( \frac{B^2 d}{c_{\min}^2} \log^4 \left( \frac{KB}{\lambda \delta c_{\min}} \right) + \frac{Bd\sqrt{K}}{c_{\min}} \log^2 \left( \frac{KB}{\lambda \delta c_{\min}} \right) \right) + \frac{K\cdot V^{\star}(\sinit)}{c_{\min}} \notag
    \\ & = \cO\left( \frac{KB}{c_{\min}} \log^2 \left( \frac{KB}{\lambda \delta c_{\min}} \right) \right),
\end{align}where the second step is by $V^{\star}(\sinit) \leq \Bstar \leq B$.  Plug \eqref{eq: bernstein T bound eq 2} into \eqref{eq: CM bound eq 2} and we get that 
\begin{align}
    R_K = \cO\left( d^2B \zeta^2 + Bd\sqrt{K} \zeta + B^{1.5}\sqrt{d} \sqrt{\frac{K}{c_{\min}}}\zeta\right) , 
\end{align}where $\zeta = \log^2(KB/(\lambda \delta c_{\min}))$. This finishes the proof.

\end{proof}

\subsection{Proof of Lemma \ref{lem: key bern lemma}}\label{sec: proof of key bern lemma}
We need the following lemma.

\begin{lemma}\label{lem: distance variance cs}
For any $t$ and $j = j(t)$, let $V_j$, $\hat\btheta_t$, $\bsigma_t$, $\tilde\btheta_t$, $\tilde\bSigma_t$ be as given in Algorithm \ref{alg: linear kernel ssp bernstein}. If it further holds that $|V_j| \leq B$, then we have
\begin{align*}
    \left|\hat\VV_t V_j (s_t,a_t) - \VV V_j(s_t,a_t) \right| & \leq \min\left\{ B^2, \left\| \tilde\bSigma_{t-1}^{1/2} (\btheta^* - \tilde\btheta_{t-1})\right\|_2 \left\| \tilde\bSigma_{t-1}^{-1/2} \bphi_{V_j^2}(s_t,a_t)\right\|_2 \right\}
    \\ & \qquad \qquad + \min\left\{ B^2 , 2B \left\| \bSigma_{t-1}^{1/2} (\btheta^* - \hat\btheta_{t-1})\right\|_2 \left\| \bSigma_{t-1}^{-1/2} \bphi_{V_j}(s_t,a_t)\right\|_2 \right\}.
\end{align*}
\end{lemma}

\begin{proof}[Proof of Lemma \ref{lem: distance variance cs}]

    By \eqref{eq: VV and E} we can write 
    \begin{align*}
        & \left|\hat\VV_t V_j (s_t,a_t) - \VV V_j(s_t,a_t) \right| 
        \\ & = \left| [\langle \bphi_{V_j^2}(s_t, a_t), \tilde\btheta_{t-1} \rangle]_{[0, B^2]} - \langle \bphi_{V_j^2}(s_t, a_t), \btheta^* \rangle  + \langle \bphi_{V_j}(s_t, a_t), \btheta^* \rangle^2 - (\langle \bphi_{V_j}(s_t, a_t), \hat\btheta_{t-1} \rangle_{[0, B]})^2 \right|
        \\ & \leq \underbrace{\left| [\langle \bphi_{V_j^2}(s_t, a_t), \tilde\btheta_{t-1} \rangle]_{[0, B^2]} - \langle \bphi_{V_j^2}(s_t, a_t), \btheta^* \rangle \right|}_{I_1}  + \underbrace{\left| \langle \bphi_{V_j}(s_t, a_t), \btheta^* \rangle^2 - (\langle \bphi_{V_j}(s_t, a_t), \hat\btheta_{t-1} \rangle_{[0, B]})^2 \right|}_{I_2}.
    \end{align*}
    
    To bound $I_1$, we have 
    \begin{align*}
        I_1 \leq \left| \langle \bphi_{V_j^2}(s_t, a_t), \tilde{\btheta}_{t-1} - \btheta^* \rangle \right| \leq \left\| \tilde\bSigma_{t-1}^{1/2} (\btheta^* - \tilde\btheta_{t-1})\right\|_2 \left\| \tilde\bSigma_{t-1}^{-1/2} \bphi_{V_j^2}(s_t,a_t)\right\|_2 .
    \end{align*} The first step holds because both terms are in $[0,B^2]$: the first term is truncated to be in $[0, B^2]$, and the second term satisfies $|\langle \bphi_{V_j^2}(s_t, a_t), \btheta^* \rangle |= |\PP V_j^2 (s_t,a_t)| \leq B^2$ since $|V_j| \leq B$ by assumption.
    The second step is by Cauchy-Schwarz inequality. It follows that 
    \begin{align*}
        I_1  \leq  \min\left\{ B^2, \left\| \tilde\bSigma_{t-1}^{1/2} (\btheta^* - \tilde\btheta_{t-1})\right\|_2 \left\| \tilde\bSigma_{t-1}^{-1/2} \bphi_{V_j^2}(s_t,a_t)\right\|_2 \right\} . 
    \end{align*}
    
    To bound $I_2$, we have 
    \begin{align*}
        I_2 & = \left| \langle \bphi_{V_j}(s_t, a_t), \btheta^* \rangle - [\langle \bphi_{V_j}(s_t, a_t), \hat\btheta_{t-1} \rangle]_{[0, B]} \right| \cdot \left| \langle \bphi_{V_j}(s_t, a_t), \btheta^* \rangle + [\langle \bphi_{V_j}(s_t, a_t), \hat\btheta_{t-1} \rangle]_{[0, B]} \right| 
        \\ & \leq 2 B \left| \langle \bphi_{V_j}(s_t, a_t), \btheta^* - \hat\btheta_{t-1} \rangle \right|
        \\ & \leq 2B \left\| \bSigma_{t-1}^{1/2} (\btheta^* - \hat\btheta_{t-1})\right\|_2 \left\| \bSigma_{t-1}^{-1/2} \bphi_{V_j}(s_t,a_t)\right\|_2,
    \end{align*} where the first inequality holds because both terms are bounded by $B$, and the second step is by Cauchy-Schwarz inequality. Thus we have 
    \begin{align*}
        I_2 \leq \min\left\{ B^2 , 2B \left\| \bSigma_{t-1}^{1/2} (\btheta^* - \hat\btheta_{t-1})\right\|_2 \left\| \bSigma_{t-1}^{-1/2} \bphi_{V_j}(s_t,a_t)\right\|_2 \right\}.
    \end{align*}The result follows from combining the bounds of $I_1$ and $I_2$.
\end{proof}

Let's define to sequences of sets
\begin{align*}
    \check\cC_t & \coloneqq \left\{ \btheta: \left\| \bSigma_t^{1/2} (\btheta - \hat\btheta_t) \right\|_2 \leq \check\beta_t \right\} , 
    \\ \tilde\cC_t & \coloneqq \left\{ \btheta: \left\| \tilde\bSigma_t^{1/2} (\btheta - \tilde\btheta_t) \right\|_2 \leq \tilde\beta_t \right\} .
\end{align*}

In the following proof, for each $t\geq 1$, we define the confidence ellipsoid $\hat\cC_{(t)}$ as 
\begin{align*}
    \hat\cC_{(t)} = \left\{ \btheta: \left\| \bSigma_t^{1/2} (\btheta - \hat\btheta_t) \right\|_2 \leq \hat\beta_t \right\},
\end{align*}for all $t\geq 1$. Note that the confidence ellipsoid $\hat\cC_j$ in Algorithm \ref{alg: linear kernel ssp bernstein} are indexed by $j$. By definition, we have $\hat\cC_j = \hat\cC_{(t_j)}$ for all $j$, where $t_j$ defined by line \ref{algline: def t_j bernstein} is the time step when the $j$-th \texttt{DEVI} is triggered. 

We now prove Lemma \ref{lem: key bern lemma}. 

\begin{proof}[Proof of Lemma \ref{lem: key bern lemma}]
    We prove by an induction argument. 
    
    \textbf{Special case $t=1$.} This is a special case since the variance function estimate used for the step $t=1$ is $V_0$, which is from the initialization of Algorithm \ref{alg: linear kernel ssp bernstein} instead of the \texttt{DEVI} output. So we treat this case separately.
    From the initialization of Algorithm \ref{alg: linear kernel ssp bernstein}, we have $V_0 = 1 \leq \Bstar$. Note that the initial stage $j=0$ only contains one step $t=1$. We first show that with probability at least $1-3\delta/2$, for $t=1$ and $j=j(t)$, the event defined by \eqref{eq: event of key bern lemma} holds (except for the optimism, i.e. $0\leq Q_{j}(\cdot,\cdot) \leq Q^{\star}(\cdot,\cdot)$ and hence $0 \leq V_{j}(\cdot) \leq V^{\star}(\cdot)$). 
    
    We show $\btheta^* \in \check\cC_t$ for $t=1$ by using Theorem \ref{thm: bernstain vector martingale}. Let $\xb_t = \hat\sigma_t^{-1} \bphi_{V_j}(s_t, a_t)$ and $\eta_t = \hat\sigma_t^{-1} V_j(s_{t+1})- \hat\sigma_t^{-1}\langle \bphi_{V_j}(s_t,a_t) , \btheta^* \rangle = 0$ since $V_j$ is constant for $j=0$, $\bmu^* = \btheta^*$, $y_t = \langle \bmu^*, \xb_t \rangle + \eta_t$, $\Zb_t = \lambda \bI + \sum_{t'=1}^t \xb_{t'} \xb_{t'}^\top$, $\wb_t = \sum_{t'=1}^t \xb_{t'}y_{t'}$ and $\bmu_t = \Zb_t^{-1} \wb_t$. Then we have $y_t = \hat\sigma_t^{-1} V_j(s_{t+1})$ and $\bmu_t = \hat\btheta_t$. Furthermore, the following holds almost surely:
    \begin{align*}
        \| \xb_t\|_2 \leq \hat\sigma_t^{-1}\cdot 1  \leq \sqrt{d},  \ |\eta_t| = 0, \ \EE[\eta_t \mid \cF_t] = 0, \ \EE [\eta_t^2 \mid \cF_t] =0.
    \end{align*}
    Then by Theorem \ref{thm: bernstain vector martingale}, with probability at least $1-\delta/2$, for $t=1$, it holds that
    \begin{align*}
        \left\| \btheta^* - \hat\btheta_t \right\|_{\bSigma_t} \leq 8 d \sqrt{\log(1+t/\lambda) \log(8t^2/\delta) } + 4 \sqrt{d} \log(8t^2/\delta) + \sqrt{\lambda d} \leq \check\beta_t,
    \end{align*}which implies $\btheta^* \in \check\cC_t$ for $t=1$. 
    
    We then show $\btheta^* \in \tilde\cC_t$ for $t=1$ with probability at least $1-\delta/2$. We apply Theorem \ref{thm: bernstain vector martingale} with $\xb_t = \bphi_{V_j^2}(s_t, a_t)$ and $\eta_t =  V_j^2(s_{t+1})- \langle \bphi_{V_j^2}(s_t,a_t), \btheta^* \rangle = 0$, $\bmu^* = \btheta^*$, $y_t = \langle \bmu^*, \xb_t \rangle + \eta_t$, $\Zb_t = \lambda \bI + \sum_{t'=1}^t \xb_{t'} \xb_{t'}^\top$, $\wb_t = \sum_{t'=1}^t \xb_{t'}y_{t'}$ and $\bmu_t = \Zb_t^{-1} \wb_t$. Then we have $y_t = V_j^2(s_{t+1})$ and $\bmu_t = \tilde\btheta_t$. Furthermore, it holds almost surely that:
    \begin{align*}
        \| \xb_t\|_2 \leq B^2,  \ |\eta_t| = 0, \ \EE[\eta_t \mid \cF_t] = 0, \ \EE [\eta_t^2 \mid \cF_t] = 0. 
    \end{align*}Then by Theorem \ref{thm: bernstain vector martingale}, with probability at least $1-\delta/2$, for $t=1$, it holds that
    \begin{align*}
        \left\| \btheta^* - \tilde\btheta_t \right\|_{\tilde\bSigma_t} \leq 8  \sqrt{dB^4 \log(1+ \frac{tB^4}{d\lambda} ) \log(8t^2/\delta) } + 4 B^2 \log(8t^2/\delta) + \sqrt{\lambda d} \leq \tilde\beta_t,
    \end{align*}which implies $\btheta^* \in \tilde\cC_t$. 
    
    We then show $\btheta^* \in \hat\cC_{(t)}$ for $t=1$ with probability at least $1-\delta/2$. Let $\xb_t = \hat\sigma_t^{-1} \bphi_{V_j}(s_t, a_t)$ and $\eta_t = \hat\sigma_t^{-1}\left[ V_j(s_{t+1})- \langle \bphi_{V_j}(s_t,a_t) , \btheta^* \rangle \right]$.
    Let $\bmu^* = \btheta^*$, $y_t = \langle \bmu^*, \xb_t \rangle + \eta_t$, $\Zb_t = \lambda \bI + \sum_{t'=1}^t \xb_{t'} \xb_{t'}^\top$, $\wb_t = \sum_{t'=1}^t \xb_{t'}y_{t'}$ and $\bmu_t = \Zb_t^{-1} \wb_t$. Then we have $y_t = \hat\sigma_t^{-1} V_j(s_{t+1})$ and $\bmu_t = \hat\btheta_t$. Furthermore, the following holds almost surely:
    \begin{align*}
        \| \xb_t\|_2 \leq \hat\sigma_t^{-1}\cdot 1  \leq \sqrt{d},  \ |\eta_t| = 0, \ \EE[\eta_t \mid \cF_t] = 0, \ \EE [\eta_t^2 \mid \cF_t] =0.
    \end{align*}
    Then by Theorem \ref{thm: bernstain vector martingale}, with probability at least $1-\delta/2$, for $t=1$, it holds that
    \begin{align*}
        \left\| \btheta^* - \hat\btheta_t \right\|_{\bSigma_t} \leq 8  \sqrt{d\log(1+t/\lambda) \log(8t^2/\delta) } + 4 \sqrt{d} \log(8t^2/\delta) + \sqrt{\lambda d} \leq \hat\beta_t,
    \end{align*} which implies $\btheta^* \in \hat\cC_{(t)}$ for $t=1$. By a union bound, we get that, with probability at least $1-3\delta/2$, for $t=1$, we have 
    \begin{align*}
        \theta^* \in \check\cC_t \cap \tilde\cC_t \cap \hat\cC_{(t)} \cap \cB. 
    \end{align*}
    It remains to show $\left|\hat{\VV}_t V_j(s_t,a_t) - \VV V_j(s_t, a_t) \right|\leq E_t$. But this is trivial since $\hat\VV V_j = 0$ by the initialization of $\tilde\btheta_t$ and $\hat\btheta_t$, and $\VV V_j = 0$ by $V_j(\cdot) = 0$, for $t=1$ and $j=j(t)=0$.
    
    Note that by the time step doubling criterion,  at the end of step $t=1$, \texttt{DEVI} would be triggered and output the value function estimate $V_j$ for $j=1$.
    We now show that under the above event, the optimism holds for $j=1$, i.e. $0 \leq Q_j \leq Q^{\star}$ and hence $ 0 \leq V_j \leq V^{\star}$.
    We prove by applying another induction argument on the loop of \texttt{DEVI}. For the base step, note that by non-negativity of $Q^\star$ and $V^\star$, we have $Q^{(0)}\leq Q^{\star}$ and $V^{(0)} \leq V^\star$. For the induction hypothesis, we assume $Q^{(i)}$ and $V^{(i)}$ are optimistic for some $i$. We want to show $Q^{(i+1)}$ and $V^{(i+1)}$ are optimistic. For the $(i+1)$-th iteration, we have 
    \begin{align}\label{eq: key bern lemma optimism devi loop}
    Q^{(i+1)} (\cdot,\cdot) & = c(\cdot,\cdot) + (1-q) \cdot \min_{\btheta \in \hat\cC_1 \cap \cB} \langle \btheta , \bphi_{V^{(i)}} (\cdot,\cdot) \rangle \notag
    \\ & \leq c(\cdot,\cdot) + (1-q)\cdot \PP V^{(i)}(\cdot,\cdot) \notag
    \\ & \leq c(\cdot,\cdot) +   \PP V^{(i)}(\cdot,\cdot)  \notag
    \\ & \leq Q^\star(\cdot,\cdot),
    \end{align}where the first step is because we are considering the case where $\rho= 0$, the second step is because we are taking the minimum over a set that contains $\btheta^*$, the third step is by non-negativity of $\PP V^{(i)}(\cdot,\cdot)$, and the last step is by the Bellman optimal condition \eqref{eq: bellman optimal condition} and the induction hypothesis that $V^{(i)}$ is optimistic. By induction, we conclude that $Q^{(i)}$ is optimistic for all $i$, and thus the final output $Q_1(\cdot,\cdot)$ and thus $V_1(\cdot)$ are both optimistic.

    \textbf{Initial step $t\in[t_1+1, t_{2}]$.} Recall that in this round, the value function estimate that is used to collect the data is $V_1$, which satisfies $V_1 \leq \Bstar$ by the optimism proved above.
    
    We first show that with probability at least $1-\frac{3\delta}{t_2(t_2+1)}$, for all $t\in[t_1+1, t_2]$, the event in Lemma \ref{lem: key bern lemma} holds. The proof is actually similar to that of the $t=1$ case. 
    We show $\btheta^* \in \check\cC_t$ by using Theorem \ref{thm: bernstain vector martingale}. Let $\xb_t = \hat\sigma_t^{-1} \bphi_{V_j}(s_t, a_t)$ and $\eta_t = \hat\sigma_t^{-1} V_j(s_{t+1})- \hat\sigma_t^{-1}\langle \bphi_{V_j}(s_t,a_t), \btheta^* \rangle $, $\bmu^* = \btheta^*$, $y_t = \langle \bmu^*, \xb_t \rangle + \eta_t$, $\Zb_t = \lambda \bI + \sum_{t'=1}^t \xb_{t'} \xb_{t'}^\top$, $\wb_t = \sum_{t'=1}^t \xb_{t'}y_{t'}$ and $\bmu_t = \Zb_t^{-1} \wb_t$. Then we have $y_t = \hat\sigma_t^{-1} V_j(s_{t+1})$ and $\bmu_t = \hat\btheta_t$, and the following holds almost surely:
    \begin{align*}
        \| \xb_t\|_2 \leq \hat\sigma_t^{-1} \Bstar  \leq \sqrt{d},  \ |\eta_t| \leq \sqrt{d}, \ \EE[\eta_t \mid \cF_t] = 0, \ \EE [\eta_t^2 \mid \cF_t] \leq d,
    \end{align*}
    where the first and the second inequalities hold because $V_j \leq \Bstar\leq B$ for $t\in[t_1+1,t_2]$ and thus $\|\bphi_{V_j}\|\leq \Bstar$.
    By Theorem \ref{thm: bernstain vector martingale}, with probability at least $1-\frac{\delta}{t_2(t_2+1)}$, for all $t\in[t_1 + 1, t_2]$, it holds that
    \begin{align*}
        \left\| \btheta^* - \hat\btheta_t \right\|_{\bSigma_t} \leq 8 d \sqrt{\log(1+t/\lambda) \log(32t^4/\delta) } + 4 \sqrt{d} \log(32t^4/\delta) + \sqrt{\lambda d} = \check\beta_t,
    \end{align*} where we use the fact that $4t^2 t_2(t_2+1) \leq 32t^4$ for $t\in[t_1+1, t_2]$, since $t_2 \leq 2 t_1 \leq 2 t$ by the doubling time step criterion.  
    Thus we conclude that with the stated probability, $\btheta^* \in \check\cC_t$. 
    
    To show $\btheta^* \in \tilde\cC_t$, we apply Theorem \ref{thm: bernstain vector martingale} with $\xb_t = \bphi_{V_j^2}(s_t, a_t)$ and $\eta_t =  V_j^2(s_{t+1})- \langle \bphi_{V_j^2}(s_t,a_t), \btheta^* \rangle $, $\bmu^* = \btheta^*$, $y_t = \langle \bmu^*, \xb_t \rangle + \eta_t$, $\Zb_t = \lambda \bI + \sum_{t'=1}^t \xb_{t'} \xb_{t'}^\top$, $\wb_t = \sum_{t'=1}^t \xb_{t'}y_{t'}$ and $\bmu_t = \Zb_t^{-1} \wb_t$. Then we have $y_t = V_j^2(s_{t+1})$ and $\bmu_t = \tilde\btheta_t$, and the following holds almost surely:
    \begin{align*}
        \| \xb_t\|_2 \leq B^2,  \ |\eta_t| \leq B^2, \ \EE[\eta_t \mid \cF_t] = 0, \ \EE [\eta_t^2 \mid \cF_t] \leq B^4. 
    \end{align*}Then by Theorem \ref{thm: bernstain vector martingale}, with probability at least $1-\frac{\delta}{t_2(t_2+1)}$, for all $t\in[t_1+1,t_2]$, it holds that
    \begin{align*}
        \left\| \btheta^* - \tilde\btheta_t \right\|_{\tilde\bSigma_t} \leq 8  \sqrt{dB^4 \log(1+ \frac{tB^4}{d\lambda} ) \log(32t^4/\delta) } + 4 B^2 \log(32t^4/\delta) + \sqrt{\lambda d} = \tilde\beta_t.
    \end{align*}To show $\btheta^* \in \hat\cC_{(t)}$, we apply Theorem \ref{thm: bernstain vector martingale} with $\xb_t = \hat\sigma_t^{-1} \bphi_{V_j}(s_t, a_t)$ and 
    \begin{align*}
        \eta_t =  \hat\sigma_t^{-1} \left[ V_j(s_{t+1})- \langle \bphi_{V_j}(s_t,a_t), \btheta^* \rangle \right] \ind\{\btheta^* \in \check\cC_t \cap \tilde\cC_t\},
    \end{align*}
    $\bmu^* = \btheta^*$, $y_t = \langle \bmu^*, \xb_t \rangle + \eta_t$, $\Zb_t = \lambda \bI + \sum_{t'=1}^t \xb_{t'} \xb_{t'}^\top$, $\wb_t = \sum_{t'=1}^t \xb_{t'}y_{t'}$ and $\bmu_t = \Zb_t^{-1} \wb_t$. Then we have $y_t = \hat\sigma_t^{-1} V_j(s_{t+1})$ and $\bmu_t = \hat\btheta_t$. Furthermore, it holds almost surely that
    \begin{align*}
        & \EE [\eta_t^2 \mid \cF_t] 
        \\ & = \hat\sigma_t^{-2} \ind\{\btheta^* \in \check\cC_t \cap \tilde\cC_t\} [\VV V_j](s_t,a_t) 
        \\ & \leq \hat\sigma_t^{-2} \ind\{\btheta^* \in \check\cC_t \cap \tilde\cC_t\} \bigg[ \hat\VV_t V_j(s_t,a_t) + \min\left\{ B^2, \left\| \tilde\bSigma_t^{1/2} (\btheta^* - \tilde\btheta_t)\right\|_2 \left\| \tilde\bSigma_t^{-1/2} \bphi_{V_j^2}(s_t,a_t)\right\|_2 \right\}
    \\ & \qquad \qquad + \min\left\{ B^2 , 2B \left\| \bSigma_t^{1/2} (\btheta^* - \hat\btheta_t)\right\|_2 \left\| \bSigma_t^{-1/2} \bphi_{V_j}(s_t,a_t)\right\|_2 \right\} \bigg]
    \\ & \leq \hat\sigma_t^{-2} \bigg[ \hat\VV_t V_j(s_t,a_t) + \min\left\{ B^2, \tilde\beta_t \left\| \tilde\bSigma_t^{-1/2} \bphi_{V_j^2}(s_t,a_t)\right\|_2 \right\}  + \min\left\{ B^2 , 2B \check\beta_t \left\| \bSigma_t^{-1/2} \bphi_{V_j}(s_t,a_t)\right\|_2 \right\} \bigg]
    \\ & = 1 ,
    \end{align*}
    where the first inequality holds due to Lemma \ref{lem: distance variance cs} and $V_j \leq B$, the second inequality is due to the event in the indicator function, and the last step in by the definition of $\hat\sigma_t$. We then use Theorem~\ref{thm: bernstain vector martingale} and get that with probability at least $1- \frac{\delta}{t_2(t_2+1)}$, for all $t \in [t_1+1, t_2]$,
    \begin{align}\label{eq: bound of event 1-1}
        \left\| \btheta^* - \hat\btheta_t \right\|_{\bSigma_t} \leq 8  \sqrt{d \log(1+ t/\lambda ) \log(32t^4/\delta) } + 4 \sqrt{d} \log(32t^4/\delta) + \sqrt{\lambda d} = \hat\beta_t.
    \end{align}
    Denote by $\cE'_1$ the event where $\btheta^* \in \check\cC_t \cap \tilde\cC_t$ and \eqref{eq: bound of event 1-1} holds for all $t\in[t_1+1, t_2]$. Then on this event, we have 
    \begin{align*}
        y_t & = \langle \btheta^*, \hat\sigma_t^{-1} \bphi_{V_j}(s_t, a_t)\rangle + \hat\sigma_t^{-1} \left[ V_j(s_{t+1})- \langle \bphi_{V_j}(s_t,a_t), \btheta^* \rangle \right] = \hat\sigma_t^{-1} V_j(s_{t+1}).
    \end{align*}Therefore, the above shows that $\btheta^* \in \hat\cC_{(t)}$ for all $t\in[t_1+1, t_2]$ holds on the event $\cE'_1$. Finally, by union bound, we get that with probability at least $1-\frac{3\delta}{t_2(t_2+1)}$, for all $t\in[t_1+1, t_2]$, it holds that
    \begin{align*}
        \btheta^* \in \check\cC_t\cap \tilde\cC_t\cap\hat\cC_{(t)}\cap\cB.
    \end{align*}
    To show that the optimism holds under the above event, i.e. $0\leq Q_j \leq Q^\star$ and $0\leq V_j \leq V^\star$ for $j=2$, note that this is the same as the proof for the $j=1$ case. Indeed, the same induction trick on the loop of \texttt{DEVI} can be applied with the set $\hat\cC_{(t_1)}$ in \eqref{eq: key bern lemma optimism devi loop} being replaced by $\hat\cC_{(t_2)}$. 
    
    To show $\left|\hat{\VV}_t V_j(s_t,a_t) - \VV V_j(s_t, a_t) \right|\leq E_t$ under the above event, we apply the definition of $E_t$, Lemma \ref{lem: distance variance cs} and the fact that $V_1 \leq B_\star \leq B$. 
    
    From the case $t=1$ and $t\in[t_1+1, t_2]$, we get that, with probability at least $1- \frac{3\delta}{t_1(t_1+1)} - \frac{3\delta}{t_2(t_2+1)} $, for all $t \in [1, t_2]$ and $j = j(t)$, the following event holds
    \begin{align}\label{eq: proof of key bern lemma event}
        \btheta^* \in \check\cC_t\cap \tilde\cC_t\cap\hat\cC_{(t)}\cap\cB, \quad 0\leq Q_{j}(\cdot,\cdot) \leq Q^{\star}(\cdot,\cdot), \quad \left|\hat{\VV}_t V_{j}(s_t,a_t) - \VV V_{j}(s_t, a_t) \right|\leq E_t .
    \end{align}
    
    \noindent\textbf{Induction step.} Suppose that, with probability at least $1-\delta'$, for some $j-1\geq 2$ and $t \in [1, t_{j-1}]$, the event in \eqref{eq: proof of key bern lemma event} holds. We want to show \eqref{eq: proof of key bern lemma event} also holds for $j$ and $t \in [t_{j-1}+1, t_j]$, with probability at least $1-\delta' - \frac{3\delta}{t_j(t_j+1)}$. However, this immediately follows from the exact analysis of the initial step $t\in[t_1+1, t_2]$. Indeed, all we need is $V_{j-1}$ is an optimistic estimate of $V_\star$, which holds by the induction hypothesis, and the probability comes from a union bound. 
    
    Finally, we conclude by mathematical induction that the event \eqref{eq: proof of key bern lemma event} holds with the probability at least the following:
    \begin{align*}
        1 - \sum_{j=1}^J \frac{3\delta}{t_j(t_j+1)} = 1 - 3\delta \sum_{j=1}^J \left( \frac{1}{t_j} - \frac{1}{t_j+1} \right),
    \end{align*}which is lower bounded by $1-3\delta$. This implies the event \eqref{eq: event of key bern lemma} holds with probability at least $1-3\delta$, since the event \eqref{eq: proof of key bern lemma event} is a subset of the event \eqref{eq: event of key bern lemma}.
    Furthermore, the finite time convergence of \texttt{DEVI} follows from the contraction property. This completes the proof.
    
\end{proof}

\subsection{Proof of Theorem \ref{thm: bernstein regret upper bound with T}}\label{sec: proof of bernstein upper bound with T}

We bound $E_1$, $E_2$ and $E_3$ separately. 

\subsubsection{Bounding $E_1$}
Following the same reasoning as in Section \ref{sec: E1 bound}, we write $E_1$ as 
\begin{align*}
    E_1 & = \sum_{m=1}^M \sum_{h=1}^{H_m} \left[ c_{m,h} + \PP V_{j_m}(s_{m,h},a_{m,h}) - Q_{j_m} (s_{m,h},a_{m,h}) \right],
\end{align*} 
where 
\begin{align*}
    & c_{m,h} + \PP V_{j_m} (s_{m,h},a_{m,h}) - Q_{j_m} (s_{m,h},a_{m,h}) 
    \\ & \leq \langle \btheta^* - \btheta_{m,h} , \bphi_{V_{j_m}} (s_{m,h},a_{m,h}) \rangle + \frac{\Bstar + 1-q}{t_{j_m}}, 
\end{align*}where we use the optimism $V_{j_m} \leq V^\star \leq \Bstar$ under the event of Lemma \ref{lem: key bern lemma}, and $q = 1/t_{j_m}$ according to Algorithm \ref{alg: linear kernel ssp bernstein}.

Recall the definition of $\cM_0(M)$ being the set of $m$ such that $j_m\geq 1$, i.e., $\cM_0(M) = \{m\leq M: \ j_m \geq 1\}$. Then we use the following
\begin{align}\label{eq: bern bound E1 is A1 and A2}
    & \sum_{m\in\cM_0(M)} \sum_{h=1}^{H_m} \left[ c_{m,h} + \PP V_{j_m}(s_{m,h},a_{m,h}) - Q_{j_m} (s_{m,h},a_{m,h}) \right] \notag 
    \\ & \leq \underbrace{\sum_{m\in \cM_0(M)} \sum_{h=1}^{H_m} \langle \btheta^* - \btheta_{m,h},\bphi_{V_{j_m}} ( s_{m,h},a_{m,h} ) \rangle}_{A_1} + \underbrace{(\Bstar+1)\cdot \sum_{m\in \cM_0(M)} \sum_{h=1}^{H_m} \frac{1}{t_{j_m}}}_{A_2} . 
\end{align}

\noindent\textbf{To bound $A_1$}: Recall that $\hat\btheta_{t_{j_m}}$ given by Line \ref{algline: hat btheta} is the center of the confidence ellipsoid $\cC_{j_m}$. First for each term $ \langle \btheta^* - \btheta_{m,h} , \bphi_{V_{j_m}}(s_{m,h},a_{m,h}) \rangle $ in $A_1$, we write 
\begin{align}\label{eq: bernstein E1 bound A1 1}
  &  \langle \btheta^* - \hat\btheta_{j_m} + \hat\btheta_{j_m} - \btheta_{m,h},\bphi_{V_{j_m}} ( s_{m,h},a_{m,h} ) \rangle \notag
    \\ & \leq 4 \hat\beta_T  \| \bphi_{V_{j_m}}(s_{m,h},a_{m,h}) \|_{\bSigma_{t(m,h)}^{-1}} \notag
    \\ & = 4 \hat\beta_T  \| \bphi_{V_{j_m}}(s_{m,h},a_{m,h})/\hat{\sigma}_{t(m,h)} \|_{\bSigma_{t(m,h)}^{-1}} \cdot \hat{\sigma}_{t(m,h)} , 
\end{align} where the inequality follows from the same reasoning as in \eqref{eq: E1 bound A1 1}, with the confidence ellipsoids $\cC_{j_m}$ being replaced by $\hat\cC_{j_m}$, and $\beta_T$ replaced by $\hat\beta_T$.
Also note that for each term $ \langle \btheta^* - \btheta_{m,h} , \bphi_{V_{j_m}}(s_{m,h},a_{m,h}) \rangle $ in $A_1$, we have
\begin{align}\label{eq: bernstein E1 bound A1 2}
    \langle \btheta^* - \btheta_{m,h} , \bphi_{V_{j_m}}(s_{m,h},a_{m,h}) \rangle&\leq  \langle \btheta^* , \bphi_{V_{j_m}}(s_{m,h},a_{m,h}) \rangle\notag\\
    &=\PP V_{j_m}(s_{m,h},a_{m,h})\notag\\
    &\leq \Bstar,
\end{align}
where both inequalities hold because $\btheta^*$ and $\btheta_{m,h}$ are parameters of some transition kernels, and $0\leq V_{j_m} (\cdot) \leq \Bstar$ under the event of Lemma \ref{lem: key bern lemma}.

Combining \eqref{eq: bernstein E1 bound A1 1} and \eqref{eq: bernstein E1 bound A1 2}, we can bound $A_1$ as 
\begin{align}\label{eq: bernstein E1 bound A1 3}
    A_1 & \leq  \sum_{m\in\cM_0} \sum_{h=1}^{H_m} \min \left\{\Bstar, \ 4 \hat\beta_T\| \bphi_{V_{j_m}}(s_{m,h},a_{m,h})/\hat{\sigma}_{t(m,h)} \|_{\bSigma_{t(m,h)}^{-1}} \hat{\sigma}_{t(m,h)} \right\} \notag 
    \\ & \leq \sum_{m\in\cM_0} \sum_{h=1}^{H_m}  \left( \Bstar + 4 \hat\beta_T \hat\sigma_{t(m,h)} \right) \min\left\{1, \| \bphi_{V_{j_m}}(s_{m,h},a_{m,h})/\hat{\sigma}_{t(m,h)} \|_{\bSigma_{t(m,h)}^{-1}} \right\} \notag
    \\ & \leq \sqrt{\sum_{m\in\cM_0} \sum_{h=1}^{H_m} \left( \Bstar + 4 \hat\beta_T \hat\sigma_{t(m,h)} \right)^2 } \cdot \sqrt{\sum_{m\in\cM_0} \sum_{h=1}^{H_m} \min\left\{1, \| \bphi_{V_{j_m}}(s_{m,h},a_{m,h})/\hat{\sigma}_{t(m,h)} \|^2_{\bSigma_{t(m,h)}^{-1}} \right\} } ,
\end{align}where the second inequality holds since $\min\{ a_1a_2 , b_1b_2 \} \leq (a_1+b_1) \min\{a_2, b_2\}$ for $a_1, a_2, b_1, b_2 >0$, and the third inequality is by Cauchy-Schwarz inequality.
To further bound the R.H.S. of \eqref{eq: bernstein E1 bound A1 3}, note that 
\begin{align}\label{eq: bernstein E1 bound A1 4}
    & \sum_{m\in\cM_0} \sum_{h=1}^{H_m} \min\left\{1, \| \bphi_{V_{j_m}}(s_{m,h},a_{m,h})/\hat{\sigma}_{t(m,h)} \|^2_{\bSigma_{t(m,h)}^{-1}} \right\}\notag
    \\ & \leq  2 \left[ d \log \left( \frac{\textnormal{trace}(\lambda \Ib)+T\cdot\max_{m,h}\|\bphi_{V_{j_m}}(s_{m,h},a_{m,h})/\hat{\sigma}_{t(m,h)}\|_2^2}{d}\right) - \log\left( \det(\lambda \Ib)\right)\right]\notag
    \\ & \leq 2d \log\left( \frac{\lambda d + T d}{\lambda d}\right) \notag
    \\ & = 2d\log\left( 1 + T/\lambda \right), 
\end{align}where the first inequality holds by Lemma \ref{lem: lemma 11 in abbasi}, and the second inequality holds because $V_{j_m}(\cdot) \leq \Bstar$ under Lemma \ref{lem: key bern lemma}, $\hat\sigma_t \leq B/\sqrt{d}$ by Line \ref{algline: hat sigma}, and thus $\max_{m,h}\|\bphi_{V_{j_m}}(s_{m,h},a_{m,h})/\hat{\sigma}_{t(m,h)}\|_2 \leq \sqrt{d} $ by Assumption \ref{assump: linear kernel mdp}. Furthermore, we have 
\begin{align}\label{eq: bernstein E1 bound A1 5}
    & \sum_{m\in\cM_0} \sum_{h=1}^{H_m} \left( \Bstar + 4 \hat\beta_T \hat\sigma_{t(m,h)} \right)^2  \leq 2 T \Bstar^2 + 32 \hat\beta_T^2 \sum_{m \in \cM_0} \sum_{h=1}^{H_m} \hat\sigma_{t(m,h)}^2 ,
\end{align}and
\begin{align}\label{eq: bernstein E1 bound A1 6}
    & \sum_{m \in \cM_0} \sum_{h=1}^{H_m} \hat\sigma_{t(m,h)}^2 \notag 
    \\ & = \sum_{m \in \cM_0} \sum_{h=1}^{H_m} \max\left\{ B^2/d , \ \hat\VV_t V_{j_m} (s_t,a_t) + E_t \right\} \notag 
    \\ & = \sum_{m \in \cM_0} \sum_{h=1}^{H_m} \max\left\{ B^2/d , \ \VV_t V_{j_m} (s_t,a_t) + 2E_t + \hat\VV_t V_{j_m} (s_t,a_t) - \VV V_{j_m} (s_t,a_t) - E_t\right\} 
    \\ & \leq \frac{B^2 T}{d} + 2\sum_{m \in \cM_0} \sum_{h=1}^{H_m} E_t + \sum_{m \in \cM_0} \sum_{h=1}^{H_m}\VV V_{j_m} (s_t,a_t) + \underbrace{\sum_{m \in \cM_0} \sum_{h=1}^{H_m} \ \hat\VV_t V_{j_m} (s_t,a_t) - E_t - \VV V_{j_m} (s_t,a_t) }_{\leq 0}, \notag
\end{align}where we write $t = t(m,h)$ for simplicity. Note that the last term is at most zero under the event of Lemma \ref{lem: key bern lemma}. 

For the second term $\sum_{m \in \cM_0} \sum_{h=1}^{H_m} E_t$, by \eqref{eq: VV and E}, we have 
\begin{align}\label{eq: Et bound eq 0}
    & \sum_{m \in \cM_0} \sum_{h=1}^{H_m} E_t \notag
    \\ & = \sum_{m \in \cM_0} \sum_{h=1}^{H_m} \min\{B^2, 2B \check\beta_t\hat\sigma_t \|\bSigma_{t-1}^{-1/2} \bphi_{V_j}(s_t, a_t)/\hat\sigma_t\|_2 \} + \min\{B^2, \tilde\beta_t\| \tilde\bSigma_{t-1}^{-1/2} \bphi_{V_j^2}(s_t, a_t)\|_2 \} \notag
    \\ & \leq  2B \sum_{m \in \cM_0} \sum_{h=1}^{H_m} \check\beta_t \hat\sigma_t \min\{1, \|\bSigma_{t-1}^{-1/2} \bphi_{V_j}(s_t, a_t)/\hat\sigma_t\|_2\}\notag
    \\ & \qquad + \sum_{m \in \cM_0} \sum_{h=1}^{H_m} \tilde\beta_t \min\{1, \| \tilde\bSigma_{t-1}^{-1/2} \bphi_{V_j^2}(s_t, a_t)\|_2 \} \notag
    \\ & \leq 2\sqrt{3} B^2 \check\beta_T \sum_{m \in \cM_0} \sum_{h=1}^{H_m} \min\{1, \|\bSigma_{t-1}^{-1/2} \bphi_{V_j}(s_t, a_t)/\hat\sigma_t\|_2\} \notag
    \\ & \qquad + \tilde\beta_T \sum_{m \in \cM_0} \sum_{h=1}^{H_m} \min\{1, \| \tilde\bSigma_{t-1}^{-1/2} \bphi_{V_j^2}(s_t, a_t)\|_2 \}, 
\end{align} where the first inequality holds since $\check\beta_t \hat\sigma_t \geq B$ and $\tilde\beta_t \geq B^2$, and the second inequality holds since $\check\beta_t \leq \check\beta_T$ and $\hat\sigma_t \leq \sqrt{3}B$ by $\hat{\sigma}_t^2 \leftarrow \max\{B^2/d, [\hat{\VV}_t V_j](s_t, a_t) + E_t \}$ and \eqref{eq: VV and E}. 
\begin{align}\label{eq: Et bound eq 1}
    \sum_{m \in \cM_0} \sum_{h=1}^{H_m} \min\{1, \|\bSigma_{t-1}^{-1/2} \bphi_{V_j}(s_t, a_t)/\hat\sigma_t\|_2\} & \leq \sqrt{T} \sqrt{ \sum_{m \in \cM_0} \sum_{h=1}^{H_m} \min\{1, \|\bSigma_{t-1}^{-1/2} \bphi_{V_j}(s_t, a_t)/\hat\sigma_t\|^2_2\}} \notag
    \\ & \leq \sqrt{2dT \log\left( 1 + T(\Bstar \sqrt{d}/B)^2/(d\lambda) \right)} \notag
    \\ & = \sqrt{2dT \log\left( 1 + T/\lambda\right)},
\end{align}where the first inequality is by Cauchy Schwarz inequality, and the second inequality is by Lemma~\ref{lem: lemma 11 in abbasi} and $\|\bphi_{V_j}(s_t, a_t)/\hat\sigma_t\|_2 \leq \Bstar/(B/\sqrt{d})$, since $\|\bphi_{V_j}\| \leq \Bstar$ under the event of Lemma \ref{lem: key bern lemma} by Assumption~\ref{assump: linear kernel mdp} and $\hat\sigma_t \geq B/\sqrt{d}$ by definition. Similarly, we have 
\begin{align}\label{eq: Et bound eq 2}
    \sum_{m \in \cM_0} \sum_{h=1}^{H_m} \min\{1, \| \tilde\bSigma_{t-1}^{-1/2} \bphi_{V_j^2}(s_t, a_t)\|_2 \} & \leq \sqrt{2dT\log\left( 1 + T\Bstar^4/(d\lambda)\right)}.
\end{align}Combining \eqref{eq: Et bound eq 0}, \eqref{eq: Et bound eq 1} and \eqref{eq: Et bound eq 2}, we get 
\begin{align}\label{eq: Et bound final}
    \sum_{m\in\cM_0} \sum_{h=1}^{H_m} E_t & \leq 5 B^2 \check\beta_T \sqrt{dT\log(1+T/\lambda)} + \sqrt{2}\tilde\beta_T \sqrt{dT\log(1+T\Bstar^4/(d\lambda))}.
\end{align}

For the term $\sum_{m \in \cM_0} \sum_{h=1}^{H_m}\VV V_{j_m} (s_t,a_t)$, note that a trivial upper bound is $\Bstar^2 T$ since $|V_{j_m}| \leq \Bstar$ by the optimism. However, such bound is loose for the regret analysis. It turns out that we can use a total variance trick to bound the term by roughly $\cO(B^2 M)$, where $M$ is the number of intervals. This is summarized by Lemma \ref{lem: total variance}. 

For any $m$ , we define the event $\cE_m$ as 
\begin{align}\label{eq: def of event E_t}
    \cE_m & \coloneqq  \big\{ \textnormal{For all } m' \leq m:  \btheta^* \in \hat\cC_{j_{m'}} \cap \cB, \quad 0\leq Q_{j_{m'}}(\cdot,\cdot) \leq Q^{\star}(\cdot,\cdot),\quad \textnormal{and} \quad 0 \leq V_{j_{m'}}(\cdot) \leq V^{\star}(\cdot) \big\}.
\end{align} By the definition, it is clear that the event of Lemma \ref{lem: key bern lemma} is a subset of $\cE_M$. Also we have $\cE_m \subseteq \cE_{m'}$ for any $m > m'$. 

\begin{lemma}[Total variance bound]\label{lem: total variance}
    With probability at least $1-3\delta$, it holds that 
    \begin{align*}
        & \sum_{m \in \cM_0} \sum_{h=1}^{H_m}\VV V_{j_m} (s_t,a_t) \ind\{\cE_m\}
        \\ & \leq 18 B^2 M + \frac{338B^3 d \hat\beta_T^2}{c_{\min}} \left( \sqrt{M} \sqrt{\log(1/\delta)}\log\left( 1 + \frac{2B}{c_{\min} \lambda} \right) + \log(1+T/\lambda) \right).
    \end{align*}
\end{lemma}
The proof of Lemma \ref{lem: total variance} is in Appendix \ref{sec: proof of total variance}.

Combining \eqref{eq: bernstein E1 bound A1 6}, \eqref{eq: Et bound final} and Lemma \ref{lem: total variance}, we can bound $\sum_{m \in \cM_0} \sum_{h=1}^{H_m} \hat\sigma_{t(m,h)}^2$ by 
\begin{align}\label{eq: bernstein E1 bound sum of sigma2}
    & \sum_{m \in \cM_0} \sum_{h=1}^{H_m} \hat\sigma_{t(m,h)}^2 \notag 
    \\ & \leq \frac{B^2 T}{d} + 18B^2 M + 10 B^2 \check\beta_T \sqrt{dT\log(1+T/\lambda)} + 3\tilde\beta_T \sqrt{dT\log(1+T\Bstar^4/(d\lambda))} \notag
    \\ & \qquad + \frac{338B^3 d \hat\beta_T^2}{c_{\min}} \left( \sqrt{M} \sqrt{\log(1/\delta)}\log\left( 1 + \frac{2B}{c_{\min} \lambda} \right) + \log(1+T/\lambda) \right),
\end{align}with probability at least $1-6\delta$ by a union bound over the event of Lemma \ref{lem: key bern lemma} and Lemma \ref{lem: total variance}. Plugging \eqref{eq: bernstein E1 bound sum of sigma2} into \eqref{eq: bernstein E1 bound A1 5}, and then using \eqref{eq: bernstein E1 bound A1 3}, \eqref{eq: bernstein E1 bound A1 4} and \eqref{eq: choices of 3 betas}, we conclude that $A_1$ can be bounded as  
\begin{align}\label{eq: A1 bound tilde O}
    A_1 = \tilde\cO\left( \sqrt{B^2 d T + B^2 d^2 M + B^2 d^{3.5} T^{0.5}+\frac{B^3 d^4}{c_{\min}}M^{0.5} }\right),
\end{align}where we use $\tilde\cO(\cdot)$ to hide a term polynomial in $\log^2(TB/(d\lambda \delta))$.

\noindent\textbf{To bound $A_2$}: By rewriting the summation using the index $j$, it immediately follows from Lemma \ref{lem: J bound bernstein} that
\begin{align*}
    A_2 & \leq (\Bstar+1) \sum_{j=0}^J \sum_{t=t_j+1}^{t_{j+1}} \frac{1}{t_j} \leq 2(J+1) = 4d \log(1+T/\lambda) + 2 \log(T) < A_1.
\end{align*} 
Furthermore, since \texttt{DEVI} is called at $t=2t_0=2$ by the time step doubling condition,  we have 
\begin{align*}
    \sum_{m\in\cM_0^c} \sum_{h=1}^{H_m} \left[ c_{m,h} + \PP V_{j_m}(s_{m,h},a_{m,h}) - Q_{j_m} (s_{m,h},a_{m,h}) \right] &= \sum_{h=1}^2 \left[ c_{1,h} + \PP V_{0}(s_{1,h},a_{1,h}) - Q_{0} (s_{1,h},a_{1,h}) \right]
    \\ & \leq 4,
\end{align*}where the inequality holds because $c_{1,h}, V_0(\cdot) \leq 1$ and $0 \leq Q_0(\cdot,\cdot)$. 
Together with \eqref{eq: bern bound E1 is A1 and A2} and \eqref{eq: A1 bound tilde O}, we conclude that 
\begin{align}\label{eq: E1 bound bernstein bigO}
    E_1 & = \tilde\cO\left( \sqrt{B^2 d T + B^2 d^2 M + B^2 d^{3.5} T^{0.5}+\frac{B^3 d^4}{c_{\min}}M^{0.5} }\right) . 
\end{align}

\subsubsection{Bounding $E_2$ and $E_3$}

To bound $E_2$, by the same reasoning as in Lemma \ref{lem: E2 bound}, we have that, with probability at least $1-4\delta$, the event of Lemma \ref{lem: key bern lemma} holds and
\begin{align}\label{eq: E2 bound bernstein bigO}
    E_2 \leq 2 \Bstar \sqrt{2T \log\left( \frac{T}{\delta}\right)} = \tilde\cO \left(\Bstar\sqrt{T}\right).
\end{align} Here the $1-4\delta$ probability comes from a union bound of Lemma \ref{lem: key bern lemma} and Azuma-Hoeffding inequality. 

To bound $E_3$, following the proof of Lemma \ref{lem: regret decomp telescope 1} and \eqref{eq: lem regret decomp telescope 1}, we can write
\begin{align}\label{eq: bernstein E3 bound eq 1}
    & \sum_{m=1}^M\left( \sum_{h=1}^{H_m} V_{j_m}(s_{m,h}) - V_{j_m}(s_{m,h+1}) \right) \leq \sum_{m=1}^{M-1} \left( V_{j_{m+1}}(s_{m+1,1}) - V_{j_m}(s_{m,H_m+1}) \right) + V_{j_1} (s_{1,1}).
\end{align} 

Now consider the term $V_{j_{m+1}}(s_{m+1,1}) - V_{j_m}(s_{m,H_m+1})$. Note that by the interval decomposition, interval $m$ ends if and only if either of the three conditions are met. If interval $m$ ends because goal is reached, then we have 
\begin{align*}
    V_{j_{m+1}}(s_{m+1,1}) - V_{j_m}(s_{m,H_m+1}) = V_{j_{m+1}} (\sinit) - V_{j_m}(g) = V_{j_{m+1}} (\sinit).
\end{align*}
If it ends because the \texttt{DEVI} sub-routine is triggered, then the value function estimator is updated by \texttt{DEVI} and $j_m\neq j_{m+1}$. In such case we simply apply the trivial upper bound $V_{j_{m+1}}(s_{m+1,1}) - V_{j_m}(s_{m,H_m+1}) \leq \max_j \|V_j\|_\infty$. By Lemma \ref{lem: J bound bernstein}, this happens at most $J \leq 2 d \log \left( 1 + \frac{T}{\lambda} \right) + 2 \log(T)$ times. If the interval ends because the cumulative cost reaches $B$, then $s_{m+1,1} = s_{m,H_m+1}$ and $V_{j_m} = V_{j_{m+1}}$ and hence $V_{j_{m+1}}(s_{m+1,1}) - V_{j_m}(s_{m,H_m+1})=0$. 
Therefore, we can bound the RHS of \eqref{eq: bernstein E3 bound eq 1} as 
\begin{align*}
    & \sum_{m=1}^M\left( \sum_{h=1}^{H_m} V_{j_m}(s_{m,h}) - V_{j_m}(s_{m,h+1}) \right) \notag
    \\ & \leq \sum_{m=1}^{M-1} V_{j_{m+1}} (\sinit) \cdot \ind\{m+1\in \cM(M)\} + V_{j_1} (s_{1,1}) + \left[ 2 d \log \left( 1 + \frac{T}{\lambda} \right) + 2 \log(T)\right] \cdot \max_{j}\|V_j\|_\infty
    \\ & \leq \sum_{m \in \cM(M)} V_{j_m}(\sinit) + V_0(\sinit) + 2 d \Bstar \log \left( 1 + \frac{T}{\lambda} \right) + 2 \Bstar \log(T)
    \\ & \leq  \sum_{m \in \cM(M)} V_{j_m}(\sinit) + 1 + 2 d \Bstar \log \left( 1 + \frac{T}{\lambda} \right) + 2 \Bstar \log(T) ,
\end{align*}where the second inequality is by $\|V_j\|_\infty \leq \Bstar$ and the last step is by the initialization $\|V_0\|_\infty \leq 1$. 
Rearranging the terms, we conclude that under the event of Lemma \ref{lem: key bern lemma},
\begin{align}\label{eq: E3 bound bernstein bigO}
    E_3 \leq 1 + 2 d \Bstar \log \left( 1 + \frac{T}{\lambda} \right) + 2 \Bstar \log(T) = \tilde\cO(Bd).
\end{align}

\subsubsection{Bounding $R(M)$}
Combining \eqref{eq: bernstein regret decomposition}, \eqref{eq: E1 bound bernstein bigO}, \eqref{eq: E2 bound bernstein bigO} and \eqref{eq: E3 bound bernstein bigO}, we get the final bound for $R(M)$:
\begin{align*}
    R(M) & = \tilde\cO\left( \sqrt{B^2 d T + B^2 d^2 M + B^2 d^{3.5} T^{0.5}+\frac{B^3 d^4}{c_{\min}}M^{0.5} }\right) , 
\end{align*}where $\tilde\cO(\cdot)$ hides a term of  $C\cdot\log^2(TB/(\lambda \delta c_{\min}))$ for some problem-independent constant $C$.
This holds with probability at least $1-7\delta$ by a union bound over the event of Lemma \ref{lem: key bern lemma}, Lemma \ref{lem: total variance} and \eqref{eq: E2 bound bernstein bigO}.

\subsection{Proof of Lemma \ref{lem: total variance}}\label{sec: proof of total variance}

In this section, for any interval $m$, we define $F_m$ as the trajectory until the end of the interval $m$, i.e., 
\begin{align*}
    F_m = \cup_{m' = 1}^{m} \{s_{t(m',1)}, a_{t(m',1)},\cdots, s_{t(m',H_{m'})}, a_{t(m',H_{m'})}, s_{t(m',H_{m'}+1)} \}.
\end{align*}

We first prove the following result which bounds the expected sum of variance for an arbitrary interval by using the technique from \citet{azar2017minimax}.

\begin{lemma}\label{lem: sum of variance interval}
    For any interval $m$, 
    \begin{align*}
        &\EE\left[ \sum_{h=1}^{H_m}\VV V_j (s_t,a_t) \ind\{\cE_m\} \middle| F_{m-1} \right]  \leq 18 B^2 + 24\xi,
        \\ & \xi  = 24 \frac{B}{c_{\min}} \EE\left[\sum_{h=1}^{H_m} \left( \hat\beta_t^2 \|\bphi_{V_j}(s_t,a_t)\|_{\bSigma_t^{-1}}^2 + \frac{2\Bstar^2}{t_j^2} \right) \ind\{\cE_m\} \middle| F_{m-1}\right] .
    \end{align*}
\end{lemma}

To prove Lemma \ref{lem: sum of variance interval}, we need the following lemma.
\begin{lemma}[Lemma B.15 in \citealt{cohen2020near}]\label{lem: square and summation order}
    Let $\{X_t\}_{t=1}^\infty$ be a martingale difference sequence adapted to a filtration $\{\cF_t\}_{t=0}^\infty$, such that $X_t$ is $\cF_t$-measurable. Let $Y_n = (\sum_{t=1}^n X_t)^2 - \sum_{t=1}^n \EE[X_t^2 \mid \cF_{t-1}]$. Then $\{Y_n\}_{n=0}^\infty$ is a martingale. If we further assume $H$ is a stopping time such that $H < C$ for some fixed $C$ almost surely, then $\EE[Y_H] = 0$.
\end{lemma}

\begin{proof}[Proof of Lemma \ref{lem: sum of variance interval}]
    We define $X_{m,h} = [\PP V_j(s_t,a_t) - V_j(s_{t+1})] \ind\{\cE_m\}$, where $t =t(m,h)$ and $j = j_m$ as a simplified notation. Then conditioned on $F_{m-1}$, $\ind\{\cE_m\}$ is determined, and thus $\{X_{m,h}\}_{h=1}^\infty$ is a martingale difference sequence with respect to $\{F_{m,h}\}_{h=1}^\infty$, where $F_{m,h}$ is the trajectory from the from the beginning of interval $m$ to time $h$. Furthermore, $H_m$ is a stopping time w.r.t. the trajectory which is upper bounded by $2\Bstar/c_{\min}$ since by the interval decomposition a new interval would start if the cumulative cost exceeds $\Bstar$. Therefore, we have 
    \begin{align}\label{eq: proof of total variance single interval eq 0}
        \EE\left[ \sum_{h=1}^{H_m}\VV V_j (s_t,a_t) \ind\{\cE_m\} \middle| F_{m-1} \right] & \coloneqq \EE\left[ \sum_{h=1}^{H_m}\left( \PP V_j(s_t,a_t) - V_j(s_{t+1}) \right)^2 \ind\{\cE_m\} \middle| F_{m-1} \right]\notag
        \\ & = \EE\left[ \sum_{h=1}^{H_m}\left( X_{m,h} \right)^2  \middle| F_{m-1} \right]\notag
        \\ & = \EE\left[ \left( \sum_{h=1}^{H_m} X_{m,h} \right)^2  \middle| F_{m-1} \right]\notag
        \\ & = \EE\left[ \left(\sum_{h=1}^{H_m}  \left[ \PP V_j(s_t,a_t) - V_j(s_{t+1}) \right] \right)^2 \ind\{\cE_m\}  \middle| F_{m-1} \right],
    \end{align}where the third step is by Lemma \ref{lem: square and summation order}.
    
    By Assumption \ref{assump: linear kernel mdp}, we can write $\PP V_j(s_t, a_t) = \langle \btheta^*, \bphi_{V_j}(s_t,a_t) \rangle$. Furthermore, since $V_j$ and $Q_j$ are the output of \texttt{DEVI}, we can write $Q_j = Q^{(l)}$ and $V_j (\cdot) = \min_a Q^{(l)} (\cdot, a)$ for some $l$, which denote the $l$-th iteration in the implementation of \texttt{DEVI}. It follows that 
    \begin{align*}
        Q^{(l)} (s_t,a_t) & = c(s_t,a_t) + (1-q_j) \min_{\btheta \in \hat\cC_{j_m}\cap \cB} \langle \btheta, \bphi_{V^{(l-1)}} (s_t,a_t) \rangle
        \\ & = c_t + (1-q_j) \langle \btheta_{m,h}, \bphi_{V^{(l)}}(s_t,a_t) \rangle + (1-q_j) \langle \btheta_{m,h}, \bphi_{V^{(l-1)}}(s_t,a_t) - \bphi_{V^{(l)}}(s_t,a_t) \rangle
        \\ & = c_t + (1-q_j) \langle \btheta_{m,h}, \bphi_{V^{(l)}}(s_t,a_t) \rangle + (1-q_j)\PP_{m,h}[V^{(l-1)} - V^{(l)}](s_t,a_t), 
    \end{align*}where $\btheta_{m,h}$ in the second step denotes the minimizer of the first step, $\PP_{m,h}$ is the transition kernel defined by $\btheta_{m,h}$, and $c_t = c(s_t,a_t)$. Since $Q_j(s_t,a_t) = \min_a Q_j(s_t,a) = V_j(s_t)$, we get 
    \begin{align}\label{eq: proof of total variance single interval eq 1}
        V_j(s_t) - c_t & = Q^{(l)}(s_t,a_t) - c_t \notag
        \\ & = (1-q_j) \langle \btheta_{m,h}, \bphi_{V^{(l)}}(s_t,a_t) \rangle + (1-q_j)\PP_{m,h}[V^{(l-1)} - V^{(l)}](s_t,a_t) \notag
        \\ & = (1-q_j) \langle \btheta_{m,h} - \btheta^* + \btheta^*, \bphi_{V^{(l)}}(s_t,a_t) \rangle + (1-q_j)\PP_{m,h}[V^{(l-1)} - V^{(l)}](s_t,a_t) \notag
        \\ & = (1-q_j) \langle \btheta_{m,h} - \btheta^*, \bphi_{V_j}(s_t,a_t) \rangle + (1-q_j)\PP V_j(s_t,a_t) \notag 
        \\ & \qquad +(1-q_j)\PP_{m,h}[V^{(l-1)} - V^{(l)}](s_t,a_t) \notag
        \\ & = \PP V_j(s_t,a_t) + (1-q_j) \langle \btheta_{m,h} - \btheta^*, \bphi_{V_j}(s_t,a_t) \rangle \notag
        \\ & \qquad  +(1-q_j)\PP_{m,h}[V^{(l-1)} - V^{(l)}](s_t,a_t) - q_j\PP V_j(s_t,a_t)
    \end{align} Note that when event $\cE_m$ holds, we can write
    \begin{align}\label{eq: proof of total variance single interval eq 2}
        & \left| \langle \btheta_{m,h} - \btheta^*, \bphi_{V_j}(s_t,a_t) \rangle \right|\notag
        \\ & =   \left| \langle \btheta_{m,h} - \hat\btheta_j + \hat\btheta_j - \btheta^*, \bphi_{V_j}(s_t,a_t) \rangle \right|\notag
        \\ & \leq |(\btheta_{m,h} - \hat\btheta_j)^\top \bSigma_t^{1/2} \bSigma_t^{-1/2} \bphi_{V_j}(s_t,a_t)| + |(\btheta^* - \hat\btheta_j)^\top \bSigma_t^{1/2} \bSigma_t^{-1/2} \bphi_{V_j}(s_t,a_t)|\notag
        \\ & \leq \|\btheta_{m,h} - \hat\btheta_j\|_{\bSigma_t} \cdot \|\bphi_{V_j}(s_t,a_t)\|_{\bSigma_t^{-1}}\notag
        \\ & \qquad + \|\btheta^* - \hat\btheta_j\|_{\bSigma_t} \cdot \|\bphi_{V_j}(s_t,a_t)\|_{\bSigma_t^{-1}}\notag
        \\ & \leq 2 \hat\beta_t \cdot \|\bphi_{V_j}(s_t,a_t)\|_{\bSigma_t^{-1}},
    \end{align} where the first inequality is by the triangular inequality, the second inequality is by the Cauchy Schwarz inequality, and the third step holds because $\btheta^* \in \cC_{j_m}$ under the event $\cE_m$. 
    
    Combining \eqref{eq: proof of total variance single interval eq 1} and \eqref{eq: proof of total variance single interval eq 2}, under the event $\cE_m$, we have
    \begin{align*}
        V_j(s_t) - c_t & = \PP V_j(s_t,a_t) + e_t,
    \end{align*}where
    \begin{align}\label{eq: proof of total variance single interval eq 3}
        |e_t| \leq  2(1-q_j) \hat\beta_t \cdot \|\bphi_{V_j}(s_t,a_t)\|_{\bSigma_t^{-1}} + q_j \Bstar + (1-q_j) \epsilon_j,
    \end{align}where we use $|V_j| \leq \Bstar$ under $\cE_m$ for $j=j_m$. Together with \eqref{eq: proof of total variance single interval eq 0}, we have 
    \begin{align}\label{eq: proof of total variance single interval eq 4}
        & \EE\left[ \sum_{h=1}^{H_m}\VV V_j (s_t,a_t) \ind\{\cE_m\} \middle| F_{m-1} \right] \notag
        \\ & = \EE\left[ \left(\sum_{h=1}^{H_m}  \left[ V_j(s_t) - c_t - e_t - V_j(s_{t+1}) \right] \right)^2 \ind\{\cE_m\}  \middle| F_{m-1} \right] \notag
        \\ & \leq \EE\left[ 2 \left(\sum_{h=1}^{H_m}  \left[ V_j(s_t) - c_t  - V_j(s_{t+1}) \right] \right)^2 \ind\{\cE_m\} + 2\left( \sum_{h=1}^{H_m}e_t \right)^2\ind\{\cE_m\}  \middle| F_{m-1} \right] \notag
        \\ & = \EE\left[ 2 \left(V_j(s_{t(m,1)}) - V_j(s_{t(m,H_m+1)}) - \sum_{h=1}^{H_m}  c_t \right)^2 \ind\{\cE_m\} + 2\left( \sum_{h=1}^{H_m}e_t \right)^2\ind\{\cE_m\}  \middle| F_{m-1} \right] \notag
        \\ & \leq \EE\left[ 2 \left(3B \right)^2 \ind\{\cE_m\} + 2\left( \sum_{h=1}^{H_m}e_t \right)^2\ind\{\cE_m\}  \middle| F_{m-1} \right] ,
    \end{align}where the second step is by $(a+b)^2 \leq 2a^2 + 2b^2$, the third step is by canceling the terms in the telescoping sum, and the last step holds since $|V_j(s_{t(m,1)}) - V_j(s_{t(m,H_m+1)})| \leq \Bstar\leq B$ by optimism and $\sum_{h=1}^{H_m} c_t \leq 2 B$ by the interval decomposition. Furthermore, we can write 
    \begin{align}\label{eq: proof of total variance single interval eq 5}
        & \EE \left[ \left( \sum_{h=1}^{H_m}e_t \right)^2 \middle| F_{m-1}\right] \notag
        \\ & \leq \EE\left[ H_m \sum_{h=1}^{H_m} e_t^2 \ind\{\cE_m\} \middle| F_{m-1}\right] \notag
        \\ & \leq \frac{2B}{c_{\min}} \cdot \EE\left[ \sum_{h=1}^{H_m} e_t^2 \ind\{\cE_m\} \middle| F_{m-1}\right] \notag
        \\ & \leq \frac{2B}{c_{\min}} \cdot \EE \left[ 6 \sum_{h=1}^{H_m} \left( \hat\beta_t^2 \|\bphi_{V_j}(s_t,a_t)\|_{\bSigma_t^{-1}}^2 + q_j^2\Bstar^2 + (1-q_j)^2 \epsilon_j^2 \right)\ind\{\cE_m\} \middle| F_{m-1} \right],
    \end{align} where the first and third steps are by $(\sum_{i=1}^k a_i)^2 \leq k \sum_{i=1}^k a_i^2$ for $k\geq 1$, and the second step is by $H_m \leq 2B/c_{\min}$. 
    
    Combining \eqref{eq: proof of total variance single interval eq 4} and \eqref{eq: proof of total variance single interval eq 5}, we have 
    \begin{align*}
        & \EE\left[ \sum_{h=1}^{H_m}\VV V_j (s_t,a_t) \ind\{\cE_m\} \middle| F_{m-1} \right] \notag
        \\ & \leq 18 B^2 + 24 \frac{B}{c_{\min}} \EE\left[\sum_{h=1}^{H_m} \left( \hat\beta_t^2 \|\bphi_{V_j}(s_t,a_t)\|_{\bSigma_t^{-1}}^2 + \Bstar^2 (q_j^2 + \epsilon_j^2) \right) \ind\{\cE_m\} \middle| F_{m-1}\right] .
    \end{align*} Plugging in $q_j = \epsilon_j = 1/t_j$ finishes the proof.
    
\end{proof}

We are now ready to prove Lemma \ref{lem: total variance} by using Lemma \ref{lem: sum of variance interval} and applying Azuma-Hoeffding inequality multiple times. 
\begin{proof}[Proof of Lemma \ref{lem: total variance}]
    We first bound $\sum_{m \in \cM_0} \EE\left[ \sum_{h=1}^{H_m}\VV V_j (s_t,a_t) \ind\{\cE_m\} \middle| F_{m-1} \right] $. By Lemma \ref{lem: sum of variance interval}, we have 
    \begin{align}\label{eq: proof of total variance eq 1}
        & \sum_{m \in \cM_0} \EE\left[ \sum_{h=1}^{H_m}\VV V_j (s_t,a_t) \ind\{\cE_m\} \middle| F_{m-1} \right] \notag 
        \\ & \leq 18 B^2 M + \underbrace{\frac{24B}{c_{\min}} \sum_{m\in\cM_0} \EE \left[ \sum_{h=1}^{H_m} \left( \hat\beta_t^2 \|\bphi_{V_j}(s_t,a_t)\|_{\bSigma_t^{-1}}^2 \right) \ind\{\cE_m\} \middle| F_{m-1} \right]}_{\textnormal{I}} \notag 
        \\ & \qquad + \underbrace{\frac{48B \Bstar^2}{c_{\min}} \sum_{m\in\cM_0}\EE\left[ \sum_{h=1}^{H_m} \frac{1}{t_j^2}\ind\{\cE_m\} \middle| F_{m-1}\right]}_{\textnormal{II}} .
    \end{align}
    To bound \textnormal{I}, first note that, by picking $\lambda = 1/B^2$, 
    \begin{align}\label{eq: proof of total variance eq 2}
        & \sum_{m\in \cM_0} \sum_{h=1}^{H_m} \left( \hat\beta_t^2 \|\bphi_{V_j}(s_t,a_t)\|_{\bSigma_t^{-1}}^2 \right) \ind\{\cE_m\} \notag
        \\ & = \sum_{m\in \cM_0} \sum_{h=1}^{H_m} \left( \hat\beta_t^2\hat\sigma_t^2 \|\bphi_{V_j}(s_t,a_t)/\hat\sigma_t\|_{\bSigma_t^{-1}}^2 \right) \ind\{\cE_m\} \notag
        \\ & \leq 3\hat\beta_T^2 \cdot B^2 d \cdot \log \left( 1+ T/\lambda \right) ,
    \end{align}where we use Lemma \ref{lem: determinant trace}, Lemma \ref{lem: lemma 11 in abbasi} and $B/\sqrt{d}\leq \hat\sigma_t \leq \sqrt{3}B$. Define $X_m$ as
    \begin{align*}
        X_m = \sum_{h=1}^{H_m} \left( \hat\beta_t^2 \|\bphi_{V_j}(s_t,a_t)\|_{\bSigma_t^{-1}}^2 \right) \ind\{\cE_m\} - \EE [ \sum_{h=1}^{H_m} ( \hat\beta_t^2 \|\bphi_{V_j}(s_t,a_t)\|_{\bSigma_t^{-1}}^2 ) \ind\{\cE_m\} \mid F_{m-1} ]. 
    \end{align*}Then $\{X_m\}_{m=1}^\infty $ is a martingale difference sequence since $\EE[X_m | F_{m-1}] = 0$. Furthermore, by Lemma~\ref{lem: lemma 11 in abbasi} again we have $|X_m| \leq 3 \hat\beta_t^2 B^2 d \log(1 + 2B/(c_{\min} \lambda))$, since $H_m \leq 2B/c_{\min}$. By Azuma-Hoeffding inequality, we get that, with probability at least $1-\delta$, 
    \begin{align*}
        \left| \sum_{m\in \cM_0} X_m \right| \leq 6 \hat\beta_T^2 \cdot B^2 d \cdot \log\left(1 + \frac{2B}{c_{\min} \lambda}\right) \sqrt{M \log(1/\delta)}. 
    \end{align*}Combining with \eqref{eq: proof of total variance eq 2}, we have that, with probability at least $1-\delta$, 
    \begin{align}\label{eq: proof of total variance eq 3}
        \textnormal{I} & \leq \frac{144 B^3d \hat\beta_T^2}{c_{\min}} \left( \sqrt{M}\cdot \log(1+2B/(c_{\min}\lambda)) \cdot\sqrt{\log(1/\delta)} + \log(1+T/\lambda) \right). 
    \end{align}
    Similarly, to bound \textnormal{II}, we have 
    \begin{align*}
        \sum_{m\in\cM_0} \sum_{h=1}^{H_m} \frac{1}{t_j^2}\ind\{\cE_m\} \leq \sum_{j=1}^J \sum_{t = t_j+1}^{t_{j+1}} \frac{1}{t_j^2} \leq \sum_{j=1}^J \frac{2t_j}{t_j^2} \leq J+1,
    \end{align*}since $2t_j \leq t_j^2$ for all $t_j >1$. We then apply Azuma-Hoeffding inequality to the martingale difference sequence $Y_m \coloneqq \sum_{h=1}^{H_m} \frac{1}{t_j^2}\ind\{\cE_m\} - \EE[\sum_{h=1}^{H_m} \frac{1}{t_j^2}\ind\{\cE_m\}\mid F_{m-1}]$ with $|Y_m|\leq 2$ since $H_m\leq t_{j+1} - t_j\leq 2 t_j$, and get that, with probability at least $1-\delta$, $|\sum_{m\in \cM_0} Y_m |\leq 4\sqrt{M \log(1/\delta)} $. It follows that, with probability at least $1-\delta$,
    \begin{align}\label{eq: proof of total variance eq 4}
        \textnormal{II} \leq \frac{48 B\Bstar^2}{c_{\min}} \left( J+4\sqrt{M\log(1/\delta)} +1  \right).
    \end{align}Plugging \eqref{eq: proof of total variance eq 3} and \eqref{eq: proof of total variance eq 4} into \eqref{eq: proof of total variance eq 1} and by a union bound, we conclude that, with probability at least $1-2\delta$, 
    \begin{align}\label{eq: proof of total variance eq 5}
        & \sum_{m \in \cM_0} \EE\left[ \sum_{h=1}^{H_m}\VV V_j (s_t,a_t) \ind\{\cE_m\} \middle| F_{m-1} \right] \notag 
        \\ & \leq 18 B^2 M + \frac{336B^3 d \hat\beta_T^2}{c_{\min}} \left( \sqrt{M} \sqrt{\log(1/\delta)}\log\left( 1 + \frac{2B}{c_{\min} \lambda} \right) + \log(1+T/\lambda) \right),
    \end{align}where we use the bound for $J$ from Lemma \ref{lem: J bound bernstein}.
    
    Now, to bound $\sum_{m \in \cM_0} \sum_{h=1}^{H_m}\VV V_j (s_t,a_t) \ind\{\cE_m\} $, we apply Azuma-Hoeffding inequality once again. Specifically, we define $Z_m$ as
    \begin{align*}
        Z_m \coloneqq \sum_{h=1}^{H_m} \VV V_{j} (s_t,a_t) \ind\{\cE_m\} - \EE\left[\sum_{h=1}^{H_m} \VV V_{j} (s_t,a_t) \ind\{\cE_m\} \middle| F_{m-1} \right],
    \end{align*}and we have $|Z_m| \leq 2B\Bstar^2/c_{\min}$ since $H_m \leq 2B/c_{\min}$ and $|V_j|\leq \Bstar$ on the event $\cE_m$. Then with probability at least $1-\delta$, $|\sum_{m\in \cM_0} Z_m |\leq (2B\Bstar^2/c_{\min})\sqrt{M\log(1/\delta)}$. Together with \eqref{eq: proof of total variance eq 5} and a union bound, we conclude that, with probability at least $1-3\delta$, 
    \begin{align*}
        & \sum_{m \in \cM_0} \sum_{h=1}^{H_m}\VV V_j (s_t,a_t) \ind\{\cE_m\} 
        \\ & \leq 18 B^2 M + \frac{338B^3 d \hat\beta_T^2}{c_{\min}} \left( \sqrt{M} \sqrt{\log(1/\delta)}\log\left( 1 + \frac{2B}{c_{\min} \lambda} \right) + \log(1+T/\lambda) \right). 
    \end{align*}
\end{proof}

\section{Auxiliary Lemmas} 

In this subsection we introduce the auxiliary lemmas used in the analysis. 
\begin{lemma}[Azuma-Hoeffding inequality]\label{lem: azuma hoeffding} 
Let $\{X_t\}_{t=0}^\infty$ be a real-valued martingale such that for every $t \geq 1$, it holds that $|X_t - X_{t-1}| \leq B$ for some $B \geq 0$. Then with probability at least $1-\delta$, the following holds
\begin{align*}
    |X_t - X_0| \leq 2 B \sqrt{t \log\left(\frac{1}{\delta}\right)}. 
\end{align*}
\end{lemma}

\begin{lemma}[Azuma-Hoeffding inequality, anytime version]\label{lem: azuma hoeffding anytime} 
Let $\{X_t\}_{t=0}^\infty$ be a real-valued martingale such that for every $t \geq 1$, it holds that $|X_t - X_{t-1}| \leq B$ for some $B \geq 0$. Then for any $0< \delta \leq 1/2$, with probability at least $1-\delta$, the following holds for all $t\geq 0$
\begin{align*}
    |X_t - X_0| \leq 2 B \sqrt{2t \log\left(\frac{t}{\delta}\right)} . 
\end{align*}
\end{lemma}

\begin{proof}[Proof of Lemma \ref{lem: azuma hoeffding anytime}] 
By Lemma \ref{lem: azuma hoeffding}, for any $t$, with probability at least $1- \frac{\delta}{t(t+1)}$, we have 
\begin{align*}
    |X_t - X_0| \leq 2B\sqrt{t\log\left( \frac{t(t+1)}{\delta} \right)} .
\end{align*} Note that since
\begin{align*}
    \sum_{t=1}^\infty \frac{\delta}{t(t+1)} = \sum_{t=1}^\infty \left(\frac{1}{t} -\frac{1}{t+1} \right) \delta = \delta,
\end{align*} we take an union bound and get that, with probability at least $1-\delta$, for all $t$, the following holds
\begin{align*}
    |X_t-X_0| & \leq 2B\sqrt{t\log\left( \frac{t(t+1)}{\delta} \right)} \leq 2B\sqrt{t\log\left( \frac{t^2}{\delta^2} \right)} ,
\end{align*}where the second step is by $\delta \leq 1/2$. 

\end{proof}

\begin{theorem}[Bernstein inequality for vector-valued martingales, Theorem 4.1 in \citealt{zhou2020nearly}]\label{thm: bernstain vector martingale}
    Let $\{\cF_t\}_{t=1}^\infty$ be a filtration, and $\{\xb_t, \eta_t \}_{t\geq 1}$ be a stochastic process such that $\xb_t \in \RR^d$ is $\cF_t$-measurable and $\eta_t \in \RR$ is $\cF_{t+1}$-measurable. Define $\yb_t = \langle \mu^*, \xb_t \rangle + \eta_t$. Assume the following holds:
    \begin{align*}
        |\eta_t| \leq R, \ \EE[\eta_t \mid \cF_t] = 0, \ \EE[\eta_t^2\mid \cF_t]\leq \sigma^2, \ \|\xb_t\|_2 \leq L.
    \end{align*}And for all $t\geq 1$, let
    \begin{align*}
        \beta_t = 8 \sigma \sqrt{d \log(1+tL^2/(d\lambda)) \cdot \log(4t^2/\delta) } + 4R \log(4t^2/\delta). 
    \end{align*}Then for any $0 < \delta < 1$, with probability at least $1-\delta$, for all $t \geq 1$, it holds that
    \begin{align*}
        \left\| \sum_{i=1}^t \xb_i \eta_i \right\|_{\Zb_t^{-1}} \leq \beta_t, \ \left\| \bmu_t - \bmu^* \right\|_{\Zb_t} \leq \beta_t + \sqrt{\lambda} \| \bmu^*\|_2,
    \end{align*}where $\bmu_t = \Zb_t^{-1} \wb_t$, $\Zb_t = \lambda \Ib + \sum_{i=1}^t \xb_i\xb_i^\top$, $\wb_t = \sum_{i=1}^t y_i \xb_i$.
\end{theorem}

\begin{lemma}[Determinant-trace inequality, Lemma 10 in \citealt{abbasi2011improved}]\label{lem: determinant trace}
Assume $\bphi_1,\cdots,\bphi_t \in \RR^d$ and for any $s \leq t$, $\|\bphi_s\|_2 \leq L$. Let $\lambda > 0$ and $\bSigma_t = \lambda \Ib + \sum_{s=1}^t \bphi_s \bphi_s^\top$. Then 
\begin{align*}
    \det\left( \bSigma_t \right) & \leq \left( \lambda + tL^2/d\right)^d.
\end{align*}
\end{lemma}

\begin{lemma}[Lemma 11 in \citealt{abbasi2011improved}]\label{lem: lemma 11 in abbasi}
Let $\{\bphi_t\}_{t=1}^\infty$ be in $\RR^d$ such that $\|\bphi_t\|\leq L$ for all $t$. Assume $\bSigma_0$ is a PSD matrix in $\RR^{d\times d}$, and let $\bSigma_t = \bSigma_0 + \sum_{s=1}^t \bphi_s \bphi_s^\top$. Then we have 
\begin{align*}
    \sum_{s=1}^t \min \left\{1, \|\bphi_s\|^2_{\bSigma_{s-1}^{-1}} \right\} & \leq 2 \left[ d \log\left( \frac{\textnormal{trace}(\bSigma_0)+t L^2}{d} \right) -\log \det(\bSigma_0) \right] . 
\end{align*}Furthermore, if $\lambda_{\min} (\bSigma_0) \geq \max\{1, L^2\}$, then 
\begin{align*}
    \sum_{s=1}^t \|\bphi_s\|^2_{\bSigma_{s-1}^{-1}} & \leq 2 \left[ d \log\left( \frac{\textnormal{trace}(\bSigma_0)+t L^2}{d} \right) -\log \det(\bSigma_0) \right] .
\end{align*}

\end{lemma}

\end{document}